%% file: main.tex
\documentclass[openany]{now} %

\usepackage[disable]{todonotes}
\usepackage{comment}
\excludecomment{notes}

\newcommand{\ie}{\emph{i.e.},~}
\newcommand{\eg}{\emph{e.g.},~}

\newcommand{\nn}{\nonumber}
\newcommand{\be}{\begin{equation}}
\newcommand{\ee}{\end{equation}}
\newcommand{\bea}{\begin{eqnarray}}
\newcommand{\eea}{\end{eqnarray}}
\newcommand{\given}{\,|\,}
\newcommand{\x}{\mathbf{x}}
\newcommand{\z}{\mathbf{z}}
\def\X{{\mathcal X}}
\newcommand{\E}{\ensuremath{\mathbb{E}}}
\def\Unif{\mbox{Unif}}

\def\reals{\mathbb{R}}
\def\naturals{\mathbb{N}}
\def\N{{\mathcal N}}

\def\B{{\mathcal B}}
\def\P{{\mathbb P}}
\def\calE{{\mathcal E}}

\def\Permute{\textsc{Permute}}
\newcommand{\tv}[2]{\|{#1} - {#2}\|_{\text{TV}}}
\def\Deltamax{\cal{E}_{\text{max}}}

\usepackage{amsmath,amsthm,amssymb,amsfonts}
\usepackage{mathtools}
\usepackage[margin=0.25in,font=small,labelfont=bf]{caption}
\usepackage[labelfont=bf,margin=0.1in,font=small,subrefformat=parens]{subcaption}
\usepackage{graphicx}

\usepackage{pgf}
\usepackage{tikz}
\usetikzlibrary{arrows}
\usetikzlibrary{calc}
\usepackage{algorithm}
\usepackage[noend]{algpseudocode}
\usepackage{pdflscape}

\usepackage{pbox}
\usepackage{booktabs}
\usepackage{enumitem}

\usepackage{hyperref}

\algblockdefx[AsDo]{As}{EndAs}%
  [1]{\textbf{as} #1 \textbf{do}}%
  {}

\usepackage{amsthm}

\newcommand{\iid}[1]{\stackrel{\text{iid}}{#1}}
\newcommand{\T}{\mathsf{T}}
\DeclareMathOperator{\MGF}{M}

\newcommand\indep{\protect\mathpalette{\protect\independenT}{\perp}}
\def\independenT#1#2{\mathrel{\rlap{$#1#2$}\mkern4mu{#1#2}}}

\DeclareMathOperator{\MB}{MB}
\DeclareMathOperator{\KL}{KL}
\DeclareMathOperator*{\argmin}{arg\,min}
\DeclareMathOperator*{\argmax}{arg\,max}

\DeclarePairedDelimiter{\ceil}{\lceil}{\rceil}
\DeclarePairedDelimiter{\floor}{\lfloor}{\rfloor}

\title{Patterns of Scalable Bayesian Inference}

\author{
Elaine Angelino\thanks{Authors contributed equally} \\
UC Berkeley \\
elaine@eecs.berkeley.edu
\and
Matthew James Johnson\footnotemark[1] \\
Harvard University \\
mattjj@csail.mit.edu
\and
Ryan P. Adams \\
Harvard University and Twitter\\
rpa@seas.harvard.edu
}

\begin{document}

\frontmatter  %

\renewcommand*{\thefootnote}{\fnsymbol{footnote}}
\maketitle
\renewcommand*{\thefootnote}{\arabic{footnote}}

\setcounter{tocdepth}{3}
\tableofcontents

\mainmatter

\begin{abstract}
Datasets are growing not just in size but in complexity, creating a demand for
rich models and quantification of uncertainty. Bayesian methods are an
excellent fit for this demand, but scaling Bayesian inference is a challenge.
In response to this challenge, there has been considerable recent work based on
varying assumptions about model structure, underlying computational resources,
and the importance of asymptotic correctness. As a result, there is a zoo of
ideas with few clear overarching principles.

In this paper, we seek to identify unifying principles, patterns, and
intuitions for scaling Bayesian inference. We review existing work on utilizing
modern computing resources with both MCMC and variational
approximation techniques. From this taxonomy of ideas, we characterize the
general principles that have proven successful for designing scalable inference
procedures and comment on the path forward.
\end{abstract}

\input{chapters/intro/main.tex}
\input{chapters/background/main.tex}
\input{chapters/mcmc-subsets/main.tex}
\input{chapters/mcmc-parallel/main.tex}
\input{chapters/variational/main.tex}

\chapter{Challenges and questions}
\label{ch:discussion}

In this review, we have examined a variety of different views on scaling
Bayesian inference up to large datasets and greater model complexity and out to
parallel compute resources.  Several different themes have emerged, from
techniques that exploit subsets of data for computational savings to proposals
for distributing inference computations across multiple machines.  Progress is
being made, but there remain significant open questions and outstanding
challenges to be tackled as this research programme moves forward.

\paragraph{Trading off errors in MCMC}
One of the key insights underpinning much of the recent work on scaling
Bayesian inference can be framed in terms of a kind of bias-variance tradeoff.
Traditional MCMC theory provides asymptotically unbiased estimators for which
the error can eventually be driven arbitrarily small.
However, in practice, under limited computational budgets the error can be
significant.
This error has two components: transient bias, in which the samples produced
are too dependent on the Markov chain's initialization, and Monte Carlo
standard error, in which the samples collected may be too few or too highly
correlated to produce good estimates.

\begin{figure}[tp]
  \centering
  \includegraphics[width=\linewidth]{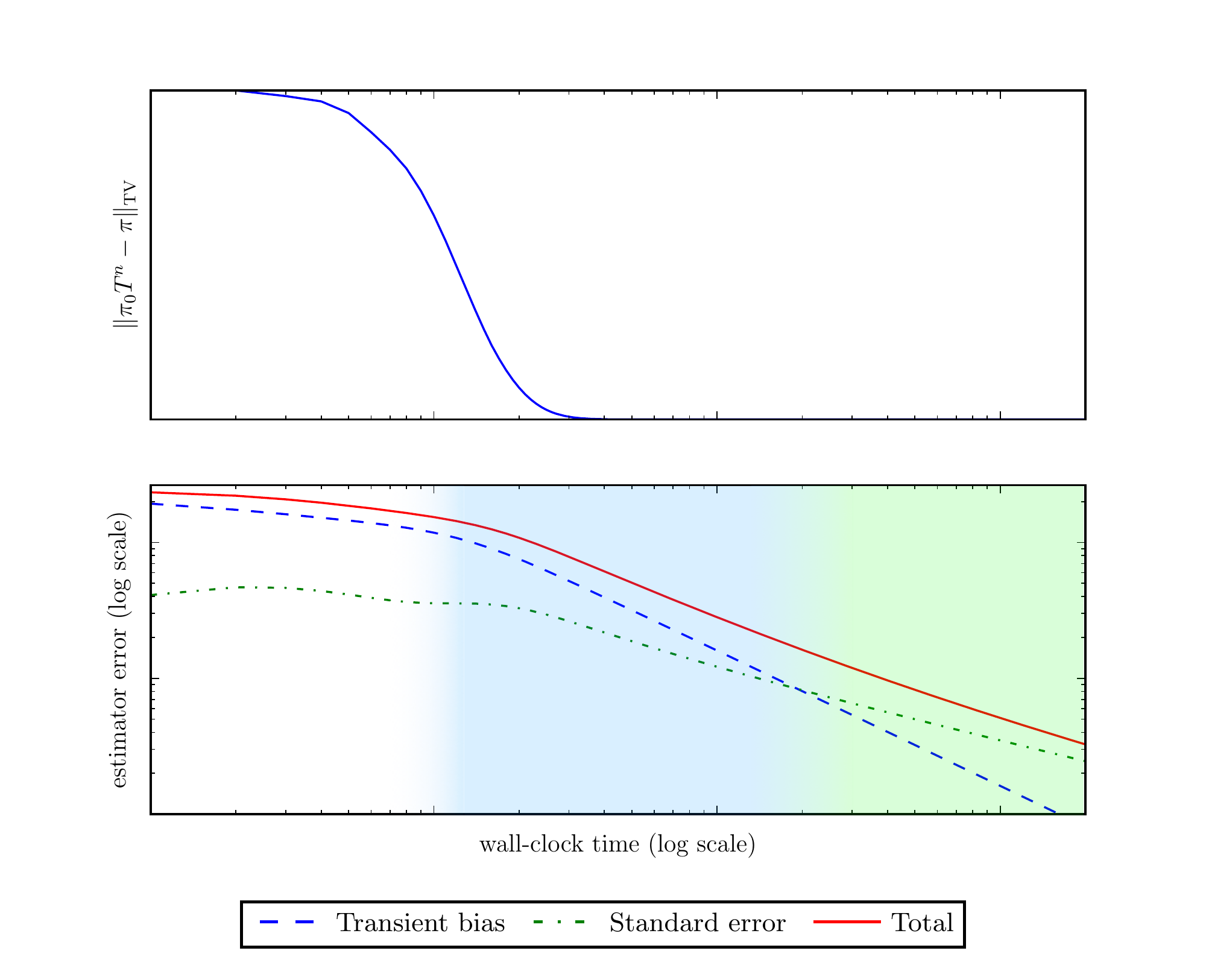}
  \caption{A simulation illustrating the error terms in traditional MCMC
  estimators as a function of wall-clock time (log scale). The marginal distributions of
  the Markov chain iterates converge to the target distribution (top panel),
  while the errors in MCMC estimates due to transient bias and Monte Carlo
  standard error are driven arbitrarily small.}
  \label{fig:traditional_mcmc}
\end{figure}

Figure~\ref{fig:traditional_mcmc} illustrates the error regimes and tradeoffs
in traditional MCMC.%
\footnote{See also Section~\ref{sec:mcmc}}
Asymptotic analysis describes the regime on the right of the plot, after the
sampler has mixed sufficiently well.
In this regime, the marginal distribution of each sample is essentially equal
to the target distribution, and the transient bias from initialization, which
affects only the early samples in the Monte Carlo sum, is washed out rapidly at
least at a $\mathcal{O}(\frac{1}{n})$ rate.
The dominant source of error is due to Monte Carlo standard error, which
diminishes only at a $\mathcal{O}(\frac{1}{\sqrt{n}})$ rate.

However, machine learning practitioners using MCMC often find themselves in
another regime: in the middle of the plot, the error is decreasing but
dominated instead by the transient bias.
The challenge in practice is often to get through this regime, or even to get
into it at all.
When the underlying Markov chain does not mix sufficiently well or when the
transitions cannot be computed sufficiently quickly, getting to this
regime may be practically infeasible for a realistic computational budget.

Several of the new MCMC techniques we have studied aim to address this
challenge.
In particular, the parallel predictive prefetching method of
Section~\ref{sec:prefetching} accelerates this phase of MCMC without affecting
the stationary distribution.
Other methods instead introduce approximate transition operators that can be
executed more efficiently.
For example, the adaptive subsampling methods of Section~\ref{sec:adaptive} and
the stochastic gradient sampler of Section~\ref{sec:sgld} can execute updates
more efficiently by operating only on data subsets, while the Weierstrass and
Hogwild Gibbs samplers of Sections~\ref{sec:weierstrass}
and~\ref{sec:pmcmc:hogwild}, respectively, execute more quickly by leveraging
data parallelism.
These transition operators are approximate in that they do not admit the exact
target distribution as a stationary distribution: instead, the stationary
distribution is only intended to be close to the target.
Framed in terms of Monte Carlo estimates, these approximations effectively
accelerate the execution of the chain at the cost of introducing an asymptotic
bias.
Figure~\ref{fig:asy_biased_mcmc} illustrates this new tradeoff.

\begin{figure}[tp]
  \centering
  \includegraphics[width=\linewidth]{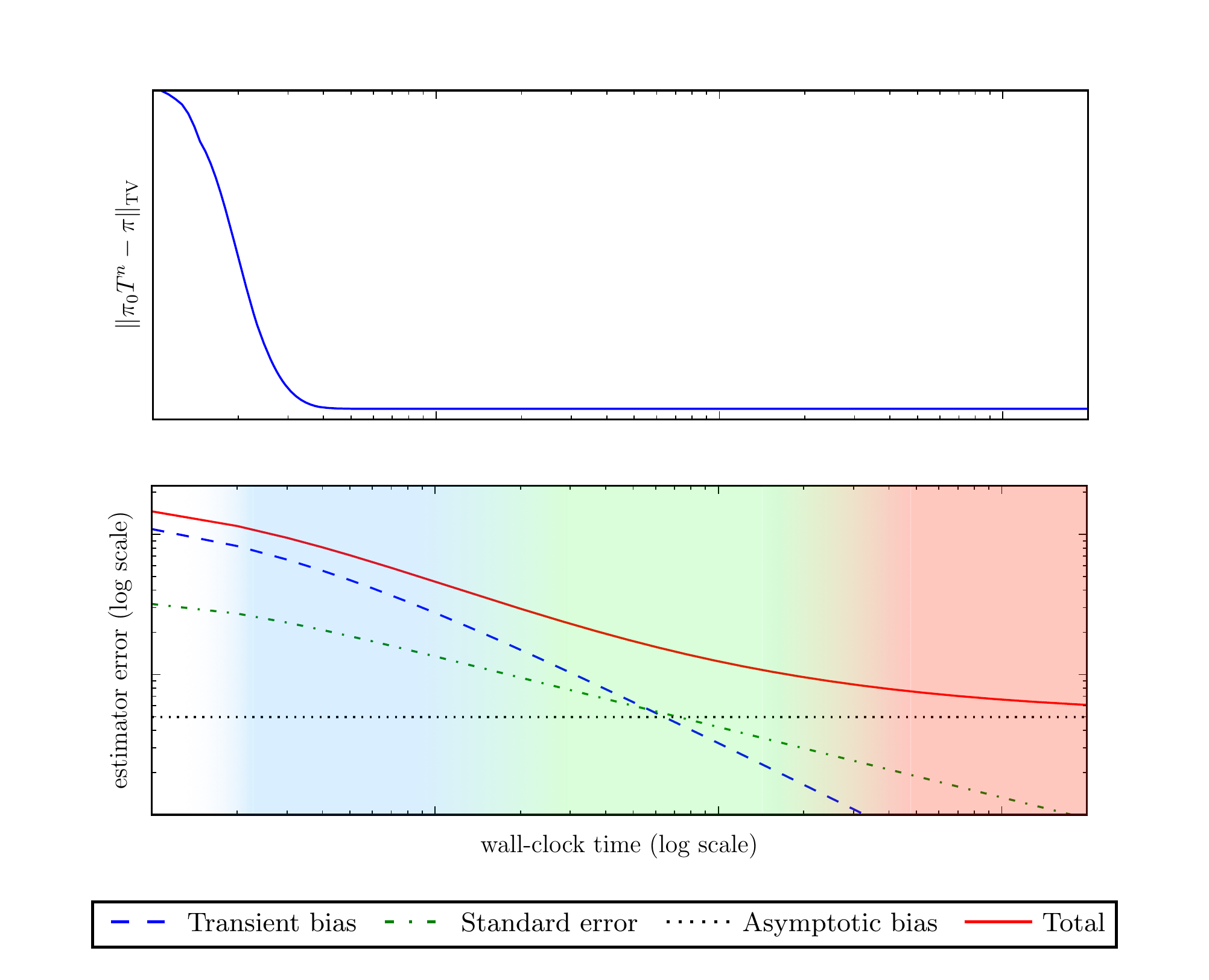}
  \caption{A simulation illustrating the new tradeoffs in some proposed
  scalable MCMC methods. Compare to Figure~\ref{fig:traditional_mcmc}.
  As a function of wall-clock time (log scale), the Markov chain iterations
  execute more than an order of magnitude faster, and hence the marginal
  distributions of the Markov chain iterates converge to the stationary
  distribution more quickly; however, because the stationary distribution is
  not the target distribution, an asymptotic bias remains (top panel).
  Correspondingly, MCMC estimator error, particularly the transient bias, can be
  driven to a small value more rapidly, but there is an error floor due to the
  introduction of the asymptotic bias (bottom panel).}
  \label{fig:asy_biased_mcmc}
\end{figure}

Allowing some asymptotic bias to reduce transient bias or even Monte Carlo
variance is likely to enable MCMC inference at a new scale.
However, both the amount of asymptotic bias introduced by these methods and the
ways in which it depends on model and algorithm parameters remain unclear.
More theoretical understanding and empirical study is necessary to guide
machine learning practice.

\paragraph{Scaling limits of Bayesian inference}
Scalability in the context of Bayesian inference is ultimately about spending
computational resources to better interrogate posterior distributions.
It is therefore important to consider whether there are fundamental limits to
what can be achieved by, \eg spending more money on Amazon EC2, for either
faster computers or more of them.

In parallel systems, linear scaling is ideal: twice as much computational power
yields twice as much useful work.
Unfortunately, even if this lofty parallel speedup goal is achieved, the
asymptotic picture for MCMC is dim: in the asymptotic regime, doubling the
number of samples collected can only reduce the Monte Carlo standard error by a
factor of~$\sqrt{2}$.
This scaling means that there are diminishing returns to purchasing additional
computational resources, even if those resources provide linear speedup in
terms of accelerating the execution of the MCMC algorithm.

Interestingly, variational methods may not suffer from such intrinsic limits.
In particular, the stochastic gradient variational inference methods surveyed
in Section~\ref{sec:stochastic} can utilize optimization methods that, at least
in smooth convex settings, can converge at least at~$\mathcal{O}(\frac{1}{n})$
rates~\citep{bubeckconvex}.
When these convergence properties are maintained for achieving local minima in
non-convex problems, applying additional computational resources would not
inherently suffer from the problem of diminishing marginal returns.

\paragraph{Measuring performance}
With all the ideas surveyed here, one thing is clear: there are many
alternatives for how to scale Bayesian inference.
How should we compare these alternative algorithms?
Can we tell when any of these algorithms work well in an absolute sense?

One standard approach for evaluating MCMC procedures is to define a set of
scalar-valued test functions (or estimands of interest) and compute effective
sample size \citep[Section 11.5]{gelman:2013-bda} as a function of wall-clock
time.
However, in complex models designing an appropriately comprehensive set of test
functions may be difficult.
Furthermore, many such measures require the Markov chain to mix and do not
account for any asyptotic bias \citep{gorham2015measuring}, hence limiting
their applicability to measuring the performance of many of the new inference
methods studied here.

To confront these challenges, one recently-proposed approach
\citep{gorham2015measuring} draws on Stein's method, classically used as an
analytical tool, to design an efficiently-computable measure of discrepancy
between a target distribution and a set of samples.
A natural measure of discrepancy between a target density $p(x)$ and a
(weighted) sample distribution $q(x)$, where $q(x) = \sum_{i=1}^n w_i
\delta_{x_i}(x)$ for some set of samples $\{x_i\}_{i=1}^n$ and weights
$\{w_i\}_{i=1}^n$, is to consider their largest absolute difference across a
large class of test functions:
\begin{equation}
  d_{\mathcal{H}}(q,p) = \sup_{h \in \mathcal{H}} | \E_q h(X) - \E_p h(X) |
  \label{eq:ipm}
\end{equation}
where $\mathcal{H}$ is the class of test functions.
While expectations with respect to the target density $p$ may be difficult to
compute, by designing $\mathcal{H}$ such that $\E_p h(X) = 0$ for every $h \in
\mathcal{H}$, we need only compute expectations with respect to the sample
distribution $q$.
To meet this requirement, instead of designing $\mathcal{H}$ directly, we can
instead choose $\mathcal{H}$ to be the image of another function class
$\mathcal{G}$ under an operator $\mathcal{T}_p$ that may depend on $p$, so that $\mathcal{H} =
\mathcal{T}_p \mathcal{G}$ and the requirement becomes $\E_p (\mathcal{T}_p g)(x) =
0$ and the discrepancy measure becomes
\begin{equation}
  d_{\mathcal{T}_p \mathcal{G}}(q,p) = \sup_{g \in \mathcal{G}}
  | \E_q (\mathcal{T}_p g)(X) |.
  \label{eq:ipm2}
\end{equation}
Such operators $\mathcal{T}_p$ can be designed using infinitessimal generators from
continuous-time ergodic Markov processes, and \citet{gorham2015measuring}
suggest using the operator
\begin{equation}
  (\mathcal{T}_p g)(x) \triangleq \langle g(x), \nabla \log p(x) \rangle + \langle \nabla, \nabla g(x) \rangle
\end{equation}
which requires computing only the gradient of the target log density.
Furthermore, while the optimization in~\eqref{eq:ipm2} is infinite-dimensional in general
and might have infinitely many smoothness constraints from $\mathcal{G}$,
\citet{gorham2015measuring} shows that for the sample distribution $q$ the test
function $g$ need only be evaluated at the finitely-many sample points $\{x_i\}_{i=1}^n$ and
that only a small number of constraints must be enforced.
This new performance metric does not require assumptions on whether the samples
are generated from an unbiased, stationary Markov chain, and so it may provide
clear ways to compare across a broad spectrum sampling-based approximate
inference algorithms.

Another recently-proposed approach attempts to estimate or bound the KL
divergence from an algorithm's approximate posterior representation to the true
posterior, at least when applied to synthetic data.
This approach, called bidirectional Monte Carlo (BDMC)
\citep{grosse2015sandwiching}, can be applied to measure the performance of
both variational mean field algorithms as well as annealed importance sampling
(AIS) and sequential Monte Carlo (SMC) algorithms.
By rearranging the variational identity~\eqref{eq:kl_div_def}, we can write the
KL divergence $\KL(q \| p)$ from an approximating distribution $q(z, \theta)$ to a target
posterior $p(z, \theta \given \bar{y})$ in terms of the log marginal likelihood $\log p(\bar{y})$ and an expectation with respect to $q(z, \theta)$:
\begin{equation}
  \label{eq:bdmc_inference}
  \KL(q \| p) = \log p(\bar{y}) - \E_{q(z, \theta)} \left[ \log \frac{p(z, \theta \given \bar{y})}{q(z, \theta)}  \right].
\end{equation}
Because the expectation can be readily computed in a mean field setting or
stochastically lower-bounded when using AIS \citep[Section
4.1]{grosse2015sandwiching}, with a stochastic upper bound on $\log p(\bar{y})$
we can use~\eqref{eq:bdmc_inference} to compute a stochastic upper bound on
the KL divergence $\KL(q \| p)$.
BDMC provides a method to compute such stochastic upper bounds on $\log p(\bar{y})$
for synthetic datasets $\bar{y}$, and so may enable new performance metrics
that apply to both sampling-based algorithms as well as variational mean field
algorithms.
However, while MCMC transition operators are used to construct AIS algorithms,
BDMC does not directly apply to evaluating the performance of such transition
operators in standard MCMC inference.

Developing performance metrics and evaluation procedures is critical to making
progress.
As observed in~\citet{grosse2015sandwiching},
\begin{quote}
  In many application areas of machine learning, especially supervised
  learning, benchmark datasets have spurred rapid progress in developing new
  algorithms and clever refinements to existing algorithms.
  [\dots]
  So far, the lack of quantitative performance evaluations in marginal
  likelihood estimation, and in sampling-based inference more generally, has
  left us fumbling around in the dark.
\end{quote}
By developing better ways to measure the performance of these Bayesian
inference algorithms, we will be much better equipped to compare, improve, and
extend them.

\subsection*{Acknowledgements}
This work was funded in part by NSF IIS-1421780 and the Alfred P.\ Sloan Foundation.
E.A. is supported by the Miller Institute for Basic Research in Science, University of California, Berkeley.
M.J. is supported by a fellowship from the Harvard/MIT Joint Grants program.

\backmatter  %

\bibliographystyle{plainnat}
\bibliography{main}

\end{document}

%% file: chapters/intro/main.tex
\chapter{Introduction}

We have entered a new era of scientific discovery, in which computational
insights are being integrated with large-scale statistical data analysis to
enable researchers to ask both grander and more subtle questions about our
natural world. This viewpoint asserts that we need not be
limited to the narrow hypotheses that can be framed by traditional small-scale
analysis techniques. Supporting new kinds of data-driven queries, however, requires that new
methods be developed for statistical inference that can \emph{scale up} along
multiple axes --- more samples, more dimensions, and greater model complexity
--- as well as \emph{scale out} by taking advantage of modern parallel
compute environments.

There are a variety of methodological frameworks for performing statistical
inference, e.g., performing estimation and evaluating hypotheses; here we are
concerned with the Bayesian formalism.  In the Bayesian setting, queries about
structure in data are framed as interrogations of the posterior distribution
over parameters, missing data, and other unknowns; these unobserved quantities
are treated as random variables.  By conditioning on the data, the Bayesian
hopes to not only perform point estimation, but also to understand the
uncertainties associated with those estimates.

Accounting for uncertainty is central to Bayesian analysis, and so the
computations associated with most common tasks -- e.g., estimation, prediction,
evaluation of hypotheses -- are typically integrations.  In some situations, it
is possible to perform such integrations exactly, either by taking advantage of
conjugate structure in the prior-likelihood pair, or by using dynamic
programming when the dependencies between random variables are appropriately simple.  Unfortunately, most real-world analysis problems are not
amenable to these exact inference procedures and so most of the interest in
Bayesian computation focuses on better methods of approximate inference.

There are two dominant paradigms for approximate inference in Bayesian
models: Monte Carlo sampling methods and variational approximations.  The Monte
Carlo approach observes that integrations performed to query posterior
distributions can be framed as expectations, and thus estimated with samples;
such samples are most often generated via simulation from carefully designed
Markov chains.  Variational inference seeks to compute these integrals by
approximating the posterior distribution with a more tractable alternative,
where identification of the best approximation can then be performed using powerful
optimization techniques.

In this paper, we examine how these techniques can be scaled up to larger
problems and scaled out across parallel computational resources.  This is
not intended to be an exhaustive survey of a rapidly-evolving area of research;
rather, we seek to identify the main ideas and themes that are emerging in this
area, and articulate what we believe are some of the significant open questions
and challenges.

\section{Why be Bayesian with big data?}

The Bayesian paradigm is fundamentally about integration: integration computes
posterior estimates and measures of uncertainty, eliminates nuisance variables
or missing data, and averages models to compute predictions or perform model
comparison.
While some statistical methods, such as MAP estimation, can be described from a
Bayesian perspective, in which case the prior might serves as a regularizer in
an optimization problem, such methods are not inherently or exclusively
Bayesian.
Posterior integration is the distinguishing characteristic of Bayesian
statistics, and so a defense of Bayesian ideas in the big data regime
rests on the utility of integration.

But from a classical perspective, the big data setting might seem to be
precisely where integration isn't important: as the dataset grows, shouldn't
the posterior distribution concentrate towards a point mass?
If big data means we end up making predictions with such concentrated posteriors,
why not focus on point estimation and avoid the specification of priors and the
burden of approximate integration?

These objections certainly apply to settings where the number of parameters is
small and fixed (``tall data'').
However, many models of interest have many parameters (``wide data''), or
indeed have a number of parameters that grows along with the amount of data.

For example, an Internet company making inferences about its users' viewing and
buying habits may have terabytes of data in total but only a few observations
for its newest customers, the ones most important to impress with personalized
recommendations.  Moreover, it may wish to adapt its model in an online way as
data arrive, a task that benefits from calibrated posterior uncertainties
\citep{stern2009matchbox}.
As another example, consider a healthcare company.  As its dataset grows,
it might hope to make more detailed and complex inferences about populations while
also making careful predictions with calibrated uncertainty for each patient,
even in the presence of massive missing data \citep{lawrence2015missing}.
These scaling issues also arise in astronomy, where hundreds of billions of
light sources, such as stars, galaxies, and quasars, each have latent variables
that must be estimated from very weak observations, and are coupled in a large
hierarchical model \citep{regier2015celeste}.
In Microsoft Bing's sponsored search advertising, predictive probabilities
inform the pricing in the keyword auction mechanism.  This problem nevertheless
must be solved at scale, with tens of millions of impressions per hour
\citep{graepel2010web}.

These are the regimes where big data can be small \citep{lawrence2015missing}
and the number and complexity of statistical hypotheses grows with the data.
The Bayesian methods we survey in this paper may provide solutions to these
challenges.

\section{The fidelity of approximate integration}

Bayesian inference may be important in some modern big data regimes, but exact
integration in general is computationally out of reach.
While decades of research in Bayesian inference in both statistics and machine
learning have produced many powerful approximate inference algorithms, the big
data setting poses some new challenges.
Iterative algorithms that read the entire dataset before making each update
become prohibitively expensive.
Sequential computation that cannot leverage parallel and distributed computing
resources is at a significant and growing disadvantage.
Insisting on zero asymptotic bias from Monte Carlo estimates of expectations
may leave us swamped in errors from high
variance~\citep{korattikara-2014-austerity} or transient bias.

These challenges, and the tradeoffs that may be necessary to address them, can
be viewed in terms of how accurate the integration in our approximate inference
algorithms must be.
Markov chain Monte Carlo (MCMC) algorithms that admit the exact posterior as
a stationary distribution may be the gold standard for generically estimating
posterior expectations, but if standard MCMC algorithms become intractable in
the big data regime we must find alternatives and understand their tradeoffs.
Indeed, someone using Bayesian methods for machine learning may be less
constrained than a classical Bayesian statistician: if the ultimate goal is to form
predictions that perform well according to a specific loss function,
computational gains at the expense of the internal posterior representation may
be worthwhile.
The methods studied here cover a range of such approximate integration
tradeoffs.

\section{Outline}

The remainder of this review is organized as five chapters.
In Chapter~\ref{ch:background}, we provide relevant background material on
exponential families, MCMC inference, mean field variational inference,
and stochastic gradient optimization.
The next three chapters survey recent algorithmic ideas for scaling Bayesian inference,
highlighting theoretical results where possible.
Each of these central technical chapters ends with a summary and discussion,
identifying emergent themes and patterns as well as open questions.
Chapters~\ref{sec:mcmc-subsets} and~\ref{sec:mcmc-parallel} focus on
MCMC algorithms, which are inherently serial and often slow to converge;
the algorithms in the first of these use various forms of data subsampling
to scale up serial MCMC and in the second use a diverse array of strategies
to scale out on parallel resources.
In Chapter~\ref{ch:mean-field} we discuss two recent techniques for scaling
variational mean field algorithms.  Both process data in minibatches:
the first applies stochastic gradient optimization methods and
the second is based on incremental posterior updating.
Finally, in Chapter~\ref{ch:discussion} we provide an overarching discussion of
the ideas we survey, focusing on challenges and open questions in
large-scale Bayesian inference.

%% file: chapters/background/main.tex
\chapter{Background}
\label{ch:background}

In this chapter we summarize background material on which the ideas in
subsequent chapters are based.
This chapter also serves to fix some common notation.
Throughout the chapter, we avoid measure-theoretic definitions and instead
assume that any density exists with respect either to Lebesgue measure or
counting measure, depending on its context.

First, we cover some relevant aspects of exponential families.
Second, we cover the foundations of Markov chain Monte Carlo (MCMC) algorithms,
which are the workhorses of Bayesian statistics and are common in Bayesian
machine learning.
Indeed, the algorithms discussed in Chapters 3 and 4 are either MCMC algorithms
or aim to approximate MCMC algorithms.
Next, we describe the basics of mean field variational inference and stochastic
gradient optimziation, both of which are used extensively in Chapter 5.
Finally, we close the chapter with notes on computational architectures and
useful notions for measuring performance.

\section{Exponential families}
Exponential families of densities play a key role in Bayesian analysis and many
practical Bayesian methods.
In particular, likelihoods that are exponential families yield natural conjugate
prior families,  which provide analytical and computational advantages in both
MCMC and variational inference algorithms.
Exponential families are also particularly relevant in the context of large
datasets: in a precise sense, they are the only families of densities which
admit a finite-dimensional sufficient statistic.
Thus only exponential families allow arbitrarily large amounts of data to be
summarized with a fixed-size description.

In this section we give basic definitions, notation, and results concerning
exponential families.
For perspectives from convex analysis see~\citet{wainwright2008graphical}, and
for perspectives from differential geometry see~\citet{amari2007methods}.

Exponential families are defined in terms of densities with respect to some
underlying $\sigma$-finite measure, which we denote~$\nu$.

\begin{definition}[Exponential family]
    We say a parameterized family of densities $\{ p( \, \cdot \, | \theta) :
    \theta \in \Theta \}$ is an
    \emph{exponential family} if each density can be written as
    \begin{equation}
        p(x | \theta) = h(x) \exp \{ \langle \eta(\theta), t(x) \rangle - \log Z(\eta(\theta)) \}
        \label{background:eq:exp_fam}
    \end{equation}
    where $\langle \cdot, \cdot \rangle$ is an inner product on a
    finite-dimensional real vector space.
    We call $\eta(\theta)$ the \emph{natural parameter} vector, $t(x)$ the
    \emph{statistic} vector, $h(\cdot)$ the \emph{base density}, and
    \begin{equation}
        \log Z(\eta) \triangleq \log \int e^{\langle \eta, t(x) \rangle} h(x) \nu(dx)
    \end{equation}
    the \emph{log partition function}.
\end{definition}

We restrict our attention to families for which the support of the density does
not depend on $\theta$.
When $\eta(\theta) = \theta$ we say the family is written in \emph{natural
    parameters} or \emph{natural coordinates}, which we denote by writing $p(x | \eta)$.
We say a family is \emph{regular} if $\Theta$ is open, and \emph{minimal} if
there is no nonzero $a$ such that $\langle a, t(x) \rangle$ is equal to a
constant ($\nu$-a.e.).

The statistic $t$ is \emph{sufficient} in the sense of the Fisher-Neyman
Factorization Theorem \citep[Theorem 3.6]{keener:2010} by construction
\begin{align*}
p(x|\theta) &\propto h(x) \exp \{\langle \eta(\theta), t(x) \},
\end{align*}
and hence $t(x)$ contains all the information about $x$ that is relevant for
the parameter $\theta$.
In the context of Bayesian analysis, in which $\theta$ is a random variable,
this definition of sufficiency is equivalent to the conditional independence
statement ${\theta \indep X \; | \; t(X)}$.
The Koopman-Pitman-Darmois Theorem shows that exponential families are the only
families which provide this powerful summarization property, under some mild
regularity conditions \citep{hipp:1974-sufficient}.

Exponential families have many convenient analytical and computational properties.
In particular, differentiating the log partition function ${\log Z}$ generates cumulants:

\begin{proposition}[Mean mapping and cumulants]
\label{background:prop:gradient_logpartition}
    For a regular exponential family of densities of the form~\eqref{background:eq:exp_fam} with~${ X \sim p(\; \cdot \; | \eta) }$, we have~${ \nabla \log Z: \Theta \to \mathcal{M} }$ and
    \begin{equation}
        \nabla \log Z(\eta) = \E[t(X)]
        \label{background:eq:gradient_expectation}
    \end{equation}
    and writing $\mu \triangleq \E[t(X)]$ we have
    \begin{equation}
        \nabla^2 \log Z(\eta) = \E[t(X) t{(X)}^\T] - \mu \mu^\T.
    \end{equation}
    More generally, the moment generating function of $t(X)$ can be written
    \begin{equation}
        \MGF_{t(X)}(s) \triangleq \E[e^{\langle s , t(X) \rangle}] = e^{\log Z(\eta + s) - \log Z(\eta)}.
    \end{equation}
    and so derivatives of $\log Z$ give \emph{cumulants}
    of $t(X)$, where the first cumulant is the mean and the second and third
    cumulants are the second and third central moments, respectively.
\end{proposition}
\begin{proof}
    To show~\ref{background:eq:gradient_expectation}, we write
    \begin{align}
        \nabla_\eta \log Z(\eta)
        &= \nabla_\eta \log \int e^{\langle \eta , t(x) \rangle} h(x) \nu(dx)
        \\
        &= \frac{1}{\int e^{\langle \eta , t(x) \rangle} h(x) \nu(dx)}
        \int t(x) e^{\langle \eta, t(x) \rangle} h(x) \nu(dx)
        \\
        &= \int t(x) p(x|\eta) \nu(dx)
        \\
        &= \E[t(X)].
    \end{align}
    To derive the form of the moment generating function, we write
    \begin{align}
        \E[e^{\langle s , t(X) \rangle}]
        &= \int e^{\langle s, t(x) \rangle} p(x) \nu(dx)
        \\
        &= \int e^{\langle s, t(x) \rangle} e^{\langle \eta, t(x) \rangle - \log Z(\eta)} h(x) \nu(dx)
        \\
        &= e^{\log Z(\eta + s) - \log Z(\eta)}.
    \end{align}
\end{proof}

For members of an exponential family, many quantities can be expressed
generically in terms of the natural parameter, expected statistics under that
parameter, and the log partition function.

\begin{proposition}[Score and Fisher information]
\label{background:prop:expfam_props}
    For a regular exponential family in natural coordinates, with $X \sim
    p(\; \cdot \; | \eta)$ and $\mu(\eta) \triangleq \E[t(X)]$ we have
    \begin{enumerate}[rightmargin=1cm]
        \item When the family is regular, the \emph{score} with respect to the natural parameter is
            \begin{equation}
                \label{background:eq:expfam_gradient}
                v(x,\eta) \triangleq \nabla_\eta \log p(x | \eta) = t(x) - \mu(\eta)
            \end{equation}
        \item When the family is regular, the \emph{Fisher information} with
            respect to the natural parameter is
            \begin{equation}
                \mathcal{I}(\eta) \triangleq \E[v(X,\eta) v{(X,\eta)}^\T] = \nabla^2 \log Z(\eta).
                \label{background:eq:expfam_fisher}
            \end{equation}
    \end{enumerate}
\end{proposition}

\begin{proof}
    Each follows from~\eqref{background:eq:exp_fam} and Proposition~\ref{background:prop:gradient_logpartition}.
\end{proof}

Below, we define a notion of conjugacy for pairs of families of distributions.
Conjugate families are especially useful for Bayesian analysis and algorithms.

\begin{definition}
A parameterized (not necessarily exponential) family of
densities~${\mathcal{F} = \{ p(\cdot | \alpha) : \alpha \in \mathcal{A} \}}$ is
\emph{conjugate} for a likelihood function~${p(x | \cdot)}$ if
for every density~$p(\cdot | \alpha)$ in~$\mathcal{F}$ the posterior distribution
\begin{equation}
p(\theta | \alpha') \propto p(\theta | \alpha) p(x | \theta)
\end{equation}
also belongs to~$\cal{F}$, for some ${\alpha' = \alpha'(x, \alpha)}$ that may
depend on~$x$ and~$\alpha$.
\end{definition}

Conjugate pairs are particularly useful in Bayesian analysis because
if we have a prior family~${p(\theta | \alpha)}$ and we observe data generated
according to a likelihood~${p(x | \theta)}$ then the
posterior~${p(\theta | x, \alpha)}$ is in the same family as the prior.
In the context of Bayesian updating, we call~$\alpha$ the hyperparameter
and~$\alpha'$ the posterior hyperparameter.

Given a regular exponential family likelihood, we can always define a
conjugate prior, as shown in the next proposition.

\begin{proposition}
\label{background:prop:expfam_conj}
    Given a regular exponential family
    \begin{align}
        p_{X | \theta}(x | \theta)
        &= h_X(x) \exp \{ \langle \eta_X(\theta) , t_X(x) \rangle - \log Z_X(\eta_X(\theta))\}
        \\
        &= h_X(x) \exp \{ \langle (\eta_X(\theta), -\log Z_X(\eta(\theta))), \; (t_X(x),1) \rangle \}
    \end{align}
    then if we define the statistic $t_\theta(\theta) \triangleq
    (\eta_X(\theta), -\log Z_X(\eta(\theta)))$ and an exponential family of
    densities with respect to that statistic as
    \begin{align}
        p_{\theta | \alpha}(\theta | \alpha) = h_\theta(\theta) \exp \{ \langle \eta_\theta(\alpha), t_\theta(\theta) \rangle - \log Z_\theta(\eta_\theta(\alpha)) \}
    \end{align}
    then the pair $(p_{\theta | \alpha}, p_{X | \theta})$ is a conjugate
    pair of families with
    \begin{align}
        p(\theta | \alpha) p(x | \theta) \propto h_\theta(\theta) \exp \{
        \langle \eta_\theta(\alpha) + (t_X(x), 1), t_\theta(\theta) \rangle \}
    \end{align}
    and hence we can write the posterior hyperparameter as
    \begin{align}
    \alpha' &= \eta_\theta^{-1}(\eta_\theta(\alpha) + (t_X(x), 1))\,.
    \end{align}
    When the prior family is parameterized with its
    natural parameter, we have $\eta' = \eta + (t_X(x),1)$.
\end{proposition}

As a consequence of Proposition~\ref{background:prop:expfam_conj}, if the prior family is written
with natural parameters and we generate data ${\{x_i\}}_{i=1}^n$ according to the model
\begin{align}
    \theta &\sim p_{\theta | \eta}(\; \cdot \; | \eta)
    \\
    x_i | \theta &\iid \sim p_{X | \theta}( \; \cdot \; | \theta)\quad i=1,2,\ldots,n,
\end{align}
where the notation $x_i \iid\sim p(\; \cdot \;)$ denotes that the random
variables $x_i$ are independently and identically distributed,
then $p(\theta | {\{x_i\}}_{i=1}^n, \eta)$ has posterior hyperparameter $\eta' = \eta
+ (\sum_{i=1}^n t(x_i), n)$.
Therefore any tractable computations in the prior, such as simulation or
computing expectations, are shared by the posterior.
Furthermore, for inferences about $\theta$, the entire dataset can be
summarized by the statistic $(\sum_{i=1}^n t(x_i), n)$.

\section{Markov Chain Monte Carlo inference}
\label{sec:mcmc_background}

Markov chain Monte Carlo (MCMC) is a class of algorithms for estimating
expectations with respect to distributions.
These distributions may be intractable, such as most posterior distributions
arising in Bayesian inference.
Given a target distribution, a standard MCMC algorithm proceeds by simulating
an ergodic random walk that admits the target distribution as its stationary
distribution.
As we develop in the following subsections, by collecting samples from the
simulated trajectory and forming Monte Carlo estimates, expectations of many
functions can be approximated to arbitrary accuracy.
Thus MCMC is employed when samples or expectations from a distribution cannot
be obtained directly, as is often the case with complex, high-dimensional
systems arising across disciplines, such as estimating bulk material properties
from molecular dynamics physics simulations or performing inference in
Bayesian probabilistic models.

In this section, we first review the two underlying ideas behind MCMC
algorithms: Monte Carlo methods and Markov chains.
First we define the bias and variance of estimators.
Next, we introduce Monte Carlo estimators based on independent and identically
distributed samples.
We then describe how Monte Carlo estimates can be formed using mutually
dependent samples generated by a Markov chain simulation.
Finally, we introduce two general MCMC algorithms commonly applied to Bayesian
posterior inference, the Metropolis-Hastings and Gibbs sampling algorithms.
Our exposition here mostly follows the standard treatment, such as
in~\citet[Chapter 1]{brooks2011handbook}, \citet{geyer1992practical}, and
\citet{robert2004monte}.

\subsection{Bias and variance of estimators}
Notions of bias and variance are fundamental to understanding and comparing
estimator performance, and much of our discussion of MCMC methods is framed in
these terms.

Consider using a scalar-valued random variable $\hat \theta$ to estimate a
fixed scalar quantity of interest $\theta$.
The bias and variance of the estimator $\hat \theta$ are defined as
\begin{align}
    \text{Bias}[\hat \theta] &= \E[\hat \theta - \theta]
    \\
    \text{Var}[\hat \theta] &= \E[ {(\hat \theta - \E[\hat \theta] )}^2 ].
\end{align}
The mean squared error $\E[{(\hat \theta - \theta)}^2]$ can be decomposed in
terms of the variance and the square of the bias:
\begin{align}
    \E[{(\hat \theta - \theta)}^2] &= \E[{(\hat \theta - \E[\hat \theta] + \E[\hat \theta] - \theta)}^2]
    \\
    &= \E[{(\hat \theta - \E[\hat \theta])}^2] + {(\E[\hat \theta] - \theta)}^2
    \\
    &= \text{Var}[\hat \theta] + \text{Bias}^2[\hat \theta]
\end{align}
This decomposition provides a basic language for evaluating estimators and
thinking about tradeoffs.
Among unbiased estimators, those with lower variance are generally preferrable.
However, when an unbiased estimator has high variance, a biased estimator that
achieves low variance can have a lower overall mean squared error.

As we describe in the following sections, a substantial amount of the study of
Bayesian statistical computation has focused on algorithms that produce
asymptotically unbiased estimates of posterior expectations, in which the
bias due to initialization is transient and is washed out relatively
quickly.
In this setting, the error is typically considered to be dominated by the
variance term, which can be made as small as desired by increasing computation
time without bound.
When computation becomes expensive as in the big data setting, errors under
a realistic computational budget may in fact be dominated by variance, as
observed by~\citet{korattikara-2014-austerity}, or, as we argue in Chapter 6,
transient bias.
Several of the new algorithms we examine in Chapters~\ref{sec:mcmc-subsets}
and~\ref{sec:mcmc-parallel} aim to adjust this tradeoff by
allowing some asymptotic bias while effectively reducing the variance and
transient bias contributions through more efficient computation.

\subsection{Monte Carlo estimates from independent samples}
\label{sec:monte-carlo}

Let $X$ be a random variable with $\E[X] = \mu < \infty$, and let $(X_i : i \in
\naturals)$ be a sequence of i.i.d.\ random variables each with the same distribution
as $X$.
The Strong Law of Large Numbers (LLN) states that the sample average
converges almost surely to the expectation $\mu$ as $n \to \infty$:
\begin{equation}
    \P \left( \lim_{n \to \infty} \frac{1}{n} \sum_{i=1}^n X_i = \mu \right) = 1.
\end{equation}
This convergence immediately suggests the Monte Carlo method: to approximate the
expectation of $X$, which to compute exactly may involve an intractable
integral, one can use i.i.d.\ samples and compute a sample average.
In addition, because for any measurable function $f$ the sequence $(f(X_i) : i
\in \naturals)$ is also a sequence of i.i.d.\ random variables, we can form the
Monte Carlo estimate
\begin{equation}
    \label{background:eq:monte_carlo}
    \E[f(X)] \approx \frac{1}{n} \sum_{i=1}^n f(X_i).
\end{equation}

Monte Carlo estimates of this form are unbiased by construction, and so the
quality of a Monte Carlo estimate can be evaluated in terms of its variance as
a function of the number of samples $n$, which in turn can be understood with
the Central Limit Theorem (CLT), at least in the asymptotic regime.
If $X$ is real-valued and has finite variance $\E[{(X - \mu)}^2] = \sigma^2 <
\infty$, then the CLT states that the deviation~${\frac{1}{n}\sum_{i=1}^n X_i -
\mu}$, rescaled appropriately, converges in distribution and is asymptotically
normal:
\begin{equation}
  \lim_{n \to \infty} \P\left(\frac{1}{\sqrt{n}} \sum_{i=1}^n (X_i - \mu) < \alpha \right) = \P(Z < \alpha)
\end{equation}
where $Z \sim \N(0, \sigma^2)$.
In particular, as $n$ grows, the standard deviation of the sample average
$\frac{1}{n} \sum_{i=1}^n X_i - \mu$ converges to zero at an asymptotic rate proportional to~$\frac{1}{\sqrt{n}}$.
More generally, for any real-valued measurable function $f$, the Monte Carlo
standard error (MCSE) in the estimate~\eqref{background:eq:monte_carlo}
asymptotically scales as $\frac{1}{\sqrt{n}}$ regardless of the dimension of~$X$.

Monte Carlo estimators effectively reduce the problem of computing expectations
to the problem of generating samples.
However, the preceding statements require the samples used in the Monte
Carlo estimate to be independent, and independent samples can be
computationally difficult to generate.
Instead of relying on independent samples, Markov chain Monte Carlo algorithms
compute estimates using mutually dependent samples generated by simulating a
Markov chain.

\subsection{Markov chains}
\label{sec:markov_chains}

Let $\X$ be a discrete or continuous state space and let~$x, x^\prime \in \X$
denote states.
A time-homogeneous Markov chain is a discrete-time stochastic process $(X_t : t
\in \naturals)$ governed by a transition operator~${T(x \rightarrow x')}$ that
specifies the probability density of transitioning to a state $x^\prime$ from a
given state~$x$:
\begin{equation}
    \P(X_{t+1} \in A \given X_t = x) = \int_{A} T(x \rightarrow x^\prime) \, dx^\prime \quad \forall t \in \naturals
\end{equation}
for all measurable sets $A$.
A Markov chain is memoryless in the sense that its future behavior depends only
on the current state and is independent of its past history.

Given an initial density~$\pi_0(x)$ for $X_0$, a Markov chain evolves this
density from one time point to the next through iterative application of
the transition operator.
We write the application of the transition operator to a density $\pi_0$ to yield a new density $\pi_1$ as
\begin{equation}
    \pi_1(x^\prime) = (\pi_0 T)(x^\prime) = \int_{\X} T(x \to x^\prime) \pi_0(x) \; dx.
\end{equation}
Writing $T^t$ to denote $t$ repeated applications of the transition operator~$T$, the density of~$X_t$ induced by~$\pi_0$ and~$T$ is then given by~${\pi_t =
\pi_0 T^t}$.

Markov chain simulation follows this iterative definition by iteratively sampling
the next state using the current state and the transition operator.
That is, after first sampling $X_0$ from $\pi_0(\, \cdot \,)$,
Markov chain simulation proceeds at time step $t$ by sampling $X_{t+1}$
according to the density $T(x_t \rightarrow \, \cdot \,)$ induced by the
fixed sample~$x_t$.

We are interested in Markov chains that converge in total variation to a unique
stationary density $\pi(x)$ in the sense that
\begin{equation}
    \lim_{t \to \infty}  \|   \pi_t - \pi  \| _{\text{TV}} = 0
\end{equation}
for any initial distribution $\pi_0$, where $ \|  \, \cdot \,  \| _{\text{TV}}$
denotes the total variation norm on densities:
\begin{equation}
     \|  p - q  \| _{\text{TV}} = \frac{1}{2} \int_{\X} |p(x) - q(x)| \, dx.
\end{equation}
For a transition operator~$T(x \rightarrow x')$ to admit $\pi(x)$ as a
stationary density, its application must leave~$\pi(x)$ invariant:
\begin{equation}
    \pi = \pi T. %
    \label{eq:stationary}
\end{equation}
For a discussion of general conditions that guarantee a Markov chain converges
to a unique stationary distribution, i.e., that the chain is ergodic,
see~\citet{meyntweedie2009}.

In some cases it is easy to show that a transition operator has a particular
unique stationary distribution.
In particular, it is clear that $\pi$ is the unique stationary distribution when a
transition operator~${T(x \rightarrow x')}$ is \emph{reversible} with respect to~$\pi$, i.e., it satisfies the detailed balance condition with respect to a
density~$\pi(x)$,
\begin{equation}
    T(x \rightarrow x') \pi(x) = T(x' \rightarrow x) \pi(x') \quad \forall x, x' \in \mathcal{X},
    \label{eq:detailed-balance}
\end{equation}
which is a pointwise condition over $\X \times \X$.
Integrating over $x$ on both sides gives:
\begin{align}
\int_\X T(x \rightarrow x') \pi(x) \, dx
&= \int_\X T(x' \rightarrow x) \pi(x') \, dx \nn \\
&= \pi(x') \int_\X T(x' \rightarrow x) \, dx \nn \\
&= \pi(x'), \nn
\end{align}
which is precisely the required condition from~\eqref{eq:stationary}.
We can interpret~\eqref{eq:detailed-balance} as stating that,
for a reversible Markov chain starting from its stationary distribution,
any transition $x \rightarrow x'$ is equilibrated by
the corresponding reverse transition $x' \rightarrow x$.
Many MCMC methods are based on deriving reversible transition operators.

For a thorough introduction to Markov chains, see \citet[Chapter
6]{robert2004monte} and \citet{meyntweedie2009}.

\subsection{Markov chain Monte Carlo (MCMC)}
\label{sec:mcmc}

Markov chain Monte Carlo (MCMC) methods simulate a Markov chain for which the
stationary distribution is equal to a target distribution of interest, and use
the simulated samples to form Monte Carlo estimates of expectations.
That is, consider simulating a Markov chain with unique stationary density
$\pi(x)$, as in Section~\ref{sec:markov_chains}, and collecting its trajectory
into a set of samples ${\{X_i\}}_{i=1}^n$.
These collected samples can be used to form a Monte Carlo estimate for a
function $f$ of a random variable~$X$ with density~$\pi(x)$ via
\begin{equation}
    \label{eq:mcmc_estimate}
    \E[f(X)] = \int_{\mathcal{X}} f(x) \pi(x) \; dx \approx \frac{1}{n} \sum_{i=1}^n f(X_i).
\end{equation}
Even though this Markov chain Monte Carlo estimate is not constructed from
independent samples, it can asymptotically satisfy analogs of the Law of Large
Numbers (LLN) and Central Limit Theorem (CLT) that were used to justify
ordinary Monte Carlo methods in Section~\ref{sec:monte-carlo}.
We sketch these important results here.

The MCMC analog of the LLN states that
\begin{equation}
    \lim_{n \to \infty} \frac{1}{n} \sum_{i=1}^n f(X_i) = \int_{\mathcal{X}} f(x) \, \pi(x) \; dx \quad \text{(a.s.)}
\end{equation}
for all functions $f$ that are absolutely integrable with respect to $\pi$,
i.e.\ all $f: \mathcal{X} \to \reals$ that satisfy $ \int_\mathcal{X} \, |f(x)|
\, \pi(x) \; dx < \infty$.
To quantify the asymptotic variance of MCMC estimates, the analog of the CLT
must take into account both the Markov dependency structure among the samples
used in the Monte Carlo estimate and also the initial state in which the chain
was started.
However, under mild conditions on both the Markov chain and the function $f$,
the sample average for any initial distribution $\pi_0$ is asymptotically
normal in distribution (with appropriate scaling):
\begin{gather}
    \lim_{n \to \infty}
    \P \left( \frac{1}{\sqrt{n}} \sum_{n=1}^n ({f}(X_i) - \mu) < \alpha \right)
    = \P ( Z < \alpha ),
    \\
    Z \sim \N \left(0, \sigma^2 \right),
    \\
    \sigma^2 = \text{Var}_\pi[{f}(X_0)] + 2 \sum_{t=1}^\infty \text{Cov}_\pi[ {f}(X_0), {f}(X_t) ]
\end{gather}
where $\mu = \int_\X f(x) \, \pi(x) \; dx$ and  where $\text{Var}_\pi$ and
$\text{Cov}_\pi$ denote the variance and covariance operators with the chain
$(X_i)$ initialized in stationarity with ${\pi_0 = \pi}$.
Thus standard error in the MCMC estimate also scales asymptotically as
$\frac{1}{\sqrt{n}}$, with a constant that depends on the autocovariance
function of the stationary version of the chain.
See \citet[chapter 17]{meyntweedie2009} and \citet[section 6.7]{robert2004monte}
for precise statements of both the LLN and CLT for Markov chain Monte Carlo
estimates and for conditions on the Markov chain which guarantee that these
theorems hold.

These results show that the asymptotic behavior of MCMC estimates of the
form~\eqref{eq:mcmc_estimate} is generally comparable to that of ordinary Monte
Carlo estimates as discussed in Section~\ref{sec:monte-carlo}.
However, in the non-asymptotic regime MCMC estimates differ from ordinary Monte
Carlo estimates in an important respect: there is a transient bias due to
initializing the Markov chain out of stationarity.
That is, the initial distribution $\pi_0$ from which the first iterate is
sampled is generally not the chain's stationary distribution $\pi$, since if
it were then ordinary Monte Carlo could be performed directly.
While the marginal distribution of each Markov chain iterate converges to the
stationary distribution, the effects of initialization on the initial iterates
of the chain contribute an error term to Eq.~\eqref{eq:mcmc_estimate} in the
form of a transient bias.

This transient bias does not factor into the asymptotic behavior described by
the MCMC analogs of the LLN and the CLT; asymptotically, it decreases at a rate
of at least $\mathcal{O}(\frac{1}{n})$ and is hence dominated by the Monte Carlo
standard error which decreases only at rate $\mathcal{O}(\frac{1}{\sqrt{n}})$.
However, its effects can be significant in practice, especially in
machine learning.
Whenever a sampled chain seems ``unmixed'' because its iterates
are too dependent on the initialization, errors in MCMC estimates are dominated
by this transient bias.

\begin{figure}[tp]
  \centering
  \includegraphics[width=\linewidth]{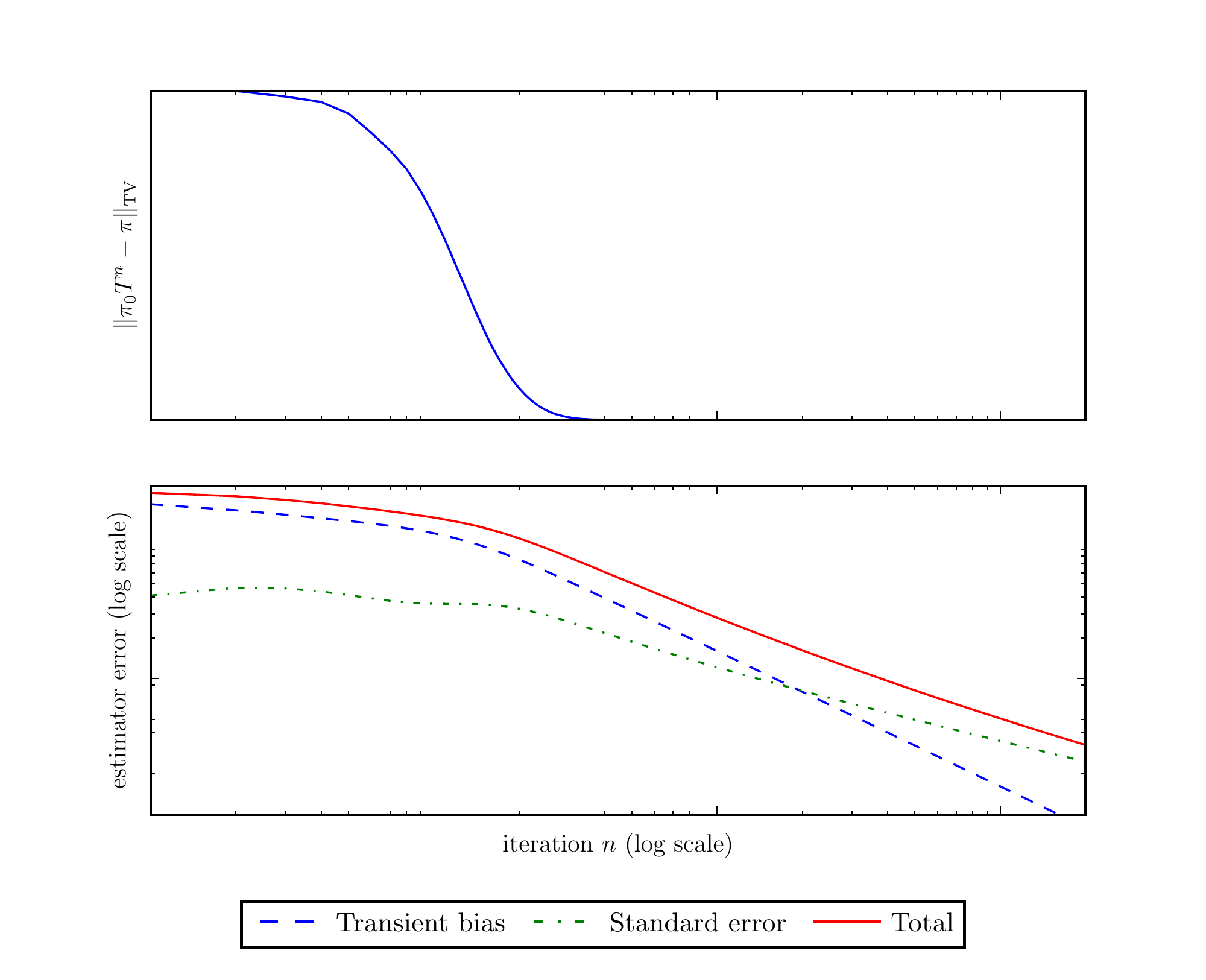}
  \caption{A simulation illustrating error terms in MCMC estimator~\eqref{eq:mcmc_estimate} as a
  function of the number of Markov chain iterations (log scale). The marginal distributions
  of the Markov chain iterates converge to the target distribution (top panel), while the
  errors in MCMC estimates due to transient bias and Monte Carlo standard error
  are eventually driven arbitrarily small at rates of $\mathcal{O}(\frac{1}{n})$
  and $\mathcal{O}(\frac{1}{\sqrt{n}})$, respectively (bottom panel).}
  \label{fig:background:traditional_mcmc}
\end{figure}

The simulation in Figure~\ref{fig:background:traditional_mcmc} illustrates these
error terms in MCMC estimates and how they can behave as more Markov chain
samples are collected.
The LLN and CLT for MCMC describe the regime on the far right of the plot:
the total error can be driven arbitrarily small because the MCMC estimates are
asymptotically unbiased, and the total error is asymptotically dominated by the
Monte Carlo standard error.
However, before reaching the asymptotic regime, the error is often dominated by
the transient initialization bias.
Several of the new methods we survey can be understood as attempts to alter the
traditional MCMC tradeoffs, as we discuss further in Chapter~\ref{ch:discussion}.

Transient bias can be traded off against Monte Carlo standard error by
choosing different subsets of Markov chain samples in the MCMC estimator.
As an extreme choice, instead of using the MCMC
estimator~\eqref{eq:mcmc_estimate} with the full set of Markov chain samples
$\{X_i\}_{i=1}^n$, transient bias can be minimized by forming estimates using
only the last Markov chain sample:
\begin{equation}
    \E[f(X)] \approx f(X_n).
\end{equation}
However, this choice of MCMC estimator maximizes the Monte Carlo standard
error, which asymptotically cannot be decreased below the posterior variance of
the estimand.
A practical choice is to form MCMC estimates using the last $\ceil{n/2}$ Monte
Carlo samples, resulting in an estimator
\begin{equation}
  \E[f(X)] \approx \frac{1}{\ceil{n/2}} \sum_{i=\floor{n/2}}^n f(X_i).
  \label{eq:mcmc_estimate_subsample}
\end{equation}
With this choice, once the marginal distribution of the
Markov chain iterates approaches the stationary distribution the error due to transient bias
is reduced at up to exponential rates.
See Figure~\ref{fig:background:traditional_mcmc_subsample} for an illustration.
With any choice of MCMC estimator, transient bias can be asymptotically decreased at least as
fast as $\mathcal{O}(\frac{1}{n})$, and potentially much faster, while MCSE can
decrease only as fast as $\mathcal{O}(\frac{1}{\sqrt{n}})$.

\begin{figure}[tp]
  \centering
  \includegraphics[width=\linewidth]{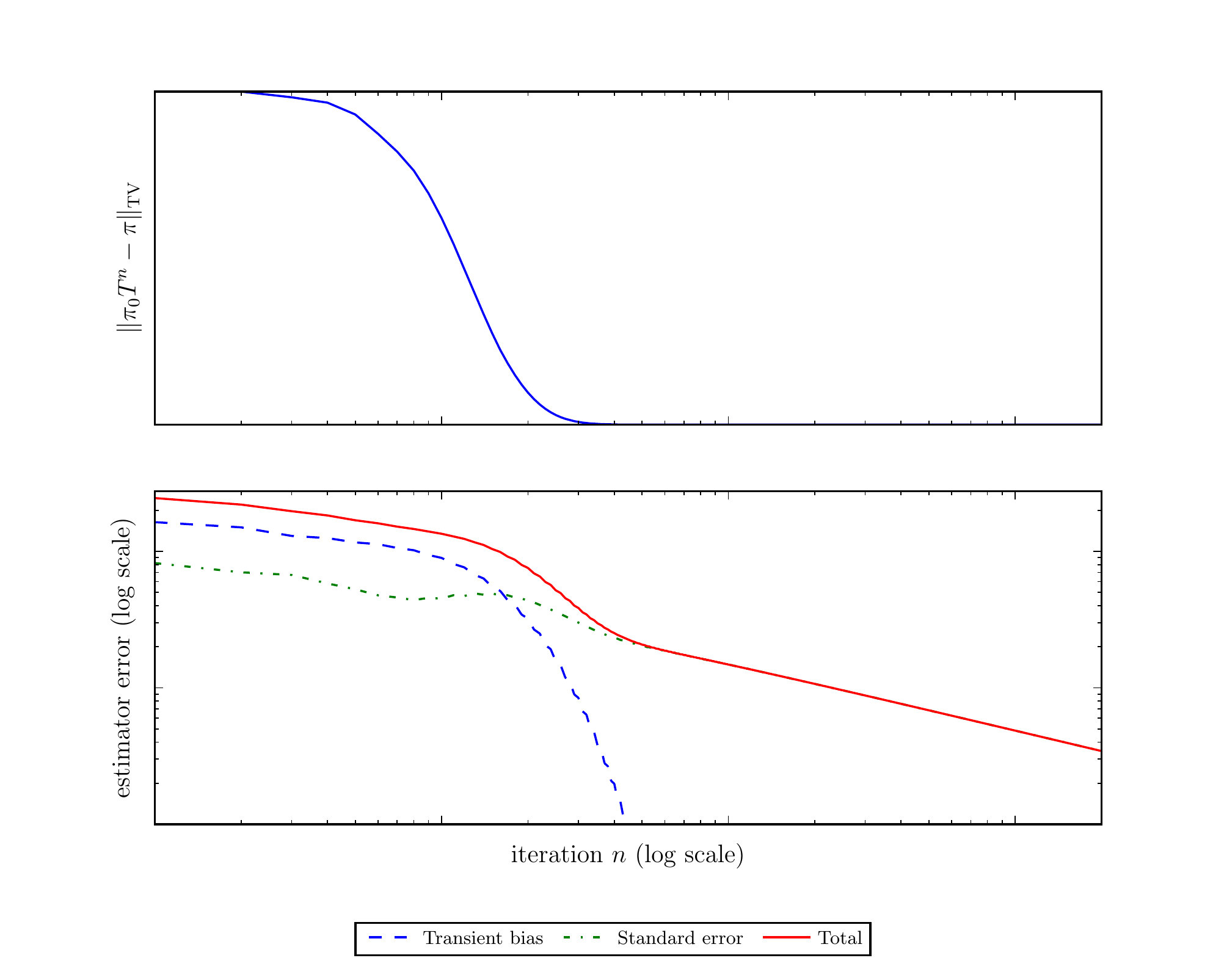}
  \caption{A simulation illustrating error terms in MCMC
  estimator~\eqref{eq:mcmc_estimate_subsample} as a function of the number of
  Markov chain iterations (log scale). Because the first half of the Markov chain
  samples are not used in the estimate, the error due to transient bias is
  reduced much more quickly than in
  Figure~\ref{fig:background:traditional_mcmc} at the cost of shifting up the
  standard error curve.}
  \label{fig:background:traditional_mcmc_subsample}
\end{figure}

Using these ideas, MCMC algorithms provide a general means for estimating
posterior expectations of interest: first construct an algorithm to simulate an
ergodic Markov chain that admits the intended posterior density as its
stationary distribution, and then simply run the simulation, collect samples,
and form Monte Carlo estimates from the samples.
The task then is to design an algorithm to simulate from such a Markov chain
with the intended stationary distribution.
In the following sections, we briefly review two canonical procedures for
constructing such algorithms: Metropolis-Hastings and Gibbs sampling.
For a thorough treatment, see~\citet{robert2004monte}
and~\citet[Chapter 1]{brooks2011handbook}.

\subsection{Metropolis-Hastings (MH) sampling}
\label{sec:mh}

\begin{algorithm}[t]
\caption{Metropolis-Hastings for posterior sampling}
\label{mh}
\begin{algorithmic}
\Require{Initial state $\theta_0$, number of iterations $T$},
joint density $p(\theta, \x)$, proposal density $q(\theta' \given \theta)$
\Ensure{Samples $\theta_1, \dots, \theta_T$}
\For {$t$ in $0, \dots, T-1$}
\State $\theta' \sim q(\theta' \given \theta_t)$ \Comment{Generate proposal}
\State $\alpha(\theta, \theta') \gets \min\left(1, \dfrac{p(\theta', \x) q(\theta_t \given \theta')}{p(\theta_t, \x) q(\theta' \given \theta_t)}\right)$
       \Comment{Acceptance probability}
\State $u \sim \Unif(0, 1)$ \Comment{Set stochastic threshold}
\If {$\alpha(\theta, \theta') > u$}
    \State $\theta_{t+1} \gets \theta'$ \Comment{Accept proposal}
\Else
    \State $\theta_{t+1} \gets \theta_t$ \Comment{Reject proposal}
\EndIf
\EndFor
\end{algorithmic}
\end{algorithm}

In the context of Bayesian posterior inference, the Metropolis-Hastings (MH)
algorithm simulates a reversible Markov chain over a
state space~$\Theta$ that admits the posterior density $p(\theta \given x)$ as its
stationary distribution.
The algorithm depends on a user-specified proposal density, $q(\theta' |
\theta)$, which can be evaluated numerically and sampled from efficiently, and
also requires that the joint density $p(\theta, x)$ can be evaluated (up to
proportionality).
The MH algorithm then generates a sequence of states~$\theta_1, \dots, \theta_T
\in \Theta$ according to Algorithm~\ref{mh}.

In each iteration, a proposal for the next state~$\theta'$ is drawn from the
proposal distribution, conditioned on the current state~$\theta$.
The proposal is stochastically accepted with probability given by the
\emph{acceptance~probability},
\be
\alpha(\theta, \theta') = \min\left(1, \frac{p(\theta', x) q(\theta \given \theta')}{p(\theta, x) q(\theta' \given \theta)}\right),
\ee
via comparison to a random variate~$u$ drawn uniformly from the interval~$[0, 1]$.
If~${u < \alpha(\theta, \theta')}$, then the next state is set to the proposal, otherwise,
the proposal is rejected and the next state is set to the current state.
MH is a generalization of the \emph{Metropolis algorithm}~\citep{metropolis-1953}, 
which requires the proposal distribution to be symmetric,
\ie~${q(\theta' \given \theta) = q(\theta \given \theta')}$, in which case
the acceptance probability is simply
\be
    \min\left(1, \frac{p(\theta', x)}{p(\theta, x)}\right).
\ee
\citet{hastings-1970} later relaxed this by showing that the proposal
distribution could be arbitrary.

One can show that the stationary distribution is indeed~$p(\theta \given x)$ by showing
that the MH transition operator satisfies detailed balance~\eqref{eq:detailed-balance}.
The MH transition operator density is a two-component mixture corresponding to
the `accept' event and the `reject' event:
\begin{align}
T(\theta \rightarrow \theta') &= \alpha(\theta, \theta') q(\theta' \given x) + (1-\beta(\theta)) \delta_{\theta}(\theta') \\
    \beta(\theta) &= \int_\Theta \alpha(\theta, \theta') q(\theta' \given \theta) \; d\theta'.
\end{align}
To show detailed balance, it suffices to show the two balance conditions
\begin{align}
    \label{eq:mh_db1}
    \alpha(\theta, \theta') q(\theta' \given \theta) p(\theta \given x) &= \alpha(\theta', \theta) q(\theta \given \theta') p(\theta' \given x)
    \\
    \label{eq:mh_db2}
    (1-\beta(\theta)) \delta_{\theta}(\theta') p(\theta \given x) &= (1-\beta(\theta')) \delta_{\theta'}(\theta) p(\theta' \given x).
\end{align}
To show~\eqref{eq:mh_db1} we write
\begin{align}
    \alpha(\theta, \theta') q(\theta' \given \theta) p(\theta \given x)
    &= \min\left( 1, \frac{p(\theta', x) q(\theta \given \theta')}{p(\theta, x) q(\theta' \given \theta)} \right)
    q(\theta' \given \theta) p(\theta \given x)
    \notag \\
    &= \min \left( q(\theta' \given \theta) p(\theta \given x), p(\theta', x) q(\theta \given \theta') \frac{p(\theta \given x)}{p(\theta, x)} \right)
    \notag \\
    &= \min \left( q(\theta' \given \theta) p(\theta \given x), p(\theta' \given x) q(\theta \given \theta') \right)
    \notag \\
    &= \min \left( q(\theta' \given \theta) p(\theta, x) \frac{p(\theta' \given x)}{p(\theta', x)}, p(\theta' \given x) q(\theta \given \theta') \right)
    \notag \\
    &= \min \left( 1, \frac{p(\theta, x) q(\theta' \given \theta)}{p(\theta', x) q(\theta \given \theta')} \right) q(\theta \given \theta') p(\theta' \given x).
    \label{eq:mh_db1_proof}
\end{align}
To show~\eqref{eq:mh_db2}, we need to verify that
\begin{align}
	(1-\beta(\theta)) p(\theta
\given x) = (1-\beta(\theta'))p(\theta' \given x)\,,
\end{align}
and we can use the same
manipulation as in~\eqref{eq:mh_db1_proof} under the integral sign:
\begin{align}
    (1-\beta(\theta)) p(\theta \given x) &= \left(1 - \int \alpha(\theta, \theta') q(\theta' \given \theta) \; d\theta' \right) p(\theta \given x)
    \notag \\
    &= \left(1 - \int \alpha(\theta', \theta) q(\theta \given \theta') \; d\theta \right) p(\theta' \given x)
    \notag \\
    &= (1-\beta(\theta')) p(\theta' \given x).
\end{align}

See~\citet[Section 7.3]{robert2004monte} for a more detailed treatment of
the Metropolis-Hastings algorithm.

\subsection{Gibbs sampling}
\label{sec:gibbs}
Given a collection of $n$ random variables $X = \{ X_i : i \in [n]\}$, the Gibbs
sampling algorithm iteratively samples each variable conditioned on the sampled
values of the others.
When the random variables are Markov on a graph $G=(V,E)$, the conditioning can
be reduced to each variable's respective Markov blanket, as in
Algorithm~\ref{background:alg:gibbs}.
In the context of Bayesian inference, the posterior of interest may correspond
to conditioning on some subset of the random variables, fixing them to observed
values.

\begin{algorithm}[t]
    \caption{Gibbs sampling}
    \label{background:alg:gibbs}
    \begin{algorithmic}
        \Require $X$ Markov on graph $G$ with nodes $\{1,2,\ldots,N\}$,
        Markov blankets $\text{MB}_G(i)$ and subroutines to sample $X_i
        \given X_{\MB_G(i)}$ for each $i \in V$
        \Ensure Samples $\{ \hat{x}^{(t)} \}$
        \State Initialize $x=(x_1,x_2,\ldots,x_N)$
        \For{$t=1,2,\ldots$}
        \For{$i=1,2,\ldots,N$}
        \State $x_i \gets \text{ sample } X_i \given X_{\MB_G(i)} =  x_{\MB_G(i)}$
        \EndFor
        \State $\hat{x}^{(t)} \gets (x_1,x_2,\ldots,x_N)$
        \EndFor
    \end{algorithmic}
\end{algorithm}

A variant of the \emph{systematic scan} of Algorithm~\ref{background:alg:gibbs}, in which
nodes are traversed in a fixed order for each outer iteration, is the
\emph{random scan}, in which nodes are traversed according to a random
permutation sampled for each outer iteration. An advantage of the random scan
(and other variants) is that the chain becomes reversible and therefore simpler
to analyze \citep[Section 10.1.2]{robert2004monte}.  With the conditional
independencies implied by a graph, some sampling steps may be performed in
parallel.

The Gibbs sampling algorithm can be analyzed as a special case of the
Metropolis-Hastings algorithm, where the proposal distribution is based on the
conditional distributions and the acceptance probability is always one.
If the Markov chain produced by a Gibbs sampling algorithm is ergodic, then the
stationary distribution is the target distribution of~$X$ \citep[Theorem
10.6]{robert2004monte}.
The Markov chain for a Gibbs sampler can fail to be ergodic if, for example,
the support of the target distribution is disconnected \citep[Example
10.7]{robert2004monte}.
A sufficient condition for Gibbs sampling to be ergodic is that all conditional
densities exist and are positive  everywhere \citep[Theorem
10.8]{robert2004monte}.

For a more detailed treatment of Gibbs sampling theory, see \citet[Chapters 6 and
10]{robert2004monte}.

\begin{notes}
\subsection{Assessing chain convergence and quality}
\label{sec:convergence}

\textbf{The language in this section should be consistent with and build on
the earlier section, ``Tradeoffs in Bayesian inference'' in Chapter 1.}

Several quantities measure the efficiency of a MCMC transition operator
in the asymptotic limit.
\emph{Asymptotic variance} measures the variance of an estimator
using samples from a chain that has reached stationarity.
The \emph{speed of convergence} or \emph{mixing time} of a Markov chain measures
how quickly~$\pi_t(x)$ approaches~$\pi(x)$; it is typically defined with respect to
a distance measure between probability distributions and a threshold.

However, in practice, we work with finite chains and in general do not know the
true stationary distribution.
Thus, we cannot empirically measure asymptotic variance or speed of convergence,
and it is difficult to know whether a chain is converging.
There are several strategies for assessing chain convergence and quality in practice.

\begin{itemize}
\item Discuss methods that rely on knowing the answer.
\item Then there are heuristic tools: Geweke, Gelman-Rubin, ESS.
\end{itemize}

The Gelman-Rubin statistic known as~$\hat{R}$ is a heuristic
for assessing convergence~\citep{gelman:1992-inference};
the description here follows that in the textbook by~\citet{gelman:2013-bda}.
Suppose we run~$S$ separate chains such that each produces~$T$ samples.
Let~${\theta_{ts} \in \reals^d}$ refer to sample~$t$ in chain~$s$.
Let~${\psi_{ts} = f(\theta_{ts})}$ where~${f: \reals^d \rightarrow \reals}$ 
is some scalar function of~$\theta_{ts}$, e.g.\ $f$ could be the log posterior,
or alternatively, the first coordinate of~$\theta_{ts}$.
First, we compute the between-chain variance
\begin{align}
& B = \frac{T}{S-1} \sum_{s=1}^S (\bar{\psi}_{\cdot s} - \bar{\psi}_{\cdot \cdot})^2, \\
\text{where} \quad & \bar{\psi}_{\cdot s} = \frac{1}{T} \sum_{t=1}^T \psi_{ts} 
\quad \mbox{and} \quad  \bar{\psi}_{\cdot \cdot} = \frac{1}{S} \sum_{s=1}^S \bar{\psi}_{\cdot s}, \nn
\end{align}
and the within-chain variance
\be
W = \frac{1}{S} \sum_{s=1}^S \delta_s^2, \quad \mbox{where} \quad
\delta_s^2 = \frac{1}{T-1} \sum_{t=1}^T (\psi_{ts} - \bar{\psi}_{\cdot s})^2.
\ee
Now we can estimate the marginal posterior variance of~$\psi$ as
\be
\nu = \frac{T-1}{T} W + \frac{1}{T} B.
\ee
The estimate of the scale of the distribution of~$\psi$ is then~$\sqrt{\nu}$.
Notice that~${\lim_{T \rightarrow \infty} \nu = W}$.
This makes sense because each chain asymptotically samples from the correct
distribution.
Furthermore,~${\nu > W}$ whenever~${B > W}$, which tends to be true before the
chains have converged, \ie differences between samples from different chains are
greater than differences within chains.
The quantity
\be
\hat{R} = \sqrt{\nu/W}
\ee
estimates the amount by which this scale
would decrease if the simulations were continued to the limit~$T \rightarrow \infty$.
Notice that in this limit,~$\hat{R}$ converges to~$1$.
Furthermore,~$\hat{R}$ tends to decrease toward~$1$, following the above
reasoning about~$\nu$.
A common heuristic is to consider values of~$\hat{R} < 1.1$ as acceptable;
lower cut-off values are considered better.

We can also assess the quality of the samples obtained.
The effective number of samples is defined by~\citet{gelman:1993-bda} as
\be
n_\text{eff} = ST \nu / B.
\ee
Note that~$n_\text{eff}$ does not monotonically decrease as we consider
these shorter subsequences.
\end{notes}

\section{Mean field variational inference}
\label{background:sec:meanfield}

\todo{this section uses notation from the excised graphical models section}

In mean field, and variational inference more generally, the task is to
approximate an intractable distribution, such as a complex posterior, with a
distribution from a tractable family so that the posterior can be efficiently interrogated for estimations of interest.
In this section we define the mean field optimization problem and derive the
standard coordinate optimization algorithm. We also give some basic results on
the relationship between mean field and both graphical model and exponential
family structure. For concreteness and simpler notation, we work mostly with
undirected graphical models; the results extend immediately to directed models.
\todo{In this section, compare to Monte Carlo expectation estimators.}

Mean field inference makes use of several densities and distributions, and so
we use a subscript notation for expectations to clarify the measure used in the
integration when it cannot easily be inferred from context. Given a function $f$ and a
random variable $X$ with range $\mathcal{X}$ and density $p$ with respect to a
base measure $\nu$, we write the expectation of $f$ as
\begin{equation}
    \E_{p(X)} \left[ f(X) \right] = \int_\mathcal{X} f(x) p(x) \nu(dx).
\end{equation}

\begin{proposition}[Mean field variational inequality]
\label{background:prop:var_inequality}
    For a probability density $p$ with respect to a base measure $\nu$ of the form
    \begin{equation}
        p(x) = \frac{1}{Z} \bar{p}(x) \quad \text{with} \quad Z \triangleq \int
        \bar{p}(x) \nu(dx),
        \label{background:eq:vlb_energy}
    \end{equation}
    where~$\bar{p}$ is the unnormalized density, for all densities~$q$ with respect to~$\nu$ we have
    \begin{equation}
        \label{eq:kl_div_def}
        \log Z = \mathcal{L}[q] + \KL(q  \|  p) \geq \mathcal{L}[q]
    \end{equation}
    where
    \begin{align}
        \mathcal{L}[q] &\triangleq \E_{q(X)} \left[ \log \frac{\bar{p}(X)}{q(X)} \right] = \E_{q(X)} \left[ \log \bar{p}(X) \right] + \mathbb{H}[q] \label{background:eq:avg_engy_plus_entropy}\\
        \KL(q \| p) &\triangleq \E_{q(X)} \left[ \log \frac{q(X)}{p(X)} \right].
    \end{align}
    Here,~$\mathbb{H}[q]$ is the differential entropy of~$q$.
\end{proposition}

\begin{proof}
    To show the equality, with $X \sim q$ we write
    \begin{align}
        \mathcal{L}[q] + \KL(q \| p) &= \E_{q(X)} \left[ \frac{\bar{p}(X)}{q(X)} \right] + \E_{q(X)}
        \left[ \log \frac{q(X)}{p(X)} \right] \\
        &= \E_{q(X)} \left[ \log \frac{\bar{p}(X)}{p(X)} \right]\\
        &= \log Z.
    \end{align}
    The inequality follows from the property $\KL(q \| p) \geq 0$, known as Gibbs's
    inequality, which follows from Jensen's inequality and the fact that the
    logarithm is concave:
    \begin{equation}
        -\KL(q \| p) = \E_{q(X)} \left[ \log \frac{q(X)}{p(X)} \right] \leq \log
        \int q(x) \frac{p(x)}{q(x)} \nu(dx) = 0
    \end{equation}
    with equality if and only if $q = p$ ($\nu$-a.e.).
\end{proof}

We call the negative log of $\bar{p}$ in \eqref{background:eq:vlb_energy} the \emph{energy} and $\mathcal{L}[q]$ the
\emph{variational lower bound}, and say $\mathcal{L}[q]$ decomposes into the
entropy minus the average energy as in \eqref{background:eq:avg_engy_plus_entropy}. For
two densities $q$ and $p$ with respect to the same base measure,
$\KL(q \| p)$ is the Kullback-Leibler divergence from $q$ to $p$, used
as a measure of dissimilarity between pairs of densities~\citep{amari2007methods}.

The variational inequality given in Proposition~\ref{background:prop:var_inequality} is
useful in inference because if we wish to approximate an intractable~$p$ with a
tractable~$q$ by minimizing~$\KL(q \| p)$, we can equivalently choose~$q$ to
maximize~$\mathcal{L}[q]$, which is possible to evaluate since it does not include the
partition function $Z$.

In the context of Bayesian inference, $p$ is usually an intractable posterior
distribution of the form $p(\theta | x, \alpha)$, $\bar{p}$ is the unnormalized
joint distribution ${ \bar{p}(\theta) = p(\theta | \alpha) p(x |
\theta) }$, and $Z$ is the marginal likelihood~${ p(x | \alpha) = \int p(x | \theta)
p(\theta | \alpha) \nu(d \theta) }$, which plays a central role in Bayesian model
selection and the minimum description length (MDL) criterion \citep[Chapter
28]{mackay2003information} \citep[Chapter 7]{hastie2009elements}.

Given that graphical model structure can affect the complexity of probabilistic
inference \citep{koller2009probabilistic} it is natural to consider families
$q$ that factor according to tractable graphs.

\begin{definition}[Mean field variational inference]
\label{background:def:meanfield_problem}
    Let $p$ be the density with respect to $\nu$ for a
    collection of random variables $X=(X_i : i \in V)$, and
    let
    \begin{equation}
        \mathcal{Q} \triangleq \{ q : q(x) \propto \prod_{C \in \mathcal{C}}
    q_C(x_C) \}
\end{equation}
    be a family of densities with respect to $\nu$ that factorize according to
    a graph $G=(V,E)$ with $\mathcal{C}$ being the set of maximal cliques of $G$.
    Then the \emph{mean field optimization problem} is
    \begin{equation}
        q^* = \argmax_{q \in \mathcal{Q}} \mathcal{L}[q]%
        \label{background:eq:meanfield_problem}
    \end{equation}
    where $\mathcal{L}[q]$ is defined as in~\eqref{background:eq:avg_engy_plus_entropy}.
\end{definition}

Note that this optimization problem is not in general convex.%
\footnote{In the sense of maximizing a concave objective over a convex set.}
However, when the model distribution is an exponential family the objective is
concave in each $q_C$ individually, and hence an optimization procedure that
updates each factor in turn, while holding the rest fixed, will converge to a local
optimum \citep{wainwright2008graphical} \citep[Section 10.1.1]{bishop} \citep[Section 22.3]{murphy2012machine}.
We call such a coordinate ascent procedure
on~\eqref{background:eq:meanfield_problem} a mean field algorithm.

For approximating families in a factored form, we can derive a generic update
to be used in a mean field algorithm.

\begin{proposition}[Mean field update]
    Given a mean field objective as in Definition~\ref{background:def:meanfield_problem},
    the optimal update to a factor $q_A$ fixing the other factors defined by
    $q_A^* = \argmax_{q_A} \mathcal{L}[q]$ is
    \begin{equation}
        q_A^*(x_A) \propto \exp \{ \E [ \log \bar{p}(x_A, X_{A^c}) ] \}
        \label{background:eq:meanfield_update}
    \end{equation}
    where the expectation is over~${ X_{A^c} \sim q_{A^c} }$ with
    \begin{align}
	q_{A^c}(x_{A^c}) &\propto \prod_{C \in \mathcal{C} \setminus A} q_{C}(x_{C})\,.
	\end{align}
\end{proposition}
\begin{proof}
    Dropping terms constant with respect to $q_A$, we write
    \begin{align}
        q_A^* %
              & = \argmin_{q_A} \KL(q  \|  p) \\
              &= \argmin_{q_A} \E_{q_A} \left[ \log q_A(X_A) \right] + \E_{q_A} \left[ \E_{q_{A^c}} \left[ \log \bar{p}(X) \right] \right] \\
              &= \argmin_{q_A} \KL(q_A  \|  \widetilde{p}_A)
    \end{align}
    where $\widetilde{p}_A(x_A) \propto \exp \{ \E_{q_{A^c}} [ \log
    \bar{p}(x_A,X_{A^c}) ] \}$. Therefore, we achieve the unique
    ($\nu$-a.e.) minimum by setting $q_A = \widetilde{p}_A$.
\end{proof}

Finally, we note the simple form of updates %
for exponential family conjugate pairs.

\begin{proposition}[Mean field and conjugacy]
    If $x_i$ appears in $\bar{p}$ only in an exponential family conjugate pair
    $(p_1,p_2)$ where
    \begin{align}
        p_1(x_i | x_{\pi_G(i)}) &\propto \exp \{ \langle \eta(x_{\pi_G(i)}), t(x_i) \rangle \}
        \label{background:eq:foo1}
        \\
        p_2(x_{c_G(i)} | x_i) &= \exp \{ \langle t(x_i), (t(x_{c_G(i)}), 1) \rangle \}
        \label{background:eq:foo2}
    \end{align}
    then the optimal factor $q_i(x_i)$ is in the prior family with natural parameter
    \begin{equation}
        \widetilde{\eta} \triangleq \E_q[\eta(X_{\pi_G(i)})] + \E_q[(t(X_{c_G(i)}),1)].
    \end{equation}
\end{proposition}

\begin{proof}
    The result follows from substituting~\eqref{background:eq:foo1} and~\eqref{background:eq:foo2}
    into~\eqref{background:eq:meanfield_update}.
\end{proof}

See \citet[Chapter 5]{wainwright2008graphical} for a convex analysis
perspective on mean field algorithms in graphical models composed of
exponential families.

\section{Stochastic gradient optimization}
\label{sec:background:sgd}
In this section we briefly review some basic ideas in stochastic gradient
optimization.
In particular, the basic algorithm we use in this paper is given in
Algorithm~\ref{alg:sgd} and sufficient conditions for its convergence to a
local extreme point are given in Theorem~\ref{thm:sgd}.

Given a dataset $\bar{y}=\{\bar{y}^{(k)}\}_{k=1}^K$, where each $\bar{y}^{(k)}$
is a data \emph{minibatch}, consider the optimization
problem
\begin{equation}
    \phi^* = \argmax_\phi f(\phi)
\end{equation}
where the objective function $f$ decomposes according to
\begin{equation}
    f(\phi) = \sum_{k=1}^K g(\phi,\bar{y}^{(k)}). \label{eq:f-decomposes}
\end{equation}
In the context of variational Bayesian inference, the objective $f$ may be a
variational lower bound on the log of the model evidence and $\phi$ may be the parameters
of the variational family.
In MAP inference, $f$ may be the log joint density and $\phi$ may be its
parameters.

Using the decomposition of $f$, we can compute unbiased Monte Carlo estimates
of its gradient.
In particular, if the random index $\hat{k}$ is sampled from
$\{1,2,\ldots,K\}$, denoting the probability of sampling index~$k$ as~${p_k > 0}$, we
have
\begin{equation}
    \nabla_\phi f(\phi)
    = \sum_{k=1}^K p_k \frac{1}{p_k} \nabla_\phi g(\phi,\bar{y}^{(k)})
    =  \E_{\hat{k}} \left[ \frac{1}{p_{\hat{k}}} \nabla_\phi g(\phi,\bar{y}^{(\hat{k})}) \right].
    \label{eq:gradient-decomposes}
\end{equation}
Thus by considering a Monte Carlo approximation to the expectation over $\hat{k}$,
we can generate stochastic approximate gradients of the objective $f$ using
only a single $\bar{y}^{(k)}$ at a time.

A stochastic gradient ascent algorithm uses these approximate gradients to
perform updates and find a stationary point of the objective.
At each iteration, such an algorithm samples a data minibatch, computes a
gradient with respect to that minibatch, and takes a step in that direction.
In particular, for a sequence of stepsizes $\rho^{(t)}$ and a sequence of
positive definite matrices $G^{(t)}$, a typical stochastic gradient ascent
algorithm is given in Algorithm~\ref{alg:sgd}.

\begin{algorithm}[t]
    \caption{Stochastic gradient ascent}
    \begin{algorithmic}
        \Require $f : \mathbb{R}^n \to \mathbb{R}$ of the form
        \eqref{eq:gradient-decomposes}, sequences $\rho^{(t)}$ and $G^{(t)}$
        \State Initialize $\phi^{(0)} \in \mathbb{R}^n$
        \For{$t=0,1,2,\ldots$}
        \State $\hat{k}^{(t)} \gets$ sample index $k$ with probability $p_k$, for $k=1,2,\ldots,K$
        \State $\phi^{(t+1)} \gets \phi^{(t)} + \rho^{(t)} \frac{1}{p_{\hat{k}}} G^{(t)} \nabla_\phi g(\phi^{(t)},\bar{y}^{(\hat{k}^{(t)})}) $
        \EndFor
    \end{algorithmic}
    \label{alg:sgd}
\end{algorithm}

Stochastic gradient algorithms have very general convergence guarantees,
requiring only weak conditions on the step size sequence and even the accuracy
of the gradients themselves.
We summarize a common set of sufficient conditions in Theorem~\ref{thm:sgd}.
Proofs of this result, along with more general versions, can be found in
\citet{bertsekas1989parallel} and \citet{bottou1998online}.
Note also that while the construction here has assumed that the stochasticity
in the gradients arises only from randomly subsampling a finite sum, more
general versions allow for other sources of stochasticity, typically requiring
only bounded variance and allowing some degree of bias \citep[Section
7.8]{bertsekas1989parallel}.

\begin{theorem}
    Given a function $f : \mathbb{R}^n \to \mathbb{R}$ of the form \eqref{eq:f-decomposes}, if
    \begin{enumerate}
        \item  there exists a constant $C_0$ such that $f(\phi) \leq C_0$ for all $\phi \in \mathbb{R}^n$,
        \item  there exists a constant $C_1$ such that
            \begin{equation}
                 \|  \nabla f(\phi) - \nabla f(\phi')  \| _2 \leq C_1  \|  \phi - \phi'  \| _2 \quad \forall \phi,\phi' \in \mathbb{R}^n, \notag
            \end{equation}
        \item  there are positive constants $C_2$ and $C_3$ such that
            \begin{equation}
                \forall t \; C_2 I \prec G^{(t)} \prec C_3 I, \notag
            \end{equation}
        \item  and the stepsize sequence $\rho^{(t)}$ satisfies
            \begin{equation}
                \sum_{t=0}^\infty \rho^{(t)} = \infty \quad \text{and} \quad \sum_{t=0}^\infty (\rho^{(t)})^2 < \infty, \notag
            \end{equation}
    \end{enumerate}
    then Algorithm~\ref{alg:sgd} converges to a stationary point in the
    sense that
    \begin{equation}
        \liminf_{t \to \infty} \| \nabla f(\phi^{(t)}) \| = 0
    \end{equation}
    with probability $1$.
    \label{thm:sgd}
\end{theorem}

While stochastic optimization theory provides convergence guarantees, there is
no general theory to analyze rates of convergence for nonconvex problems
such as those that commonly arise in posterior inference.
Indeed, the empirical rate of convergence often depends strongly on the
variance of the stochastic gradient updates and on the choice of step size
sequence.
There are automatic methods to tune or adapt the sequence of stepsizes
\citep{snoek2012practical,ranganath2013adaptive}, though we do not discuss them
here.
To make a single-pass algorithm, the minibatches can be sampled without
replacement.

\begin{notes}
\section{Computational architectures and paradigms}

Modern computational resources span a large range of parallel architectures.
Mainstream laptops commonly have at least one processor, each with at least two cores,
plus a graphical processing unit (GPU) containing hundreds or thousands of specialized cores.
Scientific computing clusters and cloud resources, such as Amazon Elastic Compute Cloud (EC2),
tend to be composed of multiple machines, each with tens or hundreds of cores, connected by a network.
Clusters of GPUs are also now common.
Cores on the same machine or GPU can communicate via shared memory,
whereas cores on different machines can only communicate over the network.
Communication via shared memory is minimal compared to communication over a network,
but shared memory architectures tend to have fewer cores
-- and thus capacity for parallel scaling -- than larger multicore compute clusters.

We can think of a process as the basic unit of program execution.
A serial process is executed as a single thread, most simply on a single core.
Multiple concurrent processes are executed as parallel threads on one or more cores.
The computational performance of a parallel execution is most straightforward to
understand when each parallel process or thread is associated with a single physical core.
Note that many modern processors support hyper-threading, where the operating system
associates multiple (two) virtual or logical cores with each physical core.
Instruction pipelining by the operating system can yield performance improvements compared
to traditional execution without hyper-threading, but can be difficult to interpret.
Especially when using cloud resources, it is important to understand the number of
physical versus virtual cores when evaluating a parallel implementation.

Parallel architectures with good computational performance tend to avoid communication
and system overheads, e.g.\ for scheduling worker cores and managing their computations.

\section{Measuring computational performance}

We provide a few notes on terminology for characterizing the effectiveness of parallel systems and algorithms.
The \emph{speedup} of a parallel procedure is a function of the amount of parallel resources~($K$),
e.g.\ the number of parallel cores.
It is measured relative to a serial procedure executed on a single core,
and defined as the (wall-clock) time of the serial procedure divided by
the (wall-clock) time of the parallel procedure.
Typically, speedup is an absolute measure of computational performance relative to a
standard reference implementation of a serial algorithm,
e.g.\ a stand-alone implementation outside of the parallel system being evaluated.
Ideally, speedup is linear in~$K$ and with a constant of proportionality equal to one.
A related measure, \emph{scaling}, is similar to speedup but tends to measure a system's
performance relative to itself, \ie the same system running on a single core, or more
generally, the minimum number of cores required by the system.\footnote{For example, a
parallel system architecture with a master-worker pattern requires at least one master and
one worker, which may correspond to two cores.}
An algorithm may have excellent speedup and scaling properties for an ideal system with
zero communication cost, but implementation and communication overheads tend to limit
practical performance.
Members of the systems community have rightly pointed out issues with researchers'
focus on an implementation's scalability rather than absolute speedup;
we refer interested readers to the recent work by~\citet{mcsherry:2015-cost}.

\end{notes}

%% file: chapters/mcmc-subsets/main.tex
\chapter{MCMC with data subsets}
\label{sec:mcmc-subsets}

In MCMC sampling for Bayesian inference, the task is to simulate a Markov chain
that admits as its stationary distribution the posterior distribution of interest.
While there are many standard procedures for constructing and simulating from
such Markov chains, when the dataset is large many of these algorithms' updates
become computationally expensive.
This growth in complexity naturally suggests the question of whether there are
MCMC procedures that can generate approximate posterior samples without using
the full dataset in each update.
In this chapter, we focus on recent MCMC sampling schemes that scale Bayesian
inference by operating on only subsets of data at a time.

\begin{notes}
Some differences between algorithms:
\begin{itemize}
\item Dynamic versus static methods for selecting data subsets.
\item Data subsets, mini-batches, incremental processing.
\item Exact versus approximate.
\end{itemize}
\end{notes}

\section{Factoring the joint density}

In most Bayesian inference problems, the fundamental object of interest is the
posterior density, which for fixed data is proportional to the product of the
prior and the likelihood:
\be
\pi(\theta \given \x) \propto \pi(\theta, \x) = \pi_0(\theta) \pi(\x \given \theta).
\ee
In this survey we are often concerned with posteriors where the data~$\x =
\{x_n\}_{n=1}^N$ are conditionally independent given the model
parameters~$\theta$, and hence the likelihood can be decomposed into a product
of terms:
\be
\pi(\theta \given \x) \propto \pi_0(\theta) \pi(\x \given \theta) 
                 = \pi_0(\theta) \prod_{n=1}^{N} \pi(x_n \given \theta).
\label{eq:factorization}
\ee
When~$N$ is large, this factorization can be exploited to construct MCMC
algorithms in which the updates depend only on subsets of the data.

In particular, we can use subsets of data to form an unbiased Monte Carlo
estimate of the log likelihood and consequently the log joint density.
The log likelihood is a sum of terms:
\be
\log \pi(\x \given \theta)
= \sum_{n=1}^{N} \log \pi(x_n \given \theta),
\label{eq:log-likelihood}
\ee
and we can approximate this sum using a random subset of~${m < N}$ terms
\be
\log \pi(\x \given \theta)
\approx \frac{N}{m} \sum_{n=1}^m \log \pi(x_n^* \given \theta),
\label{eq:log-likelihood-subset}
\ee
where $\{x_n^*\}_{n=1}^m$ is a uniformly random subset of $\{x_n\}_{n=1}^N$.
This approximation is an unbiased estimator and yields an unbiased estimate of
the log joint density:
\be
\log \pi(\theta) \pi(\x \given \theta) \approx \log \pi_0(\theta) + \frac{N}{m} \sum_{n=1}^m \log \pi(x_n^* \given \theta).
\label{eq:log-posterior-subset}
\ee
Several of the methods reviewed in this chapter exploit this estimator to
perform MCMC updates.

\section{Adaptive subsampling for Metropolis--Hastings}
\label{sec:adaptive}

In traditional Metropolis--Hastings (MH), we evaluate the joint density to
decide whether to accept or reject a proposal.
As noted by~\citet{korattikara-2014-austerity}, because the value of the joint
density depends on the full dataset, when~$N$ is large this is an unappealing
amount of computation to reach a binary decision.
In this section, we survey ideas for using approximate MH tests that depend on
only a subset of the full dataset.
The resulting approximate MCMC algorithms proceed in each iteration by reading
only as much data as required to satisfy some estimated error tolerance.

While there are several variations, the common idea is to model the probability
that the outcome of such an approximate MH test differs from the exact MH test.
This probability model allows us to construct an approximate MCMC sampler,
outlined in Section~\ref{sec:adaptive-stopping-rule}, where the user specifies
some tolerance for the error in an MH test and the amount of data evaluated is
controlled by an \emph{adaptive stopping rule}.
Different models for the MH test error lead to different stopping rules.
\citet{korattikara-2014-austerity} use a normal model to construct
a $t$-statistic hypothesis test, which we describe in Section~\ref{sec:t-test}.
\citet{bardenet:2014-subsampling} instead use concentration inequalities,
which we describe in Section~\ref{sec:concentration}.
Given an error model and resulting stopping rule, both schemes rely on an MH
test based on a Monte Carlo estimate of the log joint density, which we
summarize in Section~\ref{sec:approx-mh-test}.
Our notation in this section follows~\citet{bardenet:2014-subsampling}.

\citet{bardenet:2014-subsampling} observe that similar ideas have been
developed both in the context of simulated annealing\footnote{Simulated annealing
is a stochastic optimization heuristic that is operationally similar to MH.}
by the operations research community~\citep{bulgak:1988,alkhamis:1999,wang:2006},
and in the context of MCMC inference for factor graphs~\citep{singh:2012-mcmcmc}.

\subsection{An approximate MH test based on a data subset}
\label{sec:approx-mh-test}

In the Metropolis--Hastings algorithm~(\S\ref{sec:mh}), the proposal is
stochastically accepted when
\be
\frac{\pi(\theta' \given \x) q(\theta \given \theta')}{\pi(\theta \given \x) q(\theta' \given \theta)} > u,
\ee
where~$u \sim \Unif(0, 1)$.
Rearranging and using log probabilities gives
\be
\log \left[\frac{\pi(\x \given \theta')}{\pi(\x \given \theta)}\right]
> \log \left[u\frac{q(\theta' \given \theta)\pi_0(\theta)}{q(\theta \given \theta')\pi_0(\theta')}\right].
\label{eq:threshold}
\ee
Scaling both sides by~$1/N$ gives an equivalent threshold,
\be
\Lambda(\theta, \theta') > \psi(u, \theta, \theta'),
\label{eq:threshold}
\ee
where on the left, $\Lambda(\theta, \theta')$ is the average log likelihood ratio,
\be
\Lambda(\theta, \theta')
= \frac{1}{N} \sum_{n=1}^N \log \left[\frac{\pi(x_n \given \theta')}{\pi(x_n \given \theta)}\right]
\equiv \frac{1}{N} \sum_{n=1}^N \ell_n,
\ee
where
\be
\ell_n = \log \pi(x_n \given \theta') - \log \pi(x_n \given \theta),
\label{eq:log-likelihood-ratio}
\ee
and on the right,
\be
\psi(u, \theta, \theta')
= \frac{1}{N} \log \left[u\frac{q(\theta' \given \theta)\pi_0(\theta)}{q(\theta \given \theta')\pi_0(\theta')}\right].
\ee
We can form an approximate threshold by subsampling the~$\ell_n$.
Let~$\{\ell_n^*\}_{n=1}^{m}$ be a subsample of size~${m<N}$,
without replacement, from~$\{\ell_n\}_{n=1}^{N}$.
This gives the following approximate test:
\be
\hat\Lambda_m(\theta, \theta') > \psi(u, \theta, \theta'),
\label{eq:approx-threshold}
\ee
where
\be
\hat\Lambda_m(\theta, \theta')
= \frac{1}{m} \sum_{n=1}^m \log \left[\frac{\pi(x_n^* \given \theta')}{\pi(x_n^* \given \theta)}\right]
\equiv \frac{1}{m} \sum_{n=1}^m \ell_n^*\,.
\label{eq:difference}
\ee
This subsampled average log likelihood ratio $\hat\Lambda_m(\theta, \theta')$
is an unbiased estimate of the average log likelihood ratio $\Lambda(\theta,
\theta')$.
However, an error is made in the event that the approximate
test~\eqref{eq:approx-threshold} disagrees with the exact
test~\eqref{eq:threshold}, and the probability of such an error event depends
on the distribution of $\hat\Lambda_m(\theta, \theta')$ and not just its mean.

Note that because the proposal~$\theta'$ is usually a small perturbation
of~$\theta$, we expect~$\log \pi(x_n\given \theta')$ to be similar to~$\log
\pi(x_n \given \theta)$.
In this case, we expect the log likelihood ratios~$\ell_n$ have a smaller variance
compared to the variance of~$\log \pi(x_n \given \theta)$ across data terms.

\subsection{Approximate MH with an adaptive stopping rule}
\label{sec:adaptive-stopping-rule}

\begin{algorithm}[t!]
\caption{Approximate MH with an adaptive stopping rule}
\label{alg:mh-adaptive}
\begin{algorithmic}
\State \textbf{Input:} Initial state~$\theta_0$, number of iterations~$T$,
data~${\x = \{x_n\}_{n=1}^N}$, posterior $\pi(\theta \given \x)$,
proposal $q(\theta' \given \theta)$
\State \textbf{Output:} Samples $\theta_1, \dots, \theta_T$
\For {$t$ in $0, \dots, T-1$}
	\State $\theta' \sim q(\theta' \given \theta_t)$ \Comment{Generate proposal}
	\State $u \sim \Unif(0, 1)$ \Comment{Draw random number}
	\State $\psi(u, \theta, \theta') \gets \dfrac{1}{N} \log \left[u\dfrac{q(\theta' \given \theta)\pi_0(\theta)}{q(\theta \given \theta')\pi_0(\theta')}\right]$
	\State $\hat\Lambda(\theta,\theta') \gets$ \Call{AvgLogLikeRatioEstimate}{$\theta,\theta',\psi(u,\theta,\theta')$}
	\If {$\hat\Lambda(\theta, \theta') > \psi(u, \theta, \theta')$} \Comment{Approximate MH test}
		\State $\theta_{t+1} \gets \theta'$ \Comment{Accept proposal}
	\Else
		\State $\theta_{t+1} \gets \theta_t$ \Comment{Reject proposal}
	\EndIf
\EndFor
\end{algorithmic}
\end{algorithm}

A nested sequence of data subsets, sampled without replacement, that converges to the complete
dataset gives us a sequence of approximate MH tests that converges to the exact MH test.
Modeling the error of such an approximate MH test gives us a mechanism for
designing an approximate MH algorithm in which, at each iteration,
we incrementally read more data until an adaptive stopping rule
informs us that our error is less than some user-specified tolerance.
Algorithm~\ref{alg:mh-adaptive} outlines this approach.
The function \textsc{AvgLogLikeRatioEstimate} computes~$\hat\Lambda(\theta, \theta')$
according to an adaptive stopping rule that depends on an error model,
\ie a way to approximate or bound the probability that the approximate outcome
disagrees with the full-data outcome:
\be
\P\left[((\hat\Lambda_m(\theta, \theta') > \psi(u, \theta, \theta')) \neq
                              ((\Lambda(\theta, \theta') > \psi(u, \theta, \theta'))\right].
\label{eq:prob-wrong}
\ee
We describe two possible error models in Sections~\ref{sec:t-test} and~\ref{sec:concentration}.

A practical issue with adaptive subsampling is choosing the sizes of the data subsets.
One approach, taken by~\citet{korattikara-2014-austerity}, is to use a fixed
batch size~$b$ and read~$b$ more data points at a time.
\citet{bardenet:2014-subsampling} instead geometrically increase the total subsample size,
and also discuss connections between adaptive stopping rules and related ideas such
as bandit problems, racing algorithms and boosting.

\subsection{Using a $t$-statistic hypothesis test}
\label{sec:t-test}

\citet{korattikara-2014-austerity} propose an approximate MH acceptance
probability that uses a parametric test of significance as its error model.
By assuming a normal model for the log likelihood estimate $\hat\Lambda(\theta, \theta')$,
a $t$-statistic hypothesis test then provides an estimate of
whether the approximate outcome agrees with the full-data outcome,
\ie the expression in Equation~\eqref{eq:prob-wrong}.
This leads to an adaptive framework as in Section~\ref{sec:adaptive-stopping-rule}
where, at each iteration, the data are processed incrementally until the $t$-test
satisfies some user-specified tolerance~$\epsilon$.

Let us model the~$\ell_n$ as i.i.d.\ from a normal distribution with
bounded variance~$\sigma^2$:
\begin{align}
  {\ell_n} &\sim \N(\mu, \sigma^2)\,.
\label{model-single}
\end{align}
The mean estimate~$\hat\mu_m$ for~$\mu$ based on the subset of size~$m$
is equal to~$\hat\Lambda_m(\theta, \theta')$:
\begin{align}
  \hat\mu_m &= \hat\Lambda_m(\theta, \theta') = \frac{1}{m}\sum_{n=1}^m \ell_n^*\,.
  \label{eq:mu}
\end{align}
The error estimate~$\hat\sigma_m$ for~$\sigma$ may be derived from~$s_m/\sqrt{m}$,
where~$s_m$ is the empirical standard deviation of the~$m$ subsampled~$\ell_n$ terms, \ie
\be
s_m = \sqrt{\frac{m}{m-1}\left(\hat\Lambda_m^2(\theta, \theta') - \hat\Lambda_m(\theta, \theta')^2\right)},
\label{eq:empirical-std}
\ee
where
\be
\hat\Lambda_m^2(\theta, \theta') = \frac{1}{m} \sum_{n=1}^m (\ell_n^*)^2.
\ee
To obtain a confidence interval, we multiply this
estimate by the finite population correction, giving:
\begin{align}
\hat\sigma_m = \frac{s_m}{\sqrt{m}}  \sqrt{\frac{N - m}{N - 1}}\,.
\label{eq:sigma}
\end{align}
If~$m$ is large enough for the CLT to hold, the test statistic
\be
t = \frac{\hat\Lambda_m(\theta, \theta') - \psi(u, \theta, \theta')}{\hat\sigma_m}
\label{eq:test-statistic}
\ee
follows a Student's $t$-distribution with~${m-1}$ degrees of freedom
when~$\Lambda(\theta, \theta') = \psi(u, \theta, \theta')$.
The tail probability for~$|t|$ then gives the probability that the
approximate and actual outcomes agree, and thus
\be
\rho = 1 - \phi_{m-1}(|t|)
\ee
is the probability that they disagree, where~$\phi_{m-1}(\cdot)$ is the CDF of
the Student's $t$-distribution with~${m-1}$ degrees of freedom.
The $t$-test thus gives an adaptive stopping rule, \ie for any
user-provided tolerance~${\epsilon \ge 0}$, we can incrementally increase~$m$
until~${\rho \le \epsilon}$.
We illustrate this approach in Algorithm~\ref{alg:hypothesis}.

\begin{algorithm}[t!]
\caption{Estimate of the average log likelihood ratio.
The adaptive stopping rule uses a $t$-statistic hypothesis test.}
\label{alg:hypothesis}
\begin{algorithmic}
\State \textbf{Parameters:} batch size~$b$, user-defined error tolerance~$\epsilon$
  \Function{AvgLogLikeRatioEstimate}{$\theta,\theta',\psi(u,\theta,\theta')$}
    \State $m,~\hat\Lambda(\theta,\theta'),~\hat\Lambda^2(\theta,\theta') \gets 0,~0,~0$
    \While{True}
      \State $c \gets \min(b, N-m)$
      \State $\hat\Lambda(\theta, \theta') \gets
        \dfrac{1}{m+c} \left( m \hat\Lambda(\theta, \theta') + \displaystyle\sum_{n=m+1}^{m+c}
              \log\dfrac{\pi(x_n \given \theta')}{\pi(x_n \given \theta)} \right)$
      \State $\hat\Lambda^2(\theta, \theta') \gets
        \dfrac{1}{m+c} \left( m \hat\Lambda^2(\theta, \theta') + \displaystyle\sum_{n=m+1}^{m+c}
              \left[\log\dfrac{\pi(x_n \given \theta')}{\pi(x_n \given \theta)}\right]^2 \right)$
      \State $m \gets m + c$
      \vspace{.5em}
	\State $s \gets \sqrt{\dfrac{m}{m-1}\left(\hat\Lambda^2(\theta, \theta') - \hat\Lambda(\theta, \theta')^2\right)}$
	\vspace{.5em}
	\State $\hat\sigma \gets \dfrac{s}{\sqrt{m}}  \sqrt{\dfrac{N - m}{N - 1}}$
	\State $\rho \gets 1 - \phi_{m-1}\left(\left|\dfrac{\hat\Lambda(\theta, \theta') - \psi(u, \theta, \theta')}{\hat\sigma}\right|\right)$
      \If {$\rho > \epsilon$ \textbf{or} $m = N$}
      \vspace{.2em}
      \State \Return $\hat\Lambda(\theta,\theta')$
      \vspace{.2em}
      \EndIf
    \EndWhile
  \EndFunction
\end{algorithmic}
\end{algorithm}

\begin{notes}
\subfile{files/korattikara-austerity-2014}
\end{notes}

\subsection{Using concentration inequalities}
\label{sec:concentration}

\newcommand{\Ctt}{C_{\theta, \theta'}}

\citet{bardenet:2014-subsampling} propose an adaptive subsampling method that
is mechanically similar to using a $t$-test but instead uses concentration
inequalities.
In addition to a bound on the error (of the approximate acceptance probability)
that is local to each iteration, concentration bounds yield a bound on the
total variation distance between the approximate and true stationary distributions.

As in Section~\ref{sec:t-test}, we evaluate an approximate MH threshold
based on a data subset of size~$m$, given in Equation~\eqref{eq:approx-threshold}.
We bound the probability that the approximate binary outcome is incorrect via
\emph{concentration inequalities} that characterize the quality
of~$\hat\Lambda_m(\theta, \theta')$ as an estimate for~$\Lambda(\theta, \theta')$.
Such a concentration inequality is a probabilistic statement that,
for~$\delta_m \in (0, 1)$ and some constant~$c_m$,
\be
\P\left(\left|\hat\Lambda_m(\theta, \theta') - \Lambda(\theta, \theta')\right| \le c_m \right)
\ge 1 - \delta_m.
\ee
For example, in Hoeffding's inequality without replacement~\citep{serfling:1974}
\be
c_m = \Ctt \sqrt{\frac{2}{m}\left(1 - \frac{m-1}{N}\right)\log\left(\frac{2}{\delta_m}\right)}
\label{eq:serfling}
\ee
where
\be
\Ctt = \max_{1 \le n \le N} \left|\log \pi(x_n \given \theta') - \log \pi(x_n \given \theta)\right|
= \max_{1 \le n \le N} |\ell_n|,
\label{eq:Ctt}
\ee
using~$\ell_n$ as in Equation~\eqref{eq:log-likelihood-ratio}.
Alternatively, if the empirical standard deviation~$s_m$ of the~$m$
subsampled~$\ell_n^*$ terms is small, then the empirical Bernstein bound,
\be
c_m = s_m \sqrt{\frac{2 \log(3/\delta_m)}{m}} + \frac{6\Ctt \log(3/\delta_m)}{m},
\label{eq:bernstein}
\ee
is tighter~\citep{audibert:2009}, where~$s_m$ is given in Equation~\eqref{eq:empirical-std}.
While~$\Ctt$ can be obtained via all the~$\ell_n$, this is precisely the computation we want to avoid.
Therefore, the user must provide an estimate of~$\Ctt$.

\citet{bardenet:2014-subsampling} use a concentration bound to
construct an adaptive stopping rule based on a strategy called
empirical Bernstein stopping~\citep{mnih:2008-bernstein}.
Let~$c_m$ be a concentration bound as in Equation~\eqref{eq:serfling} or~\eqref{eq:bernstein}
and let~$\delta_m$ be the associated error.
This concentration bound states that
${|\hat\Lambda_m(\theta, \theta') - \Lambda(\theta, \theta')| \le c_m}$
with probability~${1 - \delta_m}$.
If~${|\hat\Lambda_m(\theta, \theta') - \psi(u, \theta, \theta')| > c_m}$,
then the approximate MH test agrees with the exact MH test with probability~${1 - \delta_m}$.
We reproduce a helpful illustration of this scenario
from~\citet{bardenet:2014-subsampling} in Figure~\ref{fig:bardenet}.
If instead~${|\hat\Lambda_m(\theta, \theta') - \psi(u, \theta, \theta')| \le c_m}$,
then we want to increase~$m$ until this is no longer the case.
Let~$M$ be the \emph{stopping time}, \ie the number of data points evaluated using this criterion,
\be
M = \min\left(N, \inf_{m \ge 1}
    \left|\hat\Lambda_m(\theta, \theta') - \psi(u, \theta, \theta')\right| > c_m \right).
\ee

\begin{figure}[t!]
\centering
\begin{tikzpicture}[dot/.style={circle,inner sep=2pt,fill},
                    empty/.style={circle,inner sep=0pt}]
\draw[line width=1pt] (0, 0) -- (10, 0);
\draw (0, 0) node[empty](){};
\draw (2, 0) node[dot,label=above:$\psi$](){};
\draw (6, 0) node[dot,label=above:$\hat\Lambda_m$](){};
\draw (8, 0) node[dot,label=above:$\Lambda$](){};
\draw (10, 0) node[empty](){};
\draw[<->,line width=1pt] (3.5, -0.5) node[empty](){} -- (8.5, -0.5) node[empty](){};
\draw (6, -0.5) node[empty,label=below:$2c_m$](){};
\draw[dashed,line width=1pt] (3.5, 0.2) node[empty](){} -- (3.5, -0.7) node[empty](){};
\draw[dashed,line width=1pt] (8.5, 0.2) node[empty](){} -- (8.5, -0.7) node[empty](){};
\end{tikzpicture}
\caption{Reproduction of Figure 2 from~\citet{bardenet:2014-subsampling}.
If~${|\hat\Lambda_m(\theta, \theta') - \psi(u, \theta, \theta')| > c_m}$,
then the adaptive stopping rule using a concentration bound is satisfied and
we use the approximate MH test based on~$\hat\Lambda_m(\theta, \theta')$.}
\label{fig:bardenet}
\end{figure}
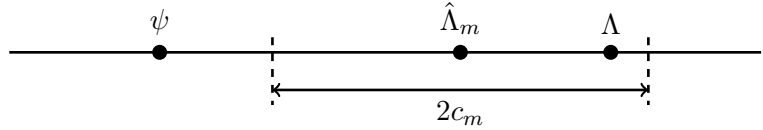

We can set~$\delta_m$ according to a user-defined parameter~${\epsilon \in (0, 1)}$
so that~$\epsilon$ gives an upper bound on the error of the
approximate acceptance probability. Let~$p > 1$ and set
\be
\delta_m = \frac{p - 1}{p m^p}\epsilon,
\quad \text{thus} \quad
\sum_{m \ge 1} \delta_m \le \epsilon.
\label{eq:delta-m}
\ee
A union bound argument gives
\be
\P\left(\bigcap_{m \ge 1}
\left\{\left|\hat\Lambda_m(\theta, \theta') - \Lambda(\theta, \theta')\right|
        \le c_m \right\} \right) \ge 1 - \epsilon,
\ee
under sampling without replacement.
Hence, with probability~${1-\epsilon}$, the approximate MH test
based on~$\hat\Lambda_M(\theta, \theta')$ agrees with the exact MH test.
In other words, the stopping rule for computing~$\hat\Lambda_m(\theta, \theta')$
in Algorithm~\ref{alg:mh-adaptive} is satisfied once we observe
${|\hat\Lambda_m(\theta, \theta') - \psi(u, \theta, \theta')| > c_m}$.
We illustrate this approach in Algorithm~\ref{alg:hoeffding}, using Hoeffding's inequality without replacement.

In their actual implementation, \citet{bardenet:2014-subsampling} modify~$\delta_m$
to reflect the number of batches processed instead of the subsample size~$m$.
For example, suppose we use the concentration bound in Equation~\eqref{eq:serfling},
\ie Hoeffding's inequality without replacement.
Then after processing a subsample of size~$m$ in~$k$ batches,
the adaptive stopping rule checks whether
${|\hat\Lambda_m(\theta, \theta') - \psi(u, \theta, \theta')| > c_m}$, where
\be
c_m = \Ctt \sqrt{\frac{2}{m}\left(1 - \frac{m-1}{N}\right)\log\left(\frac{2}{\delta_k}\right)}
\ee
and
\be
\delta_k = \frac{p - 1}{p k^p}\epsilon.
\ee
Also, as mentioned in Section~\ref{sec:adaptive-stopping-rule}, \citet{bardenet:2014-subsampling}
geometrically increase the subsample size by a factor~$\gamma$.
In their experiments, they use the empirical Bernstein-Serfling bound~\citep{bardenet:2013-concentration}.
For the hyperparameters, they set~${p = 2}$, ${\gamma=2}$, and ${\epsilon = 0.01}$,
and remark that they empirically found their algorithm to be robust to the choice of~$\epsilon$.

\begin{algorithm}[t!]
\caption{Estimate of the average log likelihood ratio.
The adaptive stopping rule uses Hoeffding's inequality without replacement.}
\label{alg:hoeffding}
\begin{algorithmic}
\State \textbf{Parameters:} batch size~$b$, user-defined error tolerance~$\epsilon$,
		estimate of~$\Ctt = \max_n |\ell_n|$, $p > 1$
  \Function{AvgLogLikeRatioEstimate}{$\theta,\theta',\psi(u,\theta,\theta')$}
    \State $m,~\hat\Lambda(\theta,\theta') \gets 0,~0$
    \While{True}
      \State $c \gets \min(b, N-M)$
      \State $\hat\Lambda(\theta, \theta') \gets
        \dfrac{1}{m+c} \left( m \hat\Lambda(\theta, \theta') + \displaystyle\sum_{n=m+1}^{m+c}
              \log\dfrac{\pi(x_n \given \theta')}{\pi(x_n \given \theta)} \right)$
      \State $m \gets m + c$
      \vspace{.3em}
      \State $\delta \gets \dfrac{p - 1}{p m^p}\epsilon$
      \State $c \gets \Ctt \sqrt{\dfrac{2}{m}\left(1 - \dfrac{m-1}{N}\right)\log\left(\dfrac{2}{\delta}\right)}$
      \If {$\left|\hat\Lambda(\theta, \theta') - \psi(u, \theta, \theta')\right| > c$ \textbf{or} $m = N$}
      \State \Return $\hat\Lambda(\theta,\theta')$
      \EndIf
    \EndWhile
  \EndFunction
\end{algorithmic}
\end{algorithm}

\subsection{Error bounds on the stationary distribution}
\label{sec:theory-general}

In this and the next subsection, we reproduce some theoretical results from 
\citet{korattikara-2014-austerity} and \citet{bardenet:2014-subsampling}.
After setting up some notation, we emphasize the most general aspects of these
results, which apply to pairs of transition kernels whose differences are bounded,
and thus are not specific to adaptive subsampling procedures.
The central theorem is an upper bound on the difference
between the stationary distributions of such pairs of kernels in the case of Metropolis--Hastings.
Its proof depends on the ability to bound the difference in the acceptance probabilities,
at each iteration, of the two MH transition kernels.

\paragraph{Preliminaries and notation.}
Let~$P$ and~$Q$ be probability measures (distributions) with Radon--Nikodym
derivatives (densities) $f_P$ and~$f_Q$, respectively, and absolutely continuous
with respect to measure~$\nu$.
The \emph{total variation distance} between~$P$ and~$Q$ is
\be
\tv{P}{Q} \equiv \frac{1}{2} \int_{\theta \in \Theta} d\nu(\theta) |f_P(\theta) - f_Q(\theta)|.
\ee
For any transition kernel~$T$, let~$T^k$ denote the kernel obtained via~$k$ iterations of~$T$.
Let~$T$ denote a transition kernel with stationary distribution~$\pi(\theta \given \x)$.
Let~$\tilde T$ denote an approximation to~$T$, with stationary distribution~$\tilde \pi$.
When~$T$ is a MH transition kernel, let~$q(\theta' \given \theta)$ denote its proposal,
and let~$\alpha(\theta, \theta')$ denote its acceptance probability,
given current and proposed states~$\theta$ and~$\theta'$, respectively, \ie
\be
\alpha(\theta, \theta') = \min\left(1, \frac{\pi(\theta') q(\theta \given \theta')}{\pi(\theta) q(\theta' \given \theta)}\right).
\ee
In this case, $\tilde T$ is an approximate MH transition kernel with the same proposal~$q(\theta' \given \theta)$,
and let~$\tilde\alpha(\theta, \theta')$ denote its acceptance probability.
Throughout this section, we specify when~$\tilde T$ is constructed from~$T$
via an adaptive stopping rule; some of the results are more general.
Let
\be
\calE(\theta, \theta') = \tilde\alpha(\theta, \theta') - \alpha(\theta, \theta')
\label{eq:accept-error}
\ee
be the acceptance probability error of the approximate MH test, with respect to the exact test.
Finally, let
\be
\Deltamax = \sup_{\theta, \theta'} |\calE(\theta, \theta')|
\label{eq:deltamax}
\ee
be the worst case absolute acceptance probability error. \\

\paragraph{Theoretical results.}
The theorem below provides an upper bound on the total variation 
distance between the stationary distributions of~$T$ and~$\tilde T$;
the bound is linear in~$\Deltamax$.

\begin{theorem}[Total variation bound under uniform geometric ergodicity~\citep{bardenet:2014-subsampling}]
Let $T$ be uniformly geometrically ergodic, \ie there exists an
integer~${h < \infty}$, probability measure~$\nu$ on~${(\Theta, \B(\Theta))}$,
and constant~${\lambda \in [0, 1)}$ such that for all ${\theta \in \Theta}$ and ${B \in \B(\Theta)}$,
\be
T^h(\theta, B) \ge (1 - \lambda) \nu(B),
\ee
and thus there exists a constant $A < \infty$ such that
for all~${\theta \in \Theta}$ and~${k > 0}$,
\be
\tv{T^k(\theta,\cdot)}{\pi} \le A \lambda^{\lfloor k/h \rfloor}.
\ee
It follows that there exists a constant $C < \infty$ such that
for all~${\theta \in \Theta}$ and~${k > 0}$,
\be
\tv{\tilde T^k (\theta, \cdot)}{\tilde \pi}
\le C \left(1 - (1-\epsilon)^h (1-\lambda) \right)^{\lfloor k/h \rfloor}.
\ee
Moreover,
\be
\tv{\pi}{\tilde \pi} \le \frac{A h \Deltamax}{1-\lambda}.
\label{eq:upper-bound}
\ee
\end{theorem}

The upper bound in Equation~\eqref{eq:upper-bound} depends on the worst case acceptance probability error.
For adaptive subsampling schemes, this depends on the choice of adaptive procedure.

We briefly outline a proof from~\citet{korattikara-2014-austerity} of a similar
theorem that exploits a stronger assumption on~$T$.
Specifically, assume~$T$ satisfies the contraction condition,
\be
\tv{PT}{\pi} \le \eta \tv{P}{\pi},
\ee
for all probability distributions~$P$ and some constant~$\eta \in [0, 1)$.
We can combine the contraction condition with a bound on the \emph{one-step error}
between~$T$ and~$\tilde T$, defined as~$\tv{P \tilde T}{P T}$,
to bound~$\tv{\pi}{\tilde \pi}$.
Note that this result does not require~$T$ to be a MH kernel.

For approximate MH with an adaptive stopping rule,~$\Deltamax$, the maximum
acceptance probability error, gives an upper bound on the one-step error.
\citet{korattikara-2014-austerity} show how to calculate an upper bound
on~$\Deltamax$ when using a $t$-test.
Using concentration inequalities leads to a simpler bound:
by construction, the user-defined error tolerance,~$\epsilon$,
directly gives an upper bound on~$\Deltamax$~\citep{bardenet:2014-subsampling}.

Finally, we note that an adaptive subsampling schemes using a concentration inequality
enables an upper bound on the stopping time~\citep{bardenet:2014-subsampling}.

\section{Sub-selecting data via a lower bound on the likelihood}
\label{sec:firefly}

\citet{maclaurin-2014-firefly} introduce \emph{Firefly Monte Carlo} (FlyMC),
an auxiliary variable MCMC sampling procedure
that operates on only subsets of data in each iteration.
At each iteration, the algorithm dynamically selects what data to evaluate
based on the random indicators included in the Markov chain state.
In addition, it generates samples from the exact target posterior rather
than an approximation.
However, FlyMC requires a lower bound on the likelihood with a particular
``collapsible'' structure (essentially an exponential family lower bound) and
is therefore not as generally applicable.
The algorithm's performance depends on the tightness of the bound; it
can achieve impressive gains in performance when model structure allows.

FlyMC samples from an augmented posterior that eliminates potentially
many likelihood factors.
Define
\be
L_n(\theta) = p(x_n \given \theta)
\label{eq:likelihood}
\ee
and let~$B_n(\theta)$ be a strictly positive lower bound on~$L_n(\theta)$,
\ie ${0 < B_n(\theta) \le L_n(\theta)}$.
For each datum, we introduce a binary auxiliary variable~${z_n \in \{0, 1\}}$
conditionally distributed according to a Bernoulli distribution,
\be
p(z_n \given x_n, \theta) = \left[\frac{L_n(\theta) - B_n(\theta)}{L_n(\theta)}\right]^{z_n}
\left[\frac{B_n(\theta)}{L_n(\theta)}\right]^{1-z_n},
\label{eq:aux}
\ee
where the $z_n$ are independent for different $n$.
When the bound is tight, \ie ${B_n(\theta) = L_n(\theta)}$, then~$z_n = 0$ with probability 1.
More generally, a tighter bound results in a higher probability that~$z_n = 0$.
Augmenting the density with ${\z = \{z_n\}_{n=1}^N}$ gives:
\begin{align}
\tilde\pi(\theta, \z \given \x)
&\propto \pi(\theta \given \x) p(\z \given \x, \theta) \nn \\
&= \pi_0(\theta) \prod_{n=1}^N \pi(x_n \given \theta) p(z_n \given x_n, \theta).
\end{align}
Using Equations~\eqref{eq:likelihood} and~\eqref{eq:aux}, we can now write:
\begin{align}
\tilde\pi(\theta, \z \given \x)
&\propto \pi_0(\theta) \prod_{n=1}^N L_n(\theta)
	\left[\frac{L_n(\theta) - B_n(\theta)}{L_n(\theta)}\right]^{z_n}
	\left[\frac{B_n(\theta)}{L_n(\theta)}\right]^{1-z_n} \nn \\
&= \pi_0(\theta) \prod_{n=1}^N (L_n(\theta) - B_n(\theta))^{z_n} B_n(\theta)^{1-z_n} \nn \\
&= \pi_0(\theta) \prod_{n : z_n=1} (L_n(\theta) - B_n(\theta))
	\prod_{n : z_n=0} B_n(\theta).
\label{eq:augmented}
\end{align}
Thus for any fixed configuration of~$\z$ we can evaluate the joint density
using only the likelihood terms $L_n(\theta)$ where ${z_n = 1}$ and the bound
values $B_n(\theta)$ for each ${n=1,2,\ldots,N}$.

While Equation~\eqref{eq:augmented} still involves a product of~$N$ terms, if the
product of the bound terms ${\prod_{n : z_n = 0} B_n(\theta)}$ can be evaluated
without reading each corresponding data point then the joint density can be
evaluated reading only the data~$x_n$ for which ${z_n = 1}$.
In particular, if the form of $B_n(\theta)$ is an exponential family density,
then the product~${\prod_{n : z_n = 0} B_n(\theta)}$ can be evaluated using only
a finite-dimensional sufficient statistic for the data $\{x_n : z_n = 0\}$.
Thus by exploiting lower bounds in the exponential family, FlyMC can reduce the
amount of data required at each iteration of the algorithm while maintaining
the exact posterior as its stationary distribution.
\citet{maclaurin-2014-firefly} show an application of this methodology to
Bayesian logistic regression.

FlyMC presents three main challenges.
The first is constructing a collapsible lower bound, such as an exponential
family, that is sufficiently tight.
The second is designing an efficient implementation.
\citet{maclaurin-2014-firefly} discuss these issues and, in particular,
design a cache-like data structure for managing the relationship between
the~$N$ indicator values and the data.  Finally, it is likely that the inclusion of these auxiliary variables slows the mixing of the Markov chain, but \citet{maclaurin-2014-firefly} only provide empirical evidence that this effect is small relative to the computational savings from using data subsets.

\begin{notes}
\subfile{files/maclaurin-2014-firefly}
\end{notes}

\section{Stochastic gradients of the log joint density}
\label{sec:sgld}

In this section, we review recent efforts to develop MCMC algorithms inspired by
stochastic optimization techniques.
This is motivated by the existence of, first, MCMC algorithms that can be
thought of as the sampling analogues of optimization algorithms, and second,
scalable stochastic versions of these optimization algorithms.

Traditional gradient ascent or descent performs optimization by iteratively
computing and following a local gradient~\citep{optimization:1983-book}.
In Bayesian MAP inference, the objective function is typically a log joint density
and the update rule for gradient ascent is given by
\be
\theta_{t+1} = \theta_t + \frac{\epsilon_t}{2} \biggl(\nabla \log\pi(\theta_t,  \x)\biggr)
\label{eq:gradient-descent}
\ee
for~${t = 1, \dots, \infty}$.
As discussed in Section~\ref{sec:background:sgd}, \emph{stochastic} gradient
descent (SGD) is simple modification of gradient descent that exploits
situations where the objective function decomposes into a sum of many terms.
While the traditional gradient descent update depends on all the data,~\ie
\be
\theta_{t+1} = \theta_t + \frac{\epsilon_t}{2} \left(\nabla\log \pi_0(\theta_t) +
                        \sum_{n=1}^N \nabla \log\pi(x_n \given \theta_t)\right),
\ee
SGD forms an update based on only a data subset,
\be
\theta_{t+1} = \theta_t + \frac{\epsilon_t}{2} \left(\nabla\log \pi_0(\theta_t) +
                        \frac{N}{m} \sum_{n=1}^m \nabla \log\pi(x_n \given \theta_t)\right).
\label{eq:sgd}
\ee
The iterates converge to a local extreme point of the log joint density in the
sense that $\lim_{t \to \infty} \nabla \log \pi(\theta_t | \x) = 0$
if the step size sequence~$\{\epsilon_t\}_{t=1}^\infty$ satisfies
\be
\sum_{t=1}^\infty \epsilon_t = \infty \quad \text{and} \quad \sum_{t=1}^\infty \epsilon_t^2 < \infty.
\label{eq:step-size}
\ee
A common choice of step size sequence is~${\epsilon_t = \alpha(\beta + t)^{-\gamma}}$ for some~${ \beta > 0 }$ and~${ \gamma \in (0.5,1] }$.

\citet{welling-2011-langevin} propose \emph{stochastic gradient Langevin dynamics} (SGLD),
an approximate MCMC procedure that combines SGD with a simple kind of
\emph{Langevin dynamics} (\emph{Langevin Monte Carlo})~\citep{neal:1994-hmc}.
They extend the \emph{Metropolis-adjusted Langevin algorithm} (MALA)
that uses noisy gradient steps to generate proposals for a Metropolis--Hastings
chain~\citep{roberts:1996-langevin}. At iteration~$t$, the MH proposal is
\be
\theta' = \theta_t + \frac{\epsilon}{2}\biggl(\nabla \log \pi(\theta_t, \x)\biggr) + \eta_t,
\label{eq:langevin}
\ee
where the injected noise~$\eta_t \sim \N(0, \epsilon)$ is Gaussian.
Notice that the scale of the noise is~$\sqrt{\epsilon}$,
\ie is constant and set by the gradient step size parameter.
The MALA proposal is thus a stochastic gradient step, constructed by adding noise to
a step in the direction of the gradient.

SGLD modifies the Langevin dynamics in Equation~\eqref{eq:langevin}
by using stochastic gradients based on data subsets, as in Equation~\eqref{eq:sgd},
and requiring that the step size parameter satisfy Equation~\eqref{eq:step-size}.
Thus, at iteration~$t$, the proposal is
\be
\theta' = \theta_t + \frac{\epsilon_t}{2}\left(\nabla\log \pi_0(\theta_t) +
            \frac{N}{m} \sum_{n=1}^m \nabla \log\pi(x_n \given \theta_t)\right) + \eta_t,
\label{eq:sgld}
\ee
where~$\eta_t \sim \N(0, \epsilon_t)$.
Notice that the injected noise decays with the gradient step size parameter, but at a slower rate.
Specifically, if~$\epsilon_t$ decays as~$t^{-\gamma}$, then~$\eta_t$ decays as~$t^{-\gamma/2}$.
As in MALA, the SGLD proposal is a stochastic gradient step, where the noise comes
from subsampling as well as the injected noise.

\begin{algorithm}[t!]
\caption{Stochastic gradient Langevin dynamics (SGLD).}
\label{alg:sgld}
\begin{algorithmic}
\State \textbf{Input:} Initial state $\theta_0$, number of iterations $T$, data $\x$,
grad log prior $\nabla\log \pi_0(\theta)$, grad log likelihood $\nabla \log\pi(x \given \theta)$,  batch size~$m$, step size tuning parameters (\eg $\alpha, \beta, \gamma$)
\State \textbf{Output:} Samples $\theta_1, \dots, \theta_T$
\State $J = N/m$
\For {$\tau$ in $0, \dots, T/J - 1$}
	\State $\x \gets \Permute(\x)$ \Comment{For sampling without replacement}
	\For {$k$ in $0, \dots, J-1$}
		\State $t = \tau J + k$
		\State $\epsilon_t \gets \alpha(\beta + t)^{-\gamma}$ \Comment{Example step size}
		\State $\eta_t \sim \N(0, \epsilon_t)$ \Comment{Draw noise to inject}
		\State $\theta' \gets \theta_t + \dfrac{\epsilon_t}{2}\left(\nabla\log \pi_0(\theta_t) +
            \dfrac{N}{m} \displaystyle\sum_{n=km+1}^{km+m} \nabla \log\pi(x_n \given \theta_t) \right) + \eta_t$
         \State $\theta_{t+1} \gets \theta'$  \Comment{Accept proposal with probability 1}
    \EndFor
\EndFor
\end{algorithmic}
\end{algorithm}

An actual Metropolis--Hastings algorithm would accept or reject the proposal
in Equation~\eqref{eq:sgld} by evaluating the full (log) joint density
at~$\theta'$ and~$\theta_t$, but this is precisely the computation we wish to avoid.
\citet{welling-2011-langevin} observe that as~${\epsilon_t \rightarrow 0}$,
$\theta' \rightarrow \theta_t$ in both Equations~\eqref{eq:langevin} and~\eqref{eq:sgld}.
In this limit, the probability of accepting the proposal converges to~$1$,
but the chain stops completely.
The authors suggest that~$\epsilon_t$ can be decayed to a value that is large enough for
efficient sampling, yet small enough for the acceptance probability to essentially be~$1$.
These assumptions lead to a scheme where~${\epsilon_t > \epsilon_\infty > 0}$, for all~$t$,
and all proposals are accepted, therefore the acceptance probability is never evaluated.
We show this scheme in Algorithm~\ref{alg:sgld}.
Without the stochastic MH acceptance step, however, asymptotic samples are no longer
guaranteed to represent the target distribution.

In more recent work, \citet{patterson:2013-sgrld} apply SGLD to
\emph{Riemann manifold Langevin dynamics}~\citep{girolami:2011-riemann} and
\citet{chen-2014-sghmc} combine the idea of SGD with \emph{Hamiltonian Monte Carlo}~(HMC),
an improved generalization of Langevin dynamics~\citep{neal:1994-hmc,neal-2010-hmc}.
Finally, we note that all the methods in this section
require gradient information that might not be readily computable.

\section{Summary}

In this chapter, we have surveyed three recent approaches to scaling MCMC that operate on subsets of data.
Below and in Table~\ref{table:subsets}, we summarize and compare
adaptive subsampling approaches~(\S\ref{sec:adaptive}),
FlyMC~(\S\ref{sec:firefly}), and SGLD~(\S\ref{sec:sgld}) along several axes.

\begin{table}
\centering
\resizebox{\textwidth}{!}{%
\begin{tabular}{llll}
\toprule
  & \textbf{Adaptive subsampling}
  & \textbf{FlyMC}
  & \textbf{SGLD} \\
\midrule
  \textbf{Approach}
  & Approximate MH test
  & Auxiliary variables
  & Optimization plus noise \\
\midrule
  \textbf{Requirements}
  & Error model, \eg $t$-test
  & Likelihood lower bound
  & Gradients, \ie $\nabla \log \pi(\theta, \x)$ \\
\midrule
  \textbf{Data access pattern}
  & Mini-batches
  & Random
  & Mini-batches \\
\midrule
  \textbf{Hyperparameters}
  & \pbox{5cm}{Batch size, \\ error tolerance per iteration}
  & None
  & \pbox{5cm}{Batch size, \\ error tolerance, \\ annealing schedule} \\
\midrule
  \textbf{Asymptotic bias}
  & Bounded TV
  & None
  & Bounded weak error \\
\bottomrule
\end{tabular}
}
\caption{Summary of recent MCMC methods for Bayesian inference that operate on data subsets.
Error refers to the total variation distance between the stationary distribution of
the Markov chain and the target posterior distribution.}
\label{table:subsets}
\end{table}

\paragraph{Approaches.}
Adaptive subsampling approaches replace the Metropolis--Hastings (MH) test,
a function of all the data, with an approximate test that depends on only a subset.
FlyMC is an auxiliary variable method that stochastically replaces likelihood
computations with a collapsible lower bound.
Stochastic gradient Langevin dynamics (SGLD) replaces gradients in a
Metropolis-adjusted Langevin algorithm (MALA) with stochastic gradients based
on data subsets and eliminates the Metropolis--Hastings test.

\paragraph{Generality, requirements, and assumptions.}
Each of the methods exploits assumptions or additional problem structure.
Adaptive subsampling methods require an error model that accurately represents the
probability that an approximate MH test will disagree with the exact MH test.
A normal model~\citep{korattikara-2014-austerity} or
concentration bounds~\citep{bardenet:2014-subsampling} represent natural choices;
under certain conditions, tighter concentration bounds may apply.
FlyMC requires a strictly positive collapsible lower bound on the likelihood,
essentially an exponential family lower bound, which may not in general be
available.
SGLD requires the log gradients of the prior and likelihood.

\paragraph{Data access patterns.}
While all the methods use subsets of data, their access patterns differ.
Adaptive subsampling and SGLD require randomization to avoid issues of bias due
to data order, but this randomization can be achieved by permuting the data
before each pass and hence these algorithms allow data access that is mostly
sequential.
In contrast, FlyMC operates on random subsets of data determined by the Markov
chain itself, leading to a random access pattern.
However, subsets from one iteration to the next tend to be correlated, and
motivate implementation details such as the proposed cache data structure.

\paragraph{Hyperparameters.}
FlyMC does not introduce additional hyperparameters that require tuning.
Both adaptive subsampling methods and SGLD introduce hyperparameters that can
significantly affect performance.
Both are mini-batch methods, and thus have the batch size as a tuning
parameter.
In adaptive subsampling methods, the stopping criterion is evaluated
potentially more than once before it is satisfied.
This motivates schemes that geometrically increase the amount of data processed
whenever the stopping criterion is not satisfied, which introduces additional
hyperparameters.
Adaptive subsampling methods additionally provide a single tuning parameter
that allows the user to control the error at each iteration.
Finally, since these adaptive methods define an approximate MH test, they
implicitly also require that the user specify a proposal distribution.
For SGLD, the user must specify an annealing schedule for the step size
parameter; in particular, it should converge to a small positive value so that
the injected noise term dominates, while not being too large compared to the
scale of the posterior distribution.

\paragraph{Error.}
FlyMC is exact in the sense that the target posterior distribution is a
marginal of its augmented state space.
The adaptive subsampling approaches and SGLD are approximate methods in that
neither has a stationary distribution equal to the target posterior.
The adaptive subsampling approaches bound the error of the MH test at each
iteration, and for MH transition kernels with uniform ergodicity this one-step
error bound leads to an upper bound on the total variation distance between the
approximate stationary distribution and the target posterior distribution.
The theoretical analysis of SGLD is less clear~\citep{sato2014approximation}.

\section{Discussion}

\paragraph{Data subsets.}
The methods surveyed in this chapter achieve computational gains by using data
subsets in place of an entire dataset of interest.
The adaptive subsampling algorithms~(\S\ref{sec:adaptive}) are more successful when
a small subsample leads to an accurate estimator for the exact MH test's accept/reject decision.
Intuitively, such an estimator is easier to construct when the log posterior
values at the proposed and current states are significantly different.
This tends to be true far away from the mode(s) of the posterior,
\eg in the tails of a distribution that decay exponentially fast,
compared to the area around a mode, which is locally more flat.
Thus, these algorithms tend to evaluate more data when the chain is
in the vicinity of a mode, and less data when the chain is far away
(which tends to be the case for an arbitrary initial condition).
SGLD~(\S\ref{sec:sgld}) exhibits somewhat related behavior.
Recall that SGLD behaves more like SGD when the update rule is dominated by the
gradient term, which tends to be true during the initial execution phase.
Similar to SGD, the chain progresses toward a mode at a rate that depends on
the accuracy of the stochastic gradients.
For a log posterior target, stochastic gradients tend to be more accurate
estimators of true gradients far away from the mode(s).
In contrast, the MAP-tuned version of FlyMC~(\S\ref{sec:firefly}) requires the
fewest data evaluations when the chain is close to the MAP, since by design,
the lower likelihood bounds are tightest there.
Meanwhile, the untuned version of FlyMC tends to exhibit the opposite behavior.

\paragraph{Adaptive proposal distributions.}
The Metropolis-Hastings algorithm requires the user to specify a proposal distribution.
Fixing proposal distribution can be problematic, because the behavior of MH is
sensitive to the proposal distribution and can furthermore change as the chain converges.
A common solution, employed \eg by~\citet{bardenet:2014-subsampling},
is to use an adaptive MH scheme~\citep{haario:2001-adaptive,andrieu:2006-adaptive}.
These algorithms tune the proposal distribution during execution, using information
from the samples as they are generated, in a way that provably converges asymptotically.
Often, it is desirable for the proposal distribution to be close to the target.
This motivates adaptive schemes that fit a distribution to the observed samples
and use this fitted model as the proposal distribution.
For example, a simple online procedure can update the mean~$\mu$ and
covariance~$\Sigma$ of a multidimensional Gaussian model as follows:
\bea
\mu_{t+1} &=& \mu_t + \gamma_{t+1} (\theta_{t+1} - \mu_t) \qquad t \ge 0 \nn \\
\Sigma_{t+1} &=& \Sigma_k + \gamma_{t+1} ((\theta_{t+1} - \mu_t)(\theta_{t+1} - \mu_t)^\top - \Sigma_t), \nn
\eea
where~$t$ indexes the MH iterations and~$\gamma_{t+1}$ controls the speed with
which the adaptation vanishes.
An appropriate choice is~${\gamma_{t} = t^{-\alpha}}$ for~${\alpha \in [1/2, 1)}$.
The tutorial by~\citet{andrieu:2008-adaptive-tutorial} provides a review of
this and other, more sophisticated, adaptive MH algorithms.

\paragraph{Combining methods.}
The subsampling-based methods in this chapter are conceptually modular, and some may be combined.
For example, it might be of interest to consider a `tunable' version of FlyMC that
achieves even greater computational efficiency at the cost of its original exactness.
For example, we might use an adaptive subsampling scheme~(\S\ref{sec:adaptive})
to evaluate only a subset of terms in Equation~\eqref{eq:augmented};
this subset would need to represent terms corresponding to both possible values of~$z_n$.
As another example, \citet{korattikara-2014-austerity} suggest using
adaptive subsampling as a way to `fix up' SGLD.
Recall that the original SGLD algorithm completely eliminates the MH test and
blindly accepts all proposals, in order to avoid evaluating the full posterior.
A reasonable compromise is to instead evaluate a fraction of the data within the
adaptive subsampling framework, since this bounds the per-iteration error.

\paragraph{Unbiased likelihood estimators.}
The estimator in Equation~\eqref{eq:log-likelihood-subset} based on a
data subset is an unbiased estimator for the log likelihood; to be explicit,
\be
\exp\left\{\frac{N}{m} \sum_{n=1}^m \log \pi(x_n^* \given \theta)\right\}
\label{eq:likelihood-biased}
\ee
is not an unbiased estimate of the likelihood.
While it is possible to transform Equation~\eqref{eq:likelihood-biased}
into an unbiased likelihood estimate,
\eg using a Poisson estimator~\citep{wagner:1987,papaspiliopoulos:2009,fearnhead:2010},
it is not necessarily non-negative, which is a requirement to incorporate the estimator into a Metropolis-Hastings algorithm.
In general, we cannot derive estimators that are both unbiased and nonnegative~\citep{jacob:2015-unbiased,lyne:2015-roulette}.
\emph{Pseudo-marginal MCMC} algorithms,\footnote{Pseudo-marginal MCMC is also known as \emph{exact-approximate sampling}.}
first introduced by~\citet{lin:2000-pseudo}, rely on non-negative unbiased likelihood estimators
to construct unbiased MCMC procedures~\citep{andrieu:2009-pseudo}.
In this context, methods for constructing unbiased non-negative likelihood estimators include
importance sampling~\citep{beaumont:2003-is} and particle filters~\citep{andrieu:2010-particle,doucet:2015-unbiased}.

%% file: chapters/mcmc-parallel/main.tex
\chapter{Parallel and distributed MCMC}
\label{sec:mcmc-parallel}

MCMC procedures that take advantage of parallel computing resources form
another broad approach to scaling Bayesian inference.
Because the computational requirements of inference often scale with the amount
of data involved, and because large datasets may not even fit on a single
machine, these approaches often focus on data parallelism.
In this chapter we consider several approaches to scaling MCMC by exploiting
parallel computation, either by adapting classical MCMC algorithms or by
defining new simulation dynamics that are inherently parallel.

One way to use parallel computing resources is to run multiple sequential MCMC
algorithms at once.
However, running identical chains in parallel does not reduce the transient
bias in MCMC estimates of posterior expectations, though it would reduce their
variance.
Instead of using parallel computation only to collect more MCMC samples and
thus reduce only estimator variance without improving transient bias, it is
often preferable to use computational resources to speed up the simulation of
the chain itself.
Section~\ref{sec:parallelizing-standard} surveys several methods that use
parallel computation to speed up the execution of MCMC procedures, including
both basic methods and more recent ideas.

Alternatively, instead of adapting serial MCMC procedures to exploit parallel
resources, another approach is to design new approximate algorithms that are
inherently parallel.
Section~\ref{sec:parallelizing-aggregation} summarizes some recent ideas for
simulations that can be executed in a data-parallel manner and have their
results aggregated or corrected to represent posterior samples.

\section{Parallelizing standard MCMC algorithms}
\label{sec:parallelizing-standard}

An advantage to parallelizing standard MCMC algorithms is that they
retain their theoretical guarantees and analyses.
Indeed, a common goal is to produce identical samples under serial and parallel
execution, so that parallel resources enable speedups without introducing new
approximations.
This section first summarizes some basic opportunities for parallelism in MCMC
and then surveys the speculative execution framework for MH.

\subsection{Conditional independence and graph structure}
\label{sec:conditional-independence}

The MH algorithm has a straightforward opportunity for parallelism.
In particular, if the target posterior can be written as
\begin{equation}
  \pi(\theta \given \x) \propto \pi_0(\theta) \pi(\x \given \theta) 
                  = \pi_0(\theta) \prod_{n=1}^{N} \pi(x_n \given \theta),
\end{equation}
then when the number of likelihood terms~$N$ is large it may be beneficial to
parallize the evaluation of the product of likelihoods.
The communication between processors is limited to transmitting the value of
the parameter and the scalar values of likelihood products.
This basic parallelization, which naturally fits in a bulk synchronous parallel
(BSP) computational model, exploits conditional independence in the
probabilistic model, namely that the data are indepdendent given the parameter.

Gibbs sampling algorithms can exploit more fine-grained conditional
independence structure, and are thus a natural fit for graphical models which
express such structure.
Given a graphical model and a corresponding graph coloring with~$K$ colors that
partitions the set of random variables into~$K$ groups, the random variables in
each color group can be resampled in parallel while conditioning on the values in the
other~${K-1}$ groups~\citep{gonzalez:2011-parallel}.
Thus graphical models provide a natural perspective on opportunities for
parallelism.
See Figure~\ref{fig:graph_coloring} for some examples.
\todo{say exponential family sufficient statistics? reduce part of map-reduce}

\begin{figure}
  \centering
  \begin{subfigure}{0.45\textwidth}
    \centering
    \includegraphics[width=0.6\textwidth]{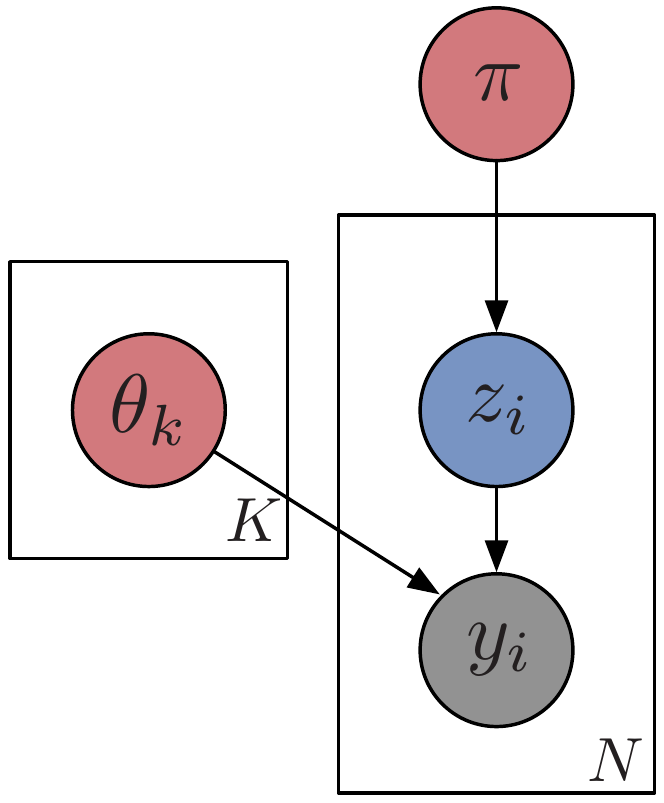}
    \caption{A mixture model}
    \label{fig:graph_coloring:mixture}
  \end{subfigure}
  \begin{subfigure}{0.45\textwidth}
    \centering
    \includegraphics[width=0.8\textwidth]{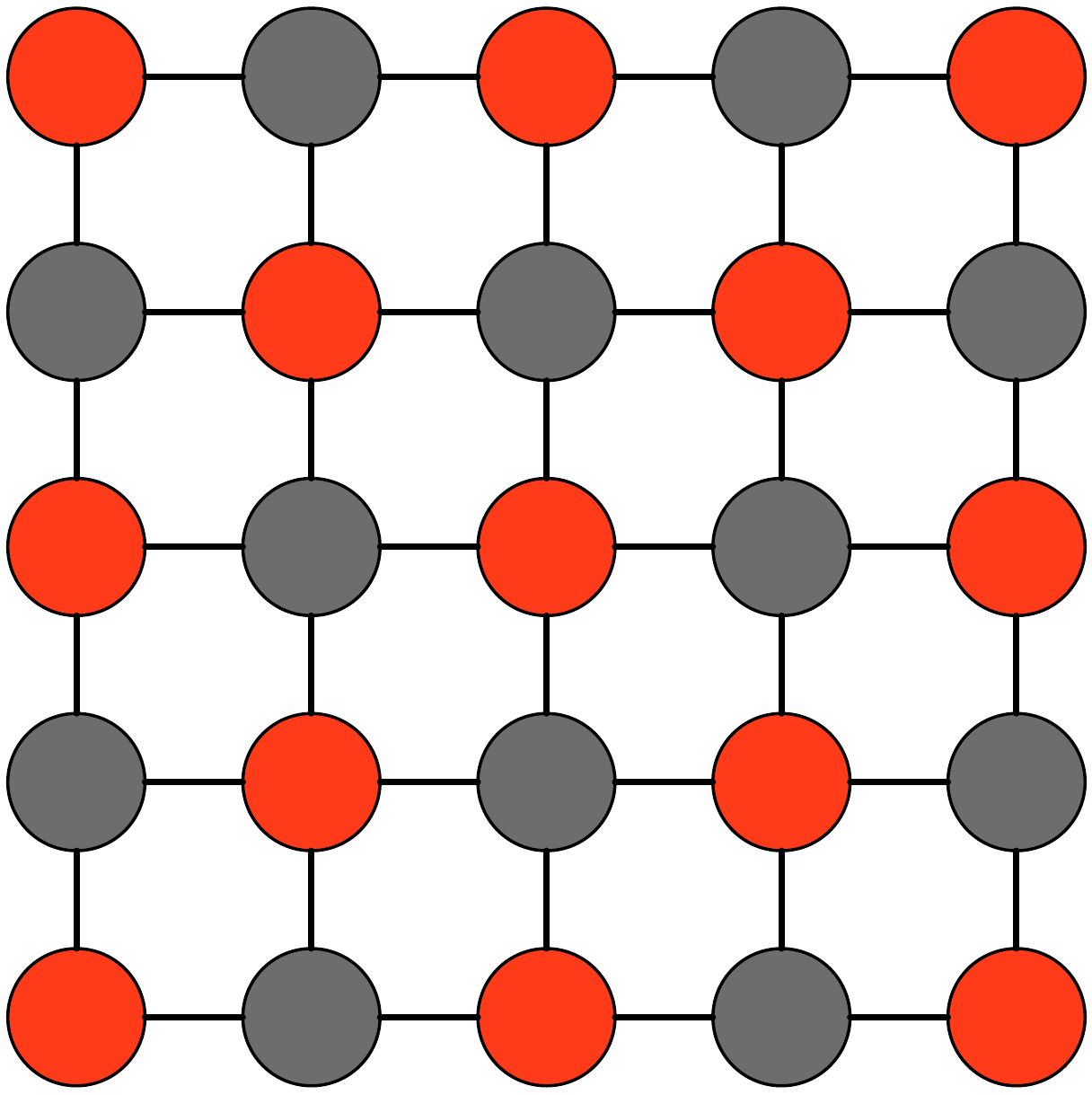}
    \caption{An undirected grid}
    \label{fig:graph_coloring:grid}
  \end{subfigure}
  \caption{Graphical models and graph colorings can expose opportunities for
    parallelism in Gibbs samplers.~\subref{fig:graph_coloring:mixture}  In this
    directed graphical model for a discrete mixture, each label (red)
    can be sampled in parallel conditioned on the parameters (blue) and data
    (gray) and similarly each parameter can be resampled in parallel
    conditioned on the labels.~\subref{fig:graph_coloring:grid} This
    undirected grid has a classical ``red-black'' coloring, emphasizing that
    the variables corresponding to red nodes can be resampled in parallel given
    the values of black nodes and vice-versa.}
  \label{fig:graph_coloring}
\end{figure}

These opportunities for parallelism, while powerful in some cases, are limited
by the fact that they require frequent global synchronization and
communication.
Indeed, at each iteration it is often the case that every element of the
dataset is read by some processor and many processors must mutually
communicate.
The methods we survey in the remainder of this chapter aim to mitigate these
limitations by adjusting the allocation of parallel resources or by reducing
communication.

\subsection{Speculative execution and prefetching}
\label{sec:prefetching}

Another class of parallel MCMC algorithms uses speculative parallel execution
to accelerate individual chains.
This idea is called \emph{prefetching} in some of the literature and appears to
have received only limited attention.

As shown in Algorithm~\ref{mh}, the body of a MH implementation is a
loop containing a single conditional statement and two associated branches.
We can thus view the possible execution paths as a binary tree,
illustrated in Figure~\ref{tree}.
The vanilla version of parallel prefetching speculatively evaluates all paths
in this binary tree on parallel processors~\citep{brockwell-2006-prefetching}.
The sampled path will be exactly one of these, so with~$J$ processors this approach
achieves a speedup of~$\log_2 J$ with respect to single core execution,
ignoring communication and bookkeeping overheads.

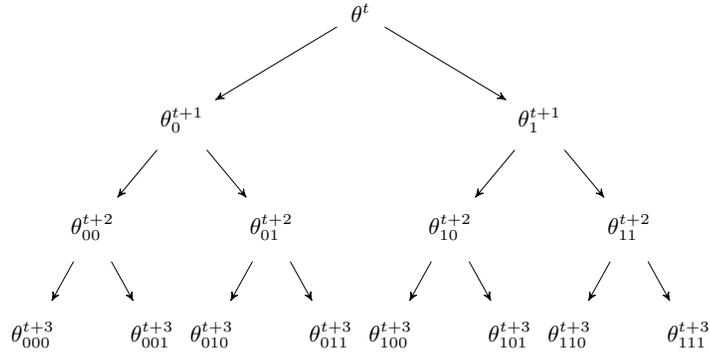
\begin{figure}[t!]
\centering%
\resizebox{\columnwidth}{!}{%
 \def\radius {5mm}
 \tikzstyle{state}=[circle, thick, minimum size=\radius, font=\footnotesize]
 \begin{tikzpicture}[->,>=stealth',level/.style={sibling distance = 5cm/#1, level distance = 1.5cm}]
   \draw [white,-] (-6cm,-0.75cm) -- (6cm,-0.75cm);
   \node [state] {$\theta^t$}
   child{ node [state] {$\theta^{t+1}_{0}$}
	 child{ node [state] {$\theta^{t+2}_{{00}}$}
	   child{ node [state] {$\theta^{t+3}_{{000}}$}}
	   child{ node [state] {$\theta^{t+3}_{{001}}$}}
	 }
	 child{ node [state] {$\theta^{t+2}_{{01}}$}
	   child{ node [state] {$\theta^{t+3}_{{010}}$}}
	   child{ node [state] {$\theta^{t+3}_{{011}}$}}
	 }
   }
   child{ node [state] {$\theta^{t+1}_1$}
	 child{ node [state] {$\theta^{t+2}_{{10}}$}
	   child{ node [state] {$\theta^{t+3}_{{100}}$}}
	   child{ node [state] {$\theta^{t+3}_{{101}}$}}
	 }
	 child{ node [state] {$\theta^{t+2}_{{11}}$}
	   child{ node [state] {$\theta^{t+3}_{{110}}$}}
	   child{ node [state] {$\theta^{t+3}_{{111}}$}}
	 }
   }
   ;
 \end{tikzpicture}
}
\caption{Metropolis--Hastings conceptualized as a binary tree.
Nodes at depth~$d$ correspond to iteration~$t + d$, where the root is at depth~0,
and branching to the right/left indicates that the proposal is accepted/rejected.
Each subscript is a sequence, of length~$d$, of~0's and~1's,
corresponding to the history of rejected and accepted proposals
with respect to the root.
}
\label{tree}
\end{figure}

Na\"{i}ve prefetching can be improved by observing that the two branches in
Algorithm~\ref{mh} are not taken with equal probability.
For typical algorithm tunings, the reject branch tends to be more probable; a
classic result for the optimal MH acceptance rate in the Gaussian case is
0.234~\citep{roberts-1997-accept}, so prefetching scheduling policies can be
built around the expectation of rejection.
\citet{angelino:2014-prefetching} provides a thorough review of these
strategies.

\emph{Parallel predictive prefetching} makes more efficient use of parallel resources
by dynamically predicting the outcome of each MH test~\citep{angelino:2014-prefetching}.
In the case of Bayesian inference, these predictions can be constructed in the
same manner as the approximate MH algorithms based on subsets of data,
as discussed in Section~\ref{sec:adaptive-stopping-rule}.
Furthermore, these predictions can be made in the context of an error model,
\eg with the concentration inequalities used by~\citet{bardenet:2014-subsampling}.
This yields a straightforward and rational mechanism for allocating parallel
cores to computations most likely to fall along the true execution path.

Algorithms~\ref{prefetching:master} and~\ref{prefetching:worker} sketch
pseudocode for an implementation of parallel predictive prefetching
that follows a master-worker pattern.
See~\citet{angelino-2014-thesis} for a formal description of the algorithm
and implementation details.

\begin{algorithm}[t]
  \caption{Parallel predictive prefetching master process}
  \label{prefetching:master}
  \begin{algorithmic}
      \Repeat
          \State Receive message from worker $j$
          \If {worker $j$ wants work}
              \State {Find highest utility node~$\rho$ in tree with work left to do}
              \State {Send worker $j$ the computational state of~$\rho$}
          \ElsIf {message contains state~$\theta_\rho$ at proposal node~$\rho$}
              \State {Record state~$\theta_\rho$ at~$\rho$}
          \ElsIf {message contains update at~$\rho$}
              \State {Update estimate of~$\pi(\theta_\rho \given \x)$ at~$\rho$}
            \For {node~$\alpha$ in $\{\rho$ and its descendants$\}$}
                \State Update utility of~$\alpha$
                  \If {utility of~$\alpha$ below threshold and worker~$k$ at~$\alpha$}
                      \State {Send worker~$k$ message to stop current computation}
                  \EndIf
            \EndFor
              \If {posterior computation at~$\rho$ and its parent complete}
                  \State {Know definitively whether to accept or reject~$\theta_\rho$}
                  \State {Delete subtree corresponding to branch not taken}
                  \If {node~$\rho$ is the root's child}
                      \Repeat
                          \State Trim old root so that new root points to child
                          \State Output state at root, the next state in the chain
                      \Until {posterior computation at root's child incomplete}
                  \EndIf
              \EndIf
          \EndIf
      \Until {master has output $T$ Metropolis--Hastings chain states}
      \State Terminate all worker processes
  \end{algorithmic}
\end{algorithm}

\begin{algorithm}[t]
  \caption{Parallel predictive prefetching worker process}
  \label{prefetching:worker}
  \begin{algorithmic}
    \Repeat
        \State Send master request for work
        \State Receive work assignment at node~$\rho$ from master
        \If {the corresponding state~$\theta_\rho$ has not yet been generated}
            \State Generate proposal~$\theta_\rho$
        \EndIf
        \Repeat
            \State Advance the computation of~$\pi(\theta_\rho \given \x)$
            \State Send update at~$\rho$ to master
            \If {receive message to stop current computation}
                \State \textbf{break}
            \EndIf
        \Until {computation of~$\pi(\theta_\rho \given \x)$ is complete}
    \Until {terminated by master}
  \end{algorithmic}
\end{algorithm}

\section{Defining new data-parallel dynamics}
\label{sec:parallelizing-aggregation}

In this section we survey two ideas for performing inference using new
data-parallel dynamics.
These algorithms define new dynamics in the sense that their iterates do not
form ergodic Markov chains which admit the posterior distribution as an
invariant distribution, and thus they do not qualify as classical MCMC schemes.
Instead, while some of the updates in these algorithms resemble standard MCMC
updates, the overall dynamics are designed to exploit parallel and
distributed computation.
A unifying theme of these new methods is to perform local computation on data
while controlling the amount of global synchronization or communication.

One such family of ideas involves the definition of \emph{subposteriors},
defined using only subsets of the full dataset.
Inference in the subposteriors can be performed in parallel, and the results
are then globally aggregated into an approximate representation of the full
posterior.
Because the synchronization and communication costs---as well as the
approximation quality---are determined by the aggregation step, several such
aggregation procedures have been proposed.
In Section~\ref{sec:pmcmc:subposteriors} we summarize some of these proposals.

Another class of data-parallel dynamics does not define independent
subposteriors but instead, motivated by Gibbs sampling, focuses on simulating
from local conditional distributions with out-of-date information.
In standard Gibbs sampling, updates can be parallelized in models with
conditional independence structure (Section~\ref{sec:parallelizing-standard}),
but without such structure the Gibbs updates may depend on the full
dataset and all latent variables, and thus must be performed sequentially.
These sequential updates can be especially expensive with large or distributed
datasets.
A natural approximation to consider is to run the same local Gibbs updates in
parallel with out-of-date global information and only infrequent communication.
While such a procedure loses the theoretical guarantees provided by standard Gibbs
sampling analysis, some empirical and theoretical results are promising.
We refer to this broad class of methods as \emph{Hogwild} Gibbs algorithms, and we
survey some particular algorithms and analyses in Section~\ref{sec:pmcmc:hogwild}.

\subsection{Aggregating from subposteriors}
\label{sec:pmcmc:subposteriors}

Suppose we want to divide the evaluation of the posterior across~$J$ parallel cores.
We can divide the data into~$J$ partition elements,~${ \x^{(1)}, \dots, \x^{(J)}} $, also called \emph{shards},
and factor the posterior into~$J$ corresponding \emph{subposteriors}, as
\be
\pi(\theta \given \x) = \prod_{j=1}^J \pi^{(j)}(\theta \given \x^{(j)}),
\ee
where
\be
\pi^{(j)}(\theta \given \x^{(j)})
= \pi_0(\theta)^{1/J} \prod_{x \in \x^{(j)}} \pi(x \given \theta),
\quad j = 1, \dots, J.
\label{eq:subposterior}
\ee
The contribution from the original prior is down-weighted so that
the posterior is equal to the product of the~$J$ subposteriors,
\ie ${\pi(\theta \given \x) = \prod_{j=1}^J \pi^{(j)}(\theta \given \x^{(j)})}$.
Note that a subposterior is not the same as the posterior formed from
the corresponding partition, \ie
\be
\pi^{(j)}(\theta \given \x^{(j)})
\neq \pi(\theta \given \x^{(j)})
= \pi_0(\theta) \prod_{x \in \x^{(j)}} \pi(x \given \theta).
\ee

\subsubsection{Embarrassingly parallel consensus of subposteriors}
\label{sec:ep-consensus}

Once a large dataset has been partitioned across multiple machines, a natural alternative
is to try running MCMC inference on each partition element separately and in parallel.
This yields samples from each subposterior in Equation~\ref{eq:subposterior},
but there is no obvious choice for how to combine them in a coherent fashion to
form approximate samples of the full posterior.
In this section, we survey various proposals for forming such a \emph{consensus}
solution from the subposterior samples.
Algorithm~\ref{alg:consensus} outlines the structure of consensus strategies for
embarrassingly parallel posterior sampling.
This terminology, used by~\citet{huang-2005-sampling} and~\citet{scott-2013-consensus},
invokes related notions of consensus, notably those that have existed for decades in the
optimization literature on data-parallel algorithms in decentralized or distributed settings.
We discuss this topic briefly in Section~\ref{sec:weierstrass}.

Below, we present two recent consensus strategies for combining
subposterior samples, through weighted averaging and density estimation, respectively.
The earlier report by~\citet{huang-2005-sampling} proposes four consensus
strategies, based either on normal approximations or importance resampling;
the authors focus on Gibbs sampling for hierarchical models and do not
evaluate any actual parallel implementations.
Another consensus strategy is the recently proposed \emph{variational consensus
Monte Carlo} (VCMC) algorithm, which casts the consensus problem within a
variational Bayes framework~\citep{rabinovich:2015-vcmc}.

\begin{algorithm}[t!]
\caption{Embarrassingly parallel consensus of subposteriors}
\label{alg:consensus}
\begin{algorithmic}
\State \textbf{Input:} Initial state~$\theta_0$, number of samples~$T$,
data partitions ${\x^{(1)}, \dots, \x^{(J)}}$,
subposteriors ${\pi^{(1)}(\theta \given \x^{(1)}), \dots, \pi^{(J)}(\theta \given \x^{(J)})}$
\State \textbf{Output:} Approximate samples $\hat\theta_1, \dots, \hat\theta_T$
\For {$j = 1, 2, \dots, J$ in parallel}
	\State Initialize $\theta_{j,0}$
	\For {$t = 1, 2, \dots, T$}
		\State Simulate MCMC sample $\theta_{j,t}$ from subposterior $\pi^{(j)}(\theta \given \x^{(j)})$
	\EndFor
	\State Collect $\theta_{j,1}, \dots, \theta_{j,T}$
\EndFor

\State $\hat\theta_1, \dots, \hat\theta_T \gets $ \Call{ConsensusSamples}{$\{\theta_{j,1}, \dots, \theta_{j,T}\}_{j=1}^J$}
\end{algorithmic} 
\end{algorithm}

Throughout this section, Gaussian densities provide a useful reference point
and motivate some of the consensus strategies.
Consider the jointly Gaussian model
\begin{align}
    \theta &\sim \mathcal{N}(0,\Sigma_0)
    \\
    \x^{(j)} \given \theta &\sim \mathcal{N}(\theta,\Sigma_j).
    \label{eq:gaussian-likelihood}
\end{align}
The joint density is:
\begin{align}
p(\theta, \x) 
 &= p(\theta) \prod_{j=1}^J p(\x^{(j)} \given \theta) \nn \\
&\propto \exp\left\{-\frac{1}{2} \theta^\top \Sigma_0^{-1} \theta\right\}
	\prod_{j=1}^J \exp\left\{-\frac{1}{2} \left(\x^{(j)} - \theta\right)^\top \Sigma_J^{-1} \left(\x^{(j)} - \theta\right) \right\} \nn \\
	&\propto \exp\left\{-\frac{1}{2} \theta^\top \left(\Sigma_0^{-1} + \sum_{j=1}^J \Sigma_j^{-1}\right) \theta
	+ \left(\sum_{j=1}^J \Sigma_j^{-1} \x^{(j)} \right)^\top \theta \right\} \nn \, .
\end{align}
Thus the posterior is Gaussian:
\be
\theta \given \x \sim \mathcal{N}(\mu, \Sigma),
\ee
where
\begin{align}
    \Sigma &= \left( \Sigma_0^{-1} + \sum_{j=1}^J \Sigma_j^{-1} \right)^{-1}
    \label{eq:gaussian-sigma}
    \\
    \mu &= \Sigma \left( \sum_{j=1}^J \Sigma_j^{-1} \x^{(j)} \right).
    \label{eq:gaussian-mu}
\end{align}
To arrive at an expression for the subposteriors, we begin by factoring the
joint distribution into an appropriate product:
\be
p(\theta, \x) \propto \prod_{j=1}^J f_j(\theta),
\ee
where
\begin{align}
f_j(\theta)  &= p(\theta)^{1/J} p(\x^{(j)} \given \theta) 
  \nn \\
&= \exp\left\{-\frac{1}{2} \theta^\top (\Sigma_0^{-1} / J) \theta \right\}
	\exp\left\{-\frac{1}{2} \left(\x^{(j)} - \theta\right)^\top \Sigma_j^{-1} \left(\x^{(j)} - \theta\right) \right\} \nn \\
&\propto \exp\left\{-\frac{1}{2} \theta^\top \left(\Sigma_0^{-1} /J + \Sigma_j^{-1}\right) \theta 
	+ \left(\Sigma_j^{-1} \x^{(j)} \right)^\top \theta \right\}. \nn
\end{align}
Thus the subposteriors are also Gaussian:
\begin{align}
    \theta_j \sim  \mathcal{N}\left(\tilde\mu_j, \tilde\Sigma_j \right) \propto f_j(\theta)
    \label{eq:gaussian-sub-post}
\end{align}
where
\begin{align}
\tilde\Sigma_j &= \left(\Sigma_0^{-1} / J + \Sigma_j^{-1} \right)^{-1}
\label{eq:tilde-sigma}
\\
\tilde\mu_j &= \left(\Sigma_0^{-1} / J + \Sigma_j^{-1} \right)^{-1} \left(\Sigma_j^{-1} \x^{(j)} \right).
\label{eq:tilde-mu}
\end{align}

\subsubsection{Weighted averaging of subposterior samples}
\label{sec:subposterior-avg}

One approach is to combine the subposterior samples via weighted averaging~\citep{scott-2013-consensus}.
For simplicity, we assume that we obtain~$T$ samples in $\reals^d$ from each 
subposterior, and let~$\{\theta_{j,t}\}_{t=1}^T$ denote the samples from the~$j$th subposterior.
The goal is to construct~$T$ \emph{consensus posterior} samples~$\{\hat\theta_t\}_{t=1}^T$,
that (approximately) represent the full posterior, from the~$JT$ subposterior samples,
where each~$\hat\theta_t$ combines subposterior samples~$\{\theta_{j,t}\}_{j=1}^J$.
We associate with each subposterior~$j$ a matrix~${W_j \in \reals^{d \times d}}$
and assume that each consensus posterior sample is a weighted\footnote{Our notation differs slightly from that of~\citet{scott-2013-consensus} in that our weights~$W_j$ are normalized.} average:
\be
\hat\theta_t = \sum_{j=1}^J W_j \theta_{j,t}\, .
\ee
The challenge now is to design an appropriate set of weights.

Following~\citet{scott-2013-consensus}, we consider the special case of
Gaussian subposteriors, as in Equation~\ref{eq:gaussian-sub-post}.
Our presentation is slightly different, as we also account for the effect of
having a prior.
We also drop the subscript $t$ from our notation for simplicity.
Let~$\{\theta_j\}_{j=1}^J$ be a set of draws from the~$J$ subposteriors. 
Each~$\theta_j$ is an independent Gaussian and thus~$\hat\theta = \sum_{j=1}^J
W_j \theta_j$ is Gaussian.
From Equation~\ref{eq:tilde-mu}, its mean is
\begin{align}
\E[\hat\theta] = \sum_{j=1}^J W_j \E[\theta_j] &= \sum_{j=1}^J W_j \tilde\mu_j \nn \\
&= \sum_{j=1}^J W_j \left(\Sigma_0^{-1} / J + \Sigma_j^{-1} \right)^{-1} \left(\Sigma_j^{-1} \x^{(j)} \right).
\end{align}
Thus, if we choose
\be
W_j = \Sigma \left(\Sigma_0^{-1}/J + \Sigma_j^{-1}\right)
=  \left( \Sigma_0^{-1} + \sum_{j=1}^J \Sigma_j^{-1} \right)^{-1} \!\!\!\!\!\left(\Sigma_0^{-1} / J + \Sigma_j^{-1}\right)
\label{eq:weights}
\ee
where~$\Sigma$ is the posterior covariance in Equation~\ref{eq:gaussian-sigma}, then
\be
\E[\hat\theta] = \Sigma \left(\sum_{j=1}^J \Sigma_j^{-1}  \x^{(j)}\right) = \mu,
\ee
where~$\mu$ is the posterior mean in Equation~\ref{eq:gaussian-mu}.
A similar calculation shows that $\textup{Cov}(\hat\theta) = \Sigma$.

Thus for the Gaussian model, $\hat\theta$ is distributed according to the
posterior distribution, indicating that Equation~\ref{eq:weights} gives the
appropriate weights. Each weight matrix~$W_j$ is a function of~$\Sigma_0$, the
prior covariance, and the subposterior covariances~$\{\Sigma_j\}_{j=1}^J$.
We can form a Monte Carlo estimate of each~$\Sigma_j$ using the empirical
sample covariance~$\bar\Sigma_j$.
Algorithm~\ref{alg:consensus-weighted} summarizes this consensus approach with
weighted averaging.
While this weighting is optimal in the Gaussian setting,
\citet{scott-2013-consensus} shows it to be effective in some non-Gaussian models.
\citet{scott-2013-consensus} also suggests weighting each dimension of
a sample~$\theta$ by the reciprocal of its marginal posterior variance,
effectively restricting the weight matrices $W_j$ to be diagonal.

\begin{algorithm}[t!]
\caption{Consensus of subposteriors via weighted averaging.}
\label{alg:consensus-weighted}
\begin{algorithmic}
\State \textbf{Parameters:} Prior covariance~$\Sigma_0$
\Function{ConsensusSamples}{$\{\theta_{j,1}, \dots, \theta_{j,T}\}_{j=1}^J$}
	\For {$j = 1, 2, \dots, J$}
		\State $\bar\Sigma_j \gets $ Sample covariance of $\{\theta_{j,1}, \dots, \theta_{j,T}\}$
	\EndFor
	\State $\Sigma \gets \left( \Sigma_0^{-1} + \displaystyle\sum_{j=1}^J \bar\Sigma_j^{-1} \right)^{-1}$
	\For {$j = 1, 2, \dots, J$}
		\State $W_j \gets \Sigma \left(\Sigma_0^{-1} / J + \bar\Sigma_j^{-1}\right)$ \Comment{Compute weight matrices}
	\EndFor
	\For {$t = 1, 2, \dots, T$}
		\State $\hat\theta_t \gets \displaystyle\sum_{j=1}^J W_j \theta_{j,t}$ \Comment{Compute weighted averages}
	\EndFor
	\State \Return $\hat\theta_1, \dots, \hat\theta_T$
\EndFunction
\end{algorithmic}
\end{algorithm}

\begin{notes}
\subfile{files/scott-2013-consensus}
\end{notes}

\subsubsection{Subposterior density estimation}
\label{sec:subposterior-kde}

Another consensus strategy relies on density estimation~\citep{xing-2014-embarrassing}.
First, use the subposterior samples to separately fit a density
estimator,~$\tilde\pi^{(j)}(\theta \given \x^{(j)})$, to each subposterior.
The product of these density estimators then represents a density estimator
for the full posterior target, \ie
\be
\pi(\theta \given \x) \approx \tilde\pi(\theta \given \x) = \prod_{j=1}^J \tilde\pi^{(j)}(\theta \given \x^{(j)}) \, .
\ee
Finally, one can sample from this posterior density estimator using MCMC;
ideally, this density is straightforward to obtain and sample.
In general, however, density estimation can yield complex models that are
not amenable to efficient sampling.

\citet{xing-2014-embarrassing} explore three density estimation approaches
of various complexities.
Their first approach assumes a parametric model and is therefore approximate.
Specifically, they fit a Gaussian to each set of subposterior samples, yielding
\be
\tilde\pi(\theta \given \x) = \prod_{j=1}^J \N(\bar\mu_j, \bar\Sigma_j),
\ee
where $\bar\mu_j$ and $\bar\Sigma_j$ are the empirical mean and covariance, respectively,
of the samples from the~$j$th subposterior.
This product of Gaussians simplifies to a single Gaussian~$\N(\hat\mu_J, \hat\Sigma_J)$, where
\begin{align}
\hat\Sigma_J &= \left( \sum_{j=1}^J \bar\Sigma_j^{-1} \right)^{-1}
\label{eq:xing-sigma} \\
\hat\mu_J &= \hat\Sigma_J  \left( \sum_{j=1}^J \bar\Sigma_j^{-1} \bar\mu_j \right).
\label{eq:xing-mu}
\end{align}
These parameters are straightforward to compute and the overall density estimate
can be sampled with reasonable efficiency and even in parallel, if desired.
Algorithm~\ref{alg:consensus-gaussians} summarizes this consensus strategy based on fits to Gaussians.

\begin{algorithm}[t!]
\caption{Consensus of subposteriors via fits to Gaussians.}
\label{alg:consensus-gaussians}
\begin{algorithmic}
\Function{ConsensusSamples}{$\{\theta_{j,1}, \dots, \theta_{j,T}\}_{j=1}^J$}
	\For {$j = 1, 2, \dots, J$}
		\State $\bar\mu_j \gets $ Sample mean of $\{\theta_{j,1}, \dots, \theta_{j,T}\}$
		\State $\bar\Sigma_j \gets $ Sample covariance of $\{\theta_{j,1}, \dots, \theta_{j,T}\}$
	\EndFor
	\State $\hat\Sigma_J \gets \left( \displaystyle\sum_{j=1}^J \bar\Sigma_j^{-1} \right)^{-1}$ \Comment{Covariance of product of Gaussians}
	\State $\hat\mu_J \gets \hat\Sigma_J  \left( \displaystyle\sum_{j=1}^J \bar\Sigma_j^{-1} \bar\mu_j \right)$ \Comment{Mean of product of Gaussians}
	\For {$t = 1, 2, \dots, T$}
		\State $\hat\theta_t \sim \N(\hat\mu_J, \hat\Sigma_J)$ \Comment{Sample from fitted Gaussian}
	\EndFor
	\State \Return $\hat\theta_1, \dots, \hat\theta_T$
\EndFunction
\end{algorithmic}
\end{algorithm}

In the case when the model is jointly Gaussian, the parametric density estimator
we form is~$\N(\hat\mu_J, \hat\Sigma_J)$, with~$\hat\mu_J$ and~$\hat\Sigma_J$
given in Equations~\ref{eq:xing-mu} and~\ref{eq:xing-sigma}, respectively.
In this special case, the estimator exactly represents the Gaussian posterior.
However, recall that we could have instead written the exact posterior directly
as~$\mathcal{N}(\mu, \Sigma)$, where~$\mu$ and~$\Sigma$ are in
Equations~\ref{eq:gaussian-mu} and~\ref{eq:gaussian-sigma}, respectively.
Thus, computing the exact posterior is more or less as expensive as computing the
density estimator, \ie $J$ local matrix inversions (or corresponding linear system solves).

The second approach proposed by \citet{xing-2014-embarrassing} is to use
a nonparametric kernel density estimate (KDE) for each subposterior.
Suppose we obtain~$T$ samples~$\{\theta_{j,t}\}_{t=1}^T$ from the~$j$th subposterior,
then its KDE with bandwidth parameter~$h$ has the following functional form:
\be
\tilde\pi^{(j)}(\theta \given \x^{(j)}) = \frac{1}{T}\sum_{t=1}^T \frac{1}{h^d} K\left(\frac{\|\theta-\theta_{j,t}\|}{h}\right),
\ee
\ie the KDE is a mixture of $T$ kernels, each centered at one of the samples.
If we use~$T$ samples from each subposterior, then the density estimator for the full posterior
is a complicated function with~$T^J$ terms, since it the a product of~$J$ such mixtures,
and is therefore very challenging to sample from.
\citet{xing-2014-embarrassing} use a Gaussian KDE for each subposterior,
and from this derive a density estimator for the full posterior that is a mixture of~$T^J$
Gaussians with unnormalized mixture weights.
They also consider a third, semi-parametric approach to density estimation given by the
product of a parametric (Gaussian) model and a nonparametric (Gaussian KDE) correction.
As the number of samples~${T \rightarrow \infty}$, the nonparametric and semi-parametric
density estimates exactly represent the subposterior densities and are therefore asymptotically exact.
Unfortunately, their complex mixture representations grow exponentially in size,
rendering them somewhat unwieldy in practice.

\begin{notes}
\subfile{files/xing-2014-embarrassing}
\end{notes}

\subsubsection{Weierstrass samplers}
\label{sec:weierstrass}

The consensus strategies surveyed so far are embarrassingly parallel.
These methods obtain samples from each subposterior independently and in
parallel, and from these attempt to construct samples that (approximately)
represent the posterior post-hoc.
The methods in this section proceed similarly, but introduce some amount of
information sharing between the parallel samplers.
This communication pattern is reminiscent of the
alternating direction method of multipliers (ADMM) algorithm for
data-parallel convex optimization;
for a detailed treatment of ADMM, see the review by~\citet{boyd:2011-admm}.

Weierstrass samplers~\citep{dunson-2013-weierstrass} are named for the
Weierstrass transform:
\be
W_h f(\theta) =
\int_{-\infty}^{\infty} \frac{1}{\sqrt{2\pi}h} \exp\left\{-\frac{(\theta-\xi)^2}{2h^2}\right\} f(\xi)d\xi \, ,
\label{eq:weierstrass-transform}
\ee
which was introduced by~\citet{weierstrass:1885}.
The transformed function~${W_h f(\theta)}$ is the convolution of a
one-dimensional function~$f(\theta)$ with a Gaussian density of standard
deviation~$h$, and so converges pointwise to~$f(\theta)$ as~$h \rightarrow 0$,
\[
\lim_{h \rightarrow 0} W_h f(\theta) =
\int_{-\infty}^{\infty} \delta(\theta - \xi) f(\xi) d\xi = f(\theta),
\]
where~$\delta(\tau)$ is the Dirac delta function.
For $h$ > 0, $W_h f(\theta)$ can be thought of as a smoothed approximation
to~$f(\theta)$.
Equivalently, if $f(\theta)$ is the density of a random variable $\theta$, then
$W_h f(\theta)$ is the density of a noisy measurement of $\theta$, where the
noise is an additive Gaussian with zero mean and standard deviation $h$.

\citet{dunson-2013-weierstrass} analyzes a more general class of Weierstrass
transforms by defining a multivariate version and also allowing non-Gaussian kernels:
\[
W_h^{(K)} f(\theta_1, \dots, \theta_d) =
\int_{-\infty}^\infty f(\xi_1, \dots, \xi_d) \prod_{i=1}^d h_i^{-1}
                                K_i\left(\frac{\theta_i - \xi_i}{h_i}\right) d\xi_i.
\]
For simplicity, we restrict our attention to the one-dimensional Weierstrass
transform.

Weierstrass samplers use Weierstrass transforms on subposterior densities to
define an augmented model.
Let~$f_j(\theta)$ denote the~$j$-th subposterior,
\begin{align}
  f_j(\theta) &= \pi^{(j)}(\theta \given \x^{(j)})
      = \pi_0(\theta)^{1/J} \prod_{x \in \x^{(j)}} \pi(x \given \theta),
  \label{eq:weierstrass-subposterior}
\end{align}
so that the full posterior can be approximated as
\begin{align}
\pi(\theta \given \x) \propto \prod_{j=1}^J f_j(\theta)
&\approx \prod_{j=1}^J W_{h} f_j(\theta) \nn \\
&= \prod_{j=1}^J \int \frac{1}{\sqrt{2\pi} h}
   \exp\left\{-\frac{(\theta-\xi_j)^2}{2h^2}\right\} f_j(\xi_j)d\xi_j \nn \\
&\propto \int \prod_{j=1}^J \exp\left\{-\frac{(\theta-\xi_j)^2}{2h^2}\right\} f_j(\xi_j)d\xi_j\, .
\label{eq:weierstrass}
\end{align}
The integrand of~\eqref{eq:weierstrass} defines the joint density of an
augmented model that includes the $\xi = \{\xi_j\}_{j=1}^J$ as auxiliary variables:
\be
\pi_h( \theta, \xi \given \x) \propto
\prod_{j=1}^J \exp\left\{-\frac{(\theta-\xi_j)^2}{2h^2}\right\} f_j(\xi_j).
\ee
The posterior of interest can then be approximated by the marginal distribution
of~$\theta$ in the augmented model,
\be
\int \pi_h(\theta, \xi \given \x) d\xi \approx \pi(\theta \given \x)
\ee
with pointwise equality in the limit as~${h \to 0}$.
Thus by running MCMC in the augmented model, producing Markov chain samples of
both~$\theta$ and~$\xi$, we can generate approximate samples of the posterior.
Furthermore, the augmented model is more amenable to parallelization due to its
conditional independence structure: conditioned on~$\theta$, the subposterior
parameters~$\xi$ are rendered independent.

\begin{figure}[p]
  \centering

  \begin{subfigure}{\textwidth}
    \centering
    \includegraphics[scale=0.5]{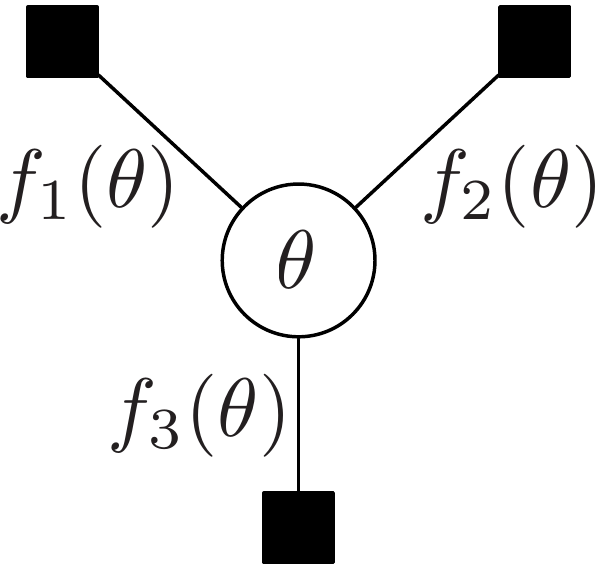}
    \caption{Factor graph for $\pi(\theta \given \x)$ in terms of
    subposterior factors $f_j(\theta)$.}
    \label{fig:weierstrass_graph_exact}
  \end{subfigure}

  \vspace{2em}

  \begin{subfigure}{\textwidth}
    \centering
    \includegraphics[scale=0.5]{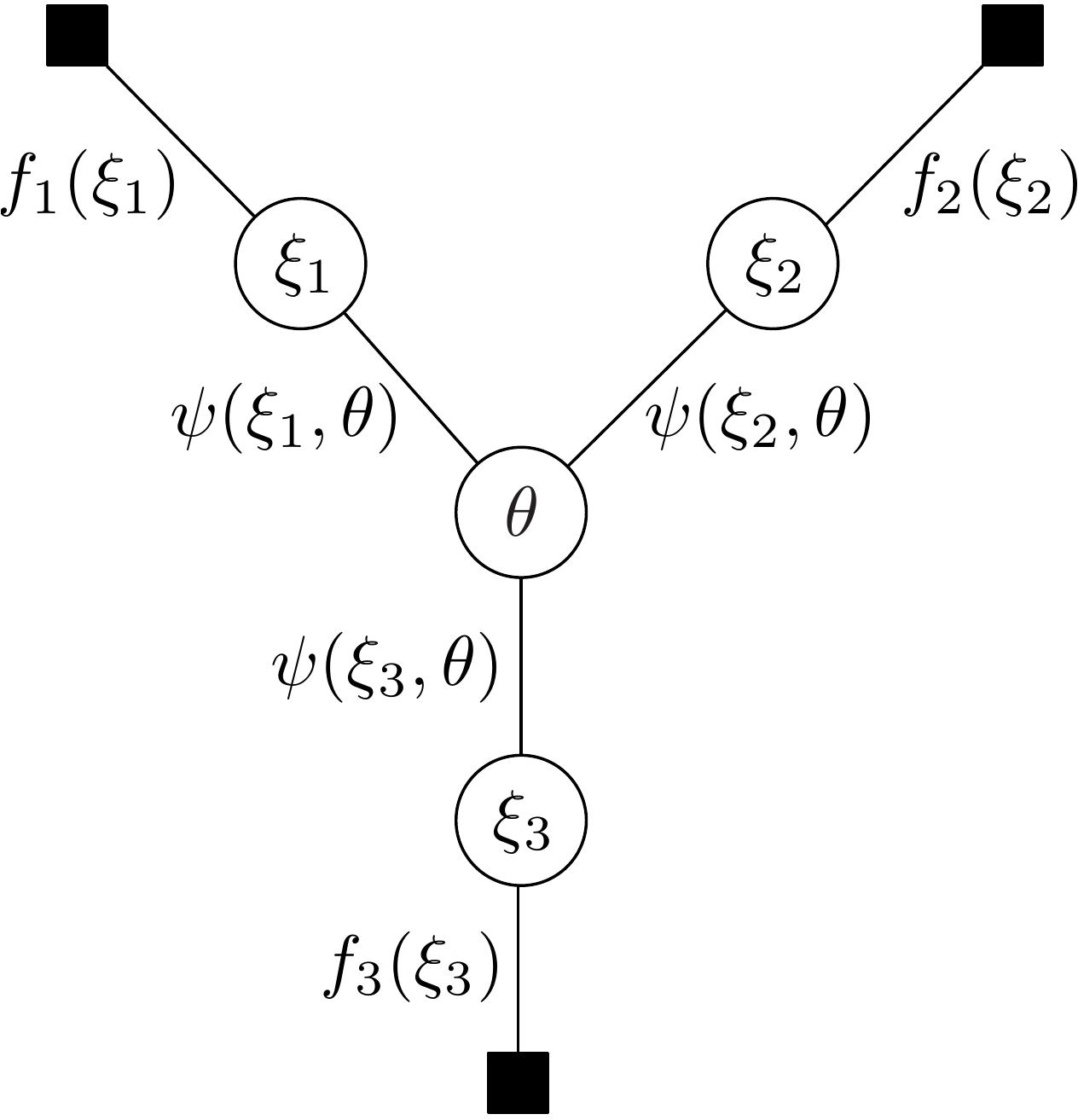}
    \caption{Factor graph for the Weierstrass augmented model $\pi(\theta, \xi
    \given \x)$.}
    \label{fig:weierstrass_graph_approx}
  \end{subfigure}

  \caption{Factor graphs defining the augmented model of the Weierstrass
  sampler.}
\end{figure}

The same augmented model construction can be motivated without explicit
reference to the Weierstrass transform of densities.
Consider the factor graph model of the posterior in
Figure~\ref{fig:weierstrass_graph_exact}, which represents the definition of
the posterior in terms of subposterior factors,
\begin{equation}
  \pi(\theta \given \x) \propto \prod_{j=1}^J f_j(\theta).
\end{equation}
This model can be equivalently expressed as a model where each subposterior
depends on an exact local copy of $\theta$.
That is, writing $\xi_j$ as the local copy of $\theta$ for subposterior $j$,
the posterior is the marginal of a new augmented model given by
\begin{equation}
  \pi(\theta, \xi \given \x) \propto \prod_{j=1}^J f_j(\xi_j) \delta(\xi_j - \theta) \, .
\end{equation}
This new model can be represented by the factor graph in
Figure~\ref{fig:weierstrass_graph_approx}, with potentials $\psi(\xi_j, \theta) =
\delta(\xi_j - \theta)$.
Finally, rather than taking the~$\xi_j$ to be exact local copies of~$\theta$,
we can instead relax them to be noisy Gaussian measurements of~$\theta$:
\begin{align}
  \pi_h(\theta, \xi \given \x) &\propto \prod_{j=1}^J f_j(\xi_j) \psi_h(\xi_j, \theta) \\
  \psi_h(\xi_j, \theta) &= \exp \left\{ - \frac{(\theta - \xi_j)^2}{2 h^2} \right\}.
\end{align}
Thus the potentials $\psi_h(\xi_j, \theta)$ enforce some consistency across the
noisy local copies of the parameter but allow them to be decoupled, where the
amount of decoupling depends on $h$.
With smaller values of $h$ the approximate model is more accurate, but the
local copies are more coupled and hence sampling in the augmented model is less
efficient.

We can construct a Gibbs sampler for the joint distribution~$\pi(\theta, \xi \given \x)$
in Equation~\ref{eq:weierstrass} by alternately sampling from~$p(\theta \given \xi)$
and~$p(\xi_j \given \theta, \x^{(j)})$, for $j=1, \dots, J$.
It follows from Equation~\ref{eq:weierstrass} that
\begin{align}
p(\theta \given \xi_1, \dots, \xi_j, \x) &\propto \prod_{j=1}^J \exp\left\{-\frac{(\theta^2-2\theta \xi_j)}{2h^2} \right\}.
\end{align}
Rearranging terms gives
\begin{align}
p(\theta \given \xi_1, \dots, \xi_j, \x)
&\propto \exp\left\{-\frac{(\theta -\bar\xi)^2}{2h^2/J} \right\},
\end{align}
where $\bar\xi = J^{-1} \sum_{j=1}^J \xi_j$.
The remaining Gibbs updates follow from Equation~\ref{eq:weierstrass},
which directly yields
\begin{align}
p(\xi_j \given \theta, \x^{(j)}) &\propto
                    \exp\left\{-\frac{(\theta - \xi_j)^2}{2h^2}\right\} f_j(\xi_j),
                    \quad j = 1, \dots, J.
\end{align}

\begin{algorithm}[t!]
\caption{Weierstrass Gibbs sampling. For simplicity,~$\theta \in \reals$.}
\label{alg:weierstrass-gibbs}
\begin{algorithmic}
\State \textbf{Input:} Initial state~$\theta_0$, number of samples~$T$,
data partitions ${\x^{(1)}, \dots, \x^{(J)}}$,
subposteriors ${f_1(\theta), \dots, f_J(\theta)}$,
tuning parameter $h$
\State \textbf{Output:} Samples $\theta_1, \dots, \theta_T$
\State Initialize $\theta_0$
\For {$t = 0, 1, \dots, T-1$}
	\State Send $\theta_{t}$ to each processor
	\For {$j = 1, 2, \dots, J$ in parallel}
		\State $\xi_{j,t+1} \sim p(\xi_{j,t+1} \given \theta_t, \x^{(j)}) \propto \N(\xi_{j,t+1} \given \theta_t, h^2) ~ f_j(\xi_{j,t+1})$
	\EndFor
	\State Collect $\xi_{1,t+1}, \dots, \xi_{J,t+1}$
	\State $\bar\xi_{t+1} = \dfrac{1}{J} \displaystyle\sum_{j=1}^J \xi_{j,t+1}$
	\State $\theta_{t+1} \sim \N(\theta_{t+1} \given \bar\xi_{t+1}, h^2/J)$
\EndFor
\end{algorithmic}
\end{algorithm}

This Gibbs sampler allows for parallelism but requires communication at every round.
A straightforward parallel implementation, shown in
Algorithm~\ref{alg:weierstrass-gibbs}, generates the updates for~$\xi_1, \dots,
\xi_J$ in parallel, but the update for~$\theta$ depends on the most recent
values of all the~$\xi_j$.
\citet{dunson-2013-weierstrass} describes an approximate variant of the full
Gibbs procedure that avoids frequent communication by only occasionally
updating~$\theta$.
In other efforts to exploit parallelism while avoiding communication, the
authors propose alternate Weierstrass samplers based on importance sampling and
rejection sampling.

\subsection{Hogwild Gibbs}
\label{sec:pmcmc:hogwild}
Instead of designing new data-parallel algorithms from scratch, another
approach is to take an existing MCMC algorithm and execute its updates in
parallel at the expense of accuracy or theoretical guarantees.
In particular, Hogwild Gibbs algorithms take a Gibbs sampling
algorithm~(\S\ref{sec:gibbs}) with interdependent sequential updates (\eg due
to collapsed parameters or lack of graphical model structure) and simply run
the updates in parallel anyway, using only occasional communication and
out-of-date (stale) information from other processors.
Because these strategies take existing algorithms and let the updates run
`hogwild' in the spirit of Hogwild! stochastic gradient descent in convex
optimization \citep{recht:2011-hogwild}, we refer to these methods as Hogwild
Gibbs.

Similar approaches have a long history.
Indeed, \citet{gonzalez:2011-parallel} attributes a version of this strategy,
Synchronous Gibbs, to the original Gibbs sampling paper \citep{geman1984stochastic}.
However, these strategies have seen renewed interest, particularly due to
extensive empirical work on Approximate Distributed Latent Dirichlet Allocation
(AD-LDA)~\citep{newman:2007-p-lda,newman:2009-ad-lda,asuncion:2008-async-lda,liu:2011-plda+,ihler:2012-understanding},
which showed that running collapsed Gibbs sampling updates in parallel allowed
for near-perfect parallelism without a loss in predictive likelihood
performance.
With the growing challenge of scaling MCMC both to not only big datasets
but also big models, it is increasingly important to understand when and how
these approaches may be useful.

In this section, we first define some variations of Hogwild Gibbs based on
examples in the literature.
Next, we survey the empirical results and summarize the current state of
theoretical understanding.

\subsubsection{Defining Hogwild Gibbs variants}
Here we define some Hogwild Gibbs methods and related schemes, such
as the stale synchronous parameter server.
In particular, we consider bulk-synchronous parallel and asynchronous
variations.
We also fix some notation used for the remainder of the section.

For all of the Hogwild Gibbs algorithms, as with standard Gibbs sampling, we
are given a collection of~$n$ random variables, $\{x_i : i \in [n]\}$ where
$[n] \triangleq \{1,2,\ldots,n\}$, and we assume that we can sample from the
conditional distributions~$x_i | x_{\neg i}$, where~$x_{\neg i}$ denotes $\{
x_j : j \neq i\}$.
For the Hogwild Gibbs algorithms, we also assume we have~$K$ processors, each
of which is assigned a set of variables on which to perform MCMC updates.
We represent an assignment of variables to processors by fixing a partition
$\{\mathcal{I}_1,\mathcal{I}_1,\ldots,\mathcal{I}_K\}$ of~$[n]$, so that the
$k$th processor performs updates on the state values indexed by~$\mathcal{I}_k$.

\subsubsection*{Bulk-synchronous parallel Hogwild Gibbs}

A bulk-synchronous parallel (BSP) Hogwild Gibbs algorithm assigns variables to
processors and alternates between performing parallel processor-local updates
and global synchronization steps.
During epoch~$t$, the~$k$th processor performs~$q(t,k)$ MCMC updates, such as
Gibbs updates, on the variables $\{x_i : i \in \mathcal{I}_k\}$ without
communicating with the other processors; in particular, these updates are
computed using out-of-date values for all $\{x_j : j \not \in \mathcal{I}_k\}$.
After all processors have completed their local updates, all processors
communicate the updated state values in a global synchronization step and the
system advances to the next epoch.
We summarize this Hogwild Gibbs variant in Algorithm~\ref{hogwild:alg:bsp},
in which the local MCMC updates are taken to be Gibbs updates.

\begin{algorithm}[t]
  \caption{Bulk-synchronous parallel (BSP) Hogwild Gibbs}
  \label{hogwild:alg:bsp}

  \begin{algorithmic}
      \Require Joint distribution over $x=(x_1,\ldots,x_n)$, partition
      $\{\mathcal{I}_1, \ldots, \mathcal{I}_K\}$ of $\{1,2,\ldots,n\}$,
      iteration schedule $q(t,k)$
      \State Initialize $\bar{x}^{(1)}$
      \For{$t=1,2,\ldots$}
          \For{$k=1,2,\ldots,K$ in parallel}
              \State $\bar{x}_{\mathcal{I}_k}^{(t+1)} \gets \Call{LocalGibbs}{\bar{x}^{(t)}, \mathcal{I}_k, q(t,k)}$
          \EndFor
          \State Synchronize
      \EndFor
  \end{algorithmic}

  \begin{algorithmic}
      \Function{LocalGibbs}{$\bar{x}$, $\mathcal{I}$, $q$}
          \For{$j=1,2,\ldots,q$}
              \For{$i \in \mathcal{I}$ in order}
                  \State $\bar{x}_i \gets \text{sample~} x_i \given x_{\neg i} = \bar{x}_{\neg i}$
              \EndFor
          \EndFor
      \Return $\bar{x}$
      \EndFunction
  \end{algorithmic}
\end{algorithm}

Several special cases of the BSP Hogwild Gibbs scheme have been of interest.
The Synchronous Gibbs scheme of \citet{gonzalez:2011-parallel}
associates one variable with each processor, so that $|\mathcal{I}_k|=1$ for
each $k=1,2,\ldots,K$ (in which case we may take $q=1$ since no local
iterations are needed with a single variable).
One may also consider the case where the partition is arbitrary and $q$ is very
large, in which case the local MCMC iterations may converge and exact block
samples are drawn on each processor using old statistics from other processors
for each outer iteration.
Finally, note that setting $K=1$ and $q(t,k)=1$ reduces to standard Gibbs sampling
on a single processor.

\subsubsection*{Asynchronous Hogwild Gibbs}

Another Hogwild Gibbs pattern involves performing updates asynchronously.
That is, processors might communicate only by sending messages to one another
instead of by a global synchronization.
Versions of this Hogwild Gibbs pattern has proven effective both for collapsed
latent Dirichlet allocation topic model inference
\citep{asuncion:2008-async-lda}, and for Indian Buffet Process inference
\citep{doshi-velez:2009-p-ibp}.
A version was also explored in the Gibbs sampler of the Stale Synchronous
Parameter (SSP) server of~\citet{ho:2013-stale}, which placed an upper bound on
the staleness of the entries of the state vector on each processor.

There are many possible communication strategies in the asynchronous setting,
and so we follow a version of the random communication strategy employed by
\citet{asuncion:2008-async-lda}.
In this approach, after performing some number of local updates, a processor
sends its updated state information to a set of randomly-chosen processors and
receives updates from other processors.
The processor then updates its state representation and performs another round
of local updates.
A version of this asynchronous Hogwild Gibbs strategy is summarized in
Algorithm~\ref{hogwild:alg:async}.

\begin{algorithm}[t]
  \caption{Asynchronous Hogwild Gibbs}
  \label{hogwild:alg:async}

  \begin{algorithmic}
    \State Initialize $\bar{x}^{(1)}$
    \For{each processor $k=1,2,\ldots,K$ in parallel}
      \For{$t=1,2,\ldots$}
        \State $\bar{x}_{\mathcal{I}_k}^{(t+1)} \gets \Call{LocalGibbs}{\bar{x}^{(t)}, \mathcal{I}_k, q(t,k)}$
        \State Send $\bar{x}_{\mathcal{I}_k}^{(t+1)}$ to $K'$ randomly-chosen processors
        \For{each $k' \neq k$}
          \If{update $\bar{x}_{\mathcal{I}_{k'}}$ received from processor $k'$}
            \State $\bar{x}_{\mathcal{I}_{k'}}^{(t+1)} \gets \bar{x}_{\mathcal{I}_{k'}}$
          \Else
            \State $\bar{x}_{\mathcal{I}_{k'}}^{(t+1)} \gets \bar{x}_{\mathcal{I}_{k'}}^{(t)}$
          \EndIf
        \EndFor
      \EndFor
    \EndFor
  \end{algorithmic}
\end{algorithm}

\subsubsection{Theoretical analysis}
Despite its empirical successes, theoretical understanding of Hogwild Gibbs
algorithms is limited.
There are two settings in which some analysis has been offered: first, in a
variant of AD-LDA, \ie~Hogwild Gibbs applied to Latent Dirichlet Allocation
models, and second in the jointly Gaussian case.

The work of \citet{ihler:2012-understanding} provides some understanding of the
effectiveness of a variant of AD-LDA by bounding in terms of run-time
quantities the one-step error probability induced by proceeding with sampling
steps in parallel, thereby allowing an AD-LDA user to inspect the computed
error bound after inference~\citep[Section 4.2]{ihler:2012-understanding}. In
experiments, the authors empirically demonstrate very small upper bounds on
these one-step error probabilities, \eg~a value of their parameter
$\varepsilon=10^{-4}$ meaning that at least $99.99\%$ of samples are expected
to be drawn just as if they were sampled sequentially.  However, this
per-sample error does not necessarily provide a direct understanding of the
effectiveness of the overall algorithm because errors might accumulate over
sampling steps; indeed, understanding this potential error accumulation is of
critical importance in iterative systems. Furthermore, the bound is in terms of
empirical run-time quantities, and thus it does not provide guidance on which other models the Hogwild strategy may be effective. \citet[Section
4.3]{ihler:2012-understanding} also provides approximate scaling analysis by
  estimating the order of the one-step bound in terms of a Gaussian
  approximation and some distributional assumptions.

The jointly Gaussian case is more tractable for
analysis~\citep{johnson:2013-analyzing,johnson:2014-thesis}.
In particular, \citet[Theorem 7.6.6]{johnson:2014-thesis} shows that for the BSP Hogwild Gibbs
process to be stable, \ie~to form an ergodic Markov chain and have a
well-defined stationary distribution, for any variable partition and any
iteration schedule it suffices for the model's joint Gaussian precision matrix
to satisfy a generalized diagonal dominance condition.
Because the precision matrix contains the coefficients of the log potentials in
a Gaussian graphical model, the diagonal dominance condition captures the
intuition that Hogwild Gibbs should be stable when variables do not interact
too strongly.
\citet[Proposition 7.6.8]{johnson:2014-thesis} gives a more refined condition
for the case where the number of processor-local Gibbs iterations is large.

When a bulk-synchronous parallel Gaussian Hogwild Gibbs process defines an
ergodic Markov chain and has a stationary distribution, \citet[Chapter
7]{johnson:2014-thesis} also provides an understanding of how that stationary
distribution relates to the model distribution.
Because both the model distribution and the Hogwild Gibbs process stationary
distribution  are Gaussian, accuracy can be measured in terms of the mean
vector and covariance matrix.
\citet[Proposition 7.6.1]{johnson:2014-thesis} shows that the mean of a stable
Gaussian Hogwild Gibbs process is always correct.
\citet[Propositions 7.7.2 and 7.7.3]{johnson:2014-thesis} identify a tradeoff
in the accuracy of the process covariance matrix as a function of the number of
processor-local Gibbs iterations: at least when the processor interactions are
sufficiently weak, more processor-local iterations between synchronization
steps increase the accuracy of the covariances among variables within each
processor but decrease the accuracy of the covariances between variables on
different processors.
\citet[Proposition 7.7.4]{johnson:2014-thesis} also gives a more refined error
bound as well as an inexpensive way to correct covariance estimates for the
case where the number of processor-local Gibbs iterations is large.

\section{Summary}
Many ideas for parallelizing MCMC have been proposed, exhibiting many
tradeoffs.
These ideas vary in generality, in faithfulness to the posterior, and in the
parallel computation architectures for which they are best suited.
Here we summarize the surveyed methods, emphasizing their relative strengths on
these criteria.
See Table~\ref{table:parallel} for an overview.

\begin{landscape}
\begin{table}[p]
\centering
\resizebox{\linewidth}{!}{%
\begin{tabular}{lllllll}
\toprule
  & \pbox{5cm}{\textbf{Parallel density}\\ \textbf{evaluation} (\S\ref{sec:conditional-independence})}
  & \textbf{Prefetching} (\S\ref{sec:prefetching})
  & \textbf{Consensus} (\S\ref{sec:pmcmc:subposteriors})
  & \textbf{Weierstrass} (\S\ref{sec:pmcmc:subposteriors})
  & \textbf{Hogwild Gibbs} (\S\ref{sec:pmcmc:hogwild}) \\
\midrule
  \textbf{Requirements}
  & Conditional independence
  & None
  & Approximate factorization
  & Approximate factorization
  & \pbox{5cm}{Weak dependencies across processors} \\
\midrule
  \textbf{Parallel model}
  & BSP
  & Speculative execution
  & MapReduce
  & BSP
  & \pbox{5cm}{BSP and asynchronous message passing variants} \\
\midrule
  \textbf{Communication}
  & Each iteration
  & Master scheduling
  & Once
  & Tuneable
  & Tuneable \\
\midrule
  \textbf{Design choices}
  & None
  & Scheduling policy
  & \pbox{5cm}{Data partition, \\ consensus algorithm}
  & \pbox{5cm}{Data partition, \\ synchronization frequency, \\ Weierstrass $h$}
  & \pbox{5cm}{Data partition, \\ communication frequency }\\
\midrule
  \pbox{5cm}{\textbf{Computational} \\ \textbf{overheads}}
  & None
  & Scheduling, bookkeeping
  & Consensus step
  & Auxiliary variable sampling
  & None \\
\midrule
  \pbox{5cm}{\textbf{Approximation} \\ \textbf{error}}
  & None
  & None
  & \pbox{5cm}{Depends on number of processors and consensus algorithm}
  & \pbox{5cm}{Depends on number of processors and Weierstrass $h$}
  & \pbox{5cm}{Depends on number of processors and staleness} \\
\midrule
\end{tabular}
}
\caption{Summary of recent approaches to parallel MCMC.}
\label{table:parallel}
\end{table}
\end{landscape}

\paragraph{Simulating independent Markov chains}
Independent instances of serial MCMC algorithms can be run in an embarrassingly
parallel manner, requiring only minimal communication between processors to
ensure distinct initializations and to collect samples.
This approach can reduce Monte Carlo variance by increasing the number of
samples collected in any time budget, achieving an ideal parallel speedup, but
does nothing to accelerate the warm-up period of the chains during which the
transient bias is eliminated (see Section~\ref{sec:mcmc} and
Chapter~\ref{ch:discussion}).
That is, using parallel resources to run independent chains does nothing to
improve mixing unless there is some mechanism for information sharing as in \citet{nishihara-2014-gess}.
  In addition, running an independent MCMC chain on each processor requires each
processor to access the full dataset, which may be problematic for especially
large datasets.
These considerations motivate both subposterior methods and Hogwild Gibbs.

\paragraph{Direct parallelization of standard updates}
Some MCMC algorithms applied to models with particular structure allow for
straightforward parallel implementation.
In particular, when the likelihood is factorized across data points,
the computation of the Metropolis--Hastings acceptance probability can be
parallelized.
This strategy lends itself to a bulk-synchronous parallel (BSP) computational
model.
Parallelizing MH in this way yields exact MCMC updates and can be effective at
reducing the mixing time required by serial MH, but it requires a simple
likelihood function and its implementation requires frequent synchronization
and communication, mitigating parallel speedups unless the likelihood function is very expensive.

Gibbs sampling also presents an opportunity for direct parallelization for
particular graphical model structures.
In particular, given a graph coloring of the graphical model, variables
corresponding to nodes assigned a particular color are conditionally mutually
independent and can be updated in parallel without communication.
However, frequent synchronization and significant communication can be required
to transmit sampled values to neighbors after each update.
Relaxing both the strict conditional independence requirements and
synchronization requirements motivates Hogwild Gibbs.

\paragraph{Prefetching and speculative execution}
The prefetching algorithms studied in Section~\ref{sec:prefetching} use
speculative execution to transform traditional (serial) Metropolis--Hastings
into a parallel algorithm without incurring approximate updates or requiring
any model structure.
The implementation naturally follows a master-worker pattern, where the master
allocates (possibly speculative) computational work, such as proposal generation
or (partial) density evaluation, to worker processors.
Ignoring overheads, basic prefetching algorithms achieve at least logarithmic
speedup in the number of processors available.
More sophisticated scheduling by the master, such as predictive
prefetching~\citep{angelino:2014-prefetching}, can increase speedup
significantly.
While this method is very general and yields the same iterates as serial MH,
the speedup can be limited.

\paragraph{Subposterior consensus and Weierstrass samplers}
Subposterior methods, such as the consensus Monte Carlo algorithms
and the Weierstrass samplers of
Section~\ref{sec:weierstrass}, allow for data parallelism and minimal
communication because each subposterior Markov chain can be allocated to a
processor and simulation can proceed independently.
Communication is required only for final sample aggregation in consensus Monte Carlo
or the periodic resampling of the global parameter in the Weierstrass sampler.
In consensus Monte Carlo, the quality of the approximate inference depends on
both the effectiveness of the aggregation strategy and the extent to which
dependencies in the posterior can be factorized into subposteriors.
The Weierstrass samplers directly trade off approximation quality and the
amount of decoupling between subposteriors.

The consensus Monte Carlo approach originated at
Google~\citep{scott-2013-consensus} and naturally fits the MapReduce
programming model, allowing it to be executed on large computational clusters.
Recent work has extended consensus Monte Carlo and provides tools for
designing simple consensus strategies~\citep{rabinovich:2015-vcmc}, but the
generality and approximation quality of subposterior methods remain unclear.
The Weierstrass sampler fits well into a BSP model.

\paragraph{Hogwild Gibbs}
Hogwild Gibbs of Section~\ref{sec:pmcmc:hogwild} also allows for data
parallelism but avoids factorizing the posterior as in consensus Monte Carlo or
instantiating coupled copies of a global parameter as in the Weierstrass
sampler.
Instead, processor-local sampling steps (such as local Gibbs updates) are
performed with each processor treating other processors' states as fixed at
stale values; processors can communicate updated states less frequently, either
via synchronous or asynchronous communication.
Hogwild Gibbs variants span a range of parallel computation paradigms from
fully synchronous BSP to fully asynchronous message-passing.
While Hogwild Gibbs has proven effective in practice for several models, its
applicability and approximation tradeoffs remain unclear.

\section{Discussion}
The ideas surveyed in this chapter suggest several challenges and questions.

\paragraph{More parallelism means less accuracy}
The new data-parallel methods surveyed here, namely consensus Monte
Carlo, the Weierstrass samplers, and Hogwild Gibbs, do not generate samples
that are asymptotically distributed according to the target posterior.
Instead, each generates samples that are asymptotically distributed according
to a distribution that is meant to approximate the target posterior.
While the nature of these approximations differ, each has a tradeoff between
parallelism and accuracy: increasing parallelism by using more processors
decreases the faithfulness of the asymptotic posterior approximation.
This tradeoff may be inherent to most data-parallel MCMC schemes, though
parallel predictive prefetching strategies do not suffer the same drawback.

\paragraph{How to split up data?}
The performance of data-parallel methods may be significantly affected by the
data partitioning that assigns data subsets to processors.
In the case of Hogwild Gibbs, it is probably best to choose a data partition
that minimizes the strength of cross-processor dependence.
Similarly, in the case of subposterior methods, some factorizations may be
more effective than others.
Since data paritioning is likely to have significant practical effects for all
of these methods, it may be fruitful to develop and analyze general heuristics
for assigning data to processors.

\paragraph{Analysis of approximation quality and tradeoffs}
The works surveyed in this chapter introduce several alternative
approximations.
While each is well motivated, it is unclear how to choose the most appropriate
method for a given model, or how to think about and compare the various
approximations and tradeoffs.
A more unified perspective is necessary, through empirical comparison or
through analyzing their application to simple models that are tractable for
analysis.

%% file: chapters/variational/main.tex
\chapter{Scaling variational mean field algorithms}
\label{ch:mean-field}

Variational inference is a standard paradigm for posterior inference in
Bayesian models.
Because variational methods pose inference as an optimization problem, ideas in
scalable optimization can in principle yield scalable posterior inference algorithms.
In this chapter, we consider such scalable algorithms mainly in the context of
mean field variational inference, which is often called variational Bayes.

These scalable variational inference algorithms can be compared to the
algorithms of Chapters~\ref{sec:mcmc-subsets} and~\ref{sec:mcmc-parallel}
in the same way that variational methods are usually compared to MCMC.
That is, because inference is typically performed in a family of distributions
that does not include the exact posterior, it can be said that variational
methods do not fully instantiate the Bayesian computation that MCMC methods do
(at least, when given unbounded computation time).
Indeed, MAP inference, in which the posterior is represented only as a single
atom, is an extreme case of variational inference.
More generally, mean field variational families typically provide only unimodal
approximations, and additionally cannot represent some posterior correlations
near particular modes.
As a result, MCMC methods can provide better performance even when the Markov
chain only explores a single mode in a reasonable number of iterations.
\todo{citation needed (maybe YWTeh stuff)}

Despite these potential shortcomings, variational inference is widely used in
machine learning because the computational advantage over MCMC can be
significant.
This computational advantage is particularly salient in the context of scaling
inference to large datasets.
The big data context may also inform the relative cost of performing inference
in a constrained variational family rather than attempting to represent the
posterior exactly: when the posterior is concentrated, a variational
approximation may suffice.
While such questions may ultimately need to be explored empirically on a
case-by-case basis, the scalable variational inference methods surveyed in this
chapter provide the tools for such an exploration.

In this chapter we summarize two patterns of scalable variational inference.
First, in Section~\ref{sec:stochastic}, we discuss the application of stochastic
gradient optimization methods to mean field variational inference problems.
Second, in Section~\ref{sec:svb}, we describe an alternative approach that
instead leverages the idea of incremental posterior updating to develop an
inference algorithm with minibatch-based updates.

\section{Stochastic optimization and variational inference}
\label{sec:stochastic}
Stochastic gradient optimization is a powerful tool for scaling optimization
algorithms to large datasets, and it has been applied to mean field
variational inference problems to great effect.
While many traditional algorithms for optimizing mean field objective
functions, including both gradient-based and coordinate optimization methods,
require re-reading the entire dataset in each iteration, the stochastic
gradient framework allows each update to be computed with respect to
minibatches of the dataset while providing very general asymptotic convergence
guarantees.

In this section we first summarize the stochastic variational inference (SVI)
framework of~\citet{hoffman2013stochastic}, which applies to models with
complete-data conjugacy.
Next, we discuss alternatives and extensions which can handle more general
models at the cost of updates with greater variance and, hence, slower
convergence.

\subsection{SVI for complete-data conjugate models}
\label{sec:svi}
This section follows the development in~\citet{hoffman2013stochastic}.
It depends on results from stochastic gradient optimization theory;
see Section~\ref{sec:background:sgd} for a review. For
notational simplicity we consider each minibatch to consist of only a single
observation; the generalization to minibatches of arbitrary sizes is immediate.

Many common probabilistic models are hierarchical: they can be written in terms
of \emph{global} latent variables (or parameters), \emph{local} latent
variables, and observations.
That is, many models can be written as
\begin{equation}
    p(\phi,z,y)=p(\phi)\prod_{k=1}^K p(z^{(k)} \given \phi)p(y^{(k)} \given z^{(k)},\phi)
    \label{eq:svi-joint}
\end{equation}
where $\phi$ denotes global latent variables, $z=\{z^{(k)}\}_{k=1}^K$ denotes local latent variables, and $y=\{y^{(k)}\}_{k=1}^K$ denotes observations.
See Figure~\ref{fig:svi} for a graphical model.
Given such a class of models, the mean field variational inference
problem is to approximate the posterior $p(\phi,z \given \bar{y})$ for fixed data
$\bar{y}$ with a distribution of the form $q(\phi)q(z) = q(\phi) \prod_k q(z^{(k)})$
by finding a local minimum of the KL divergence from the approximating distribution to
the posterior or, equivalently, finding a local maximum of the
marginal likelihood lower bound
\begin{equation}
    \mathcal{L}[q(\phi) q(z)] \triangleq \E_{q(\phi)q(z)}
    \left[ \log \frac{p(\phi,z,\bar{y})}{q(\phi)q(z)} \right] \leq \log p(\bar{y}).
    \label{eq:generic_vlb}
\end{equation}

\begin{figure}[t]
    \centering
    \includegraphics[width=1.25in]{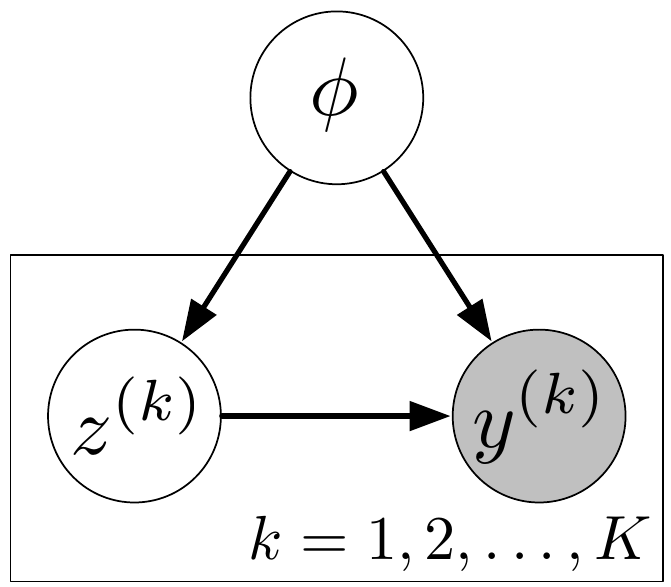}
    \caption{Prototypical graphical model for stochastic variational inference
    (SVI). The global latent variables are represented by $\phi$ and the local
    latent variables by $z^{(k)}$.}
    \label{fig:svi}
\end{figure}

\citet{hoffman2013stochastic} develops a stochastic gradient ascent algorithm
for such models that leverages complete-data conjugacy.
Gradients of $\mathcal{L}$ with respect to the parameters of $q(\phi)$ have a
convenient form if we assume the prior $p(\phi)$ and each complete-data
likelihood~$p(z^{(k)},y^{(k)} \given \phi)$ are a conjugate pair of exponential family
densities.
That is, if we have
\begin{align}
    \log p(\phi) &= \langle \eta_\phi, \; t_\phi(\phi) \rangle - \log Z_\phi(\eta_\phi) 
    \label{eq:prior_var_fam}
    \\
    \log p(z^{(k)},y^{(k)} \given \phi) &= \langle \eta_{zy}(\phi),\; t_{zy}(z^{(k)},y^{(k)}) \rangle - \log Z_{zy}(\eta_{zy}(\phi))
\end{align}
then conjugacy identifies the statistic of the prior with the natural parameter
and log partition function of the likelihood
via
\begin{align}
t_\phi(\phi) &=
(\eta_{zy}(\phi),\; -\log Z_{zy}(\eta_{zy}(\phi))\,,
\end{align}
so that
\begin{equation}
    p(\phi, z^{(k)},\bar{y}^{(k)}) \propto \exp \left\{ \langle \eta_\phi + {\textstyle (t_{zy}(z^{(k)}, \; \bar{y}^{(k)}),1)}, t_\phi(\phi) \rangle \right\}\,.
    \label{eq:var_conjugacy_term}
\end{equation}
Conjugacy implies the optimal variational factor $q(\phi)$ has the same form as
the prior; that is, without loss of generality we can write $q(\phi)$ in the same form as~\eqref{eq:prior_var_fam},
\begin{equation}
    q(\phi) = \exp \left\{
        \langle \widetilde{\eta}_\phi, \; t_\phi(\phi) \rangle -
        \log Z_\phi(\widetilde{\eta}_\phi) \right\},
\end{equation}
for some variational parameter $\widetilde{\eta}_\phi$.

Given this conjugacy structure, we can find a simple expression for the gradient of
$\mathcal{L}$ with respect to the global variational parameter
$\widetilde{\eta}_\phi$, optimizing out the local variational factor $q(z)$.
That is, we write the variational objective over global parameters as
\begin{equation}
    \mathcal{L}(\widetilde{\eta}_\phi) = \max_{q(z)} \mathcal{L}[q(\phi)q(z)]\,.
\end{equation}
Writing the optimal parameters of $q(z)$ as $\widetilde{\eta}_z^*$,
note that when $q(z)$ is partially optimized to a stationary point of
$\mathcal{L}$, so that~${\frac{\partial \mathcal{L}}{\partial
\widetilde{\eta}_z^*} = 0}$ at~$\widetilde{\eta}_z^*$, the chain rule implies
that the gradient with respect to the global variational parameters simplifies:
\begin{align}
    \frac{\partial \mathcal{L}}{\partial \widetilde{\eta}_\phi} (\widetilde{\eta}_\phi)
    &= \frac{\partial \mathcal{L}}{\partial \widetilde{\eta}_\phi} (\widetilde{\eta}_\phi, \widetilde{\eta}_z^*) + \frac{\partial \mathcal{L}}{\partial \widetilde{\eta}_z^*} \frac{\partial \widetilde{\eta}_z^*}{\partial \widetilde{\eta}_\phi} (\widetilde{\eta}_\phi, \widetilde{\eta}_z^*)
    \\
    &= \frac{\partial \mathcal{L}}{\partial \widetilde{\eta}_\phi} (\widetilde{\eta}_\phi, \widetilde{\eta}_z^*)\,.
\end{align}
Because the optimal local factor $q(z)$ can be computed with local mean
field updates for a fixed value of the global variational parameter
$\widetilde{\eta}_\phi$, we need only find an expression for the gradient
$\nabla_{\widetilde{\eta}_\phi} \mathcal{L}(\widetilde{\eta}_\phi)$ in terms of
the optimized local factors.

To find an expression for the gradient $\nabla_{\widetilde{\eta}_\phi}
\mathcal{L}(\widetilde{\phi}_\phi)$ that exploits conjugacy structure,
using~\eqref{eq:var_conjugacy_term} we can substitute
\begin{equation}
    p(\phi,z,\bar{y})\propto \exp \left\{ \langle \eta_\phi + \sum_k (t_{zy}(z^{(k)},\bar{y}^{(k)}),1), \; t_\phi(\phi) \rangle \right\},
\end{equation}
into the definition of $\mathcal{L}$ in \eqref{eq:generic_vlb}. 
Using the optimal form of $q(\phi)$, we have
\begin{align}
    \mathcal{L}(\widetilde{\eta}_\phi)
    &=
    \E_{q(\phi) q(z)}
    \left[
        \langle \eta_\phi + \sum_k t_{zy}(z^{(k)},\bar{y}^{(k)}) - \widetilde{\eta}_\phi, \;
        t_\phi(\phi) \rangle
    \right]
    \notag \\
    & \qquad + \log Z_\phi(\widetilde{\eta}_\phi) + \text{const.}
    \\
    &=
    \langle \eta_\phi + \sum_k \E_{q(z^{(k)})} [t_{zy}(z^{(k)},\bar{y}^{(k)})] -
    \widetilde{\eta}_\phi,
    \;
    \E_{q(\phi)} [t_\phi(\phi)] \rangle
    \notag \\
    & \qquad + \log Z_\phi(\widetilde{\eta}_\phi) + \text{const}\,,
\end{align}
where the constant does not depend on $\widetilde{\eta}_\phi$.
Using the identity for natural exponential families that
\begin{equation}
    \nabla  \log Z_\phi(\widetilde{\eta}_\phi) = \E_{q(\phi)}[t_\phi(\phi)],
\end{equation}
we can write the same expression as
\begin{align}
    \mathcal{L}(\widetilde{\eta}_\phi)
    =
    &
    \langle \eta_\phi + \sum_k \E_{q(z^{(k)})} [t_{zy}(z^{(k)},\bar{y}^{(k)})] -
    \widetilde{\eta}_\phi, \;
    \nabla \log Z_\phi(\widetilde{\eta}_\phi) \rangle
    \notag \\
    &
    + \log Z_\phi(\widetilde{\eta}_\phi) + \text{const}\,.
\end{align}
Thus we can compute the gradient of $\mathcal{L}(\widetilde{\eta}_\phi)$ with
respect to the global variational parameters $\widetilde{\eta}_\phi$ as
\begin{align}
    \nabla_{\widetilde{\eta}_\phi} \mathcal{L}(\widetilde{\eta}_\phi)
    &=
    \langle \nabla^2 \log Z_\phi(\widetilde{\eta}_\phi), \;
    \eta_\phi + \sum_k \E_{q(z^{(k)})} [t_{zy}(z^{(k)},\bar{y}^{(k)})] -
    \widetilde{\eta}_\phi \rangle
    \notag \\
    &\qquad - \nabla \log Z_\phi(\widetilde{\eta}_\phi)
    + \nabla \log Z_\phi(\widetilde{\eta}_\phi)
    \label{eq:gradient-of-elbo}
    \\
    &=
    \langle \nabla^2 \log Z_\phi(\widetilde{\eta}_\phi), \;
    \eta_\phi + \sum_k \E_{q(z^{(k)})} [t_{zy}(z^{(k)},\bar{y}^{(k)})] -
    \widetilde{\eta}_\phi \rangle
    \notag
\end{align}
where the first two terms come from applying the product rule.

The matrix $\nabla^2 \log Z_\phi(\widetilde{\eta}_\phi)$ is the Fisher information of the
variational family, since
\begin{equation}
    - \E_{q(\phi)} \left[ \nabla^2_{\widetilde{\eta}_\phi} \log q(\phi) \right] = \nabla^2 \log Z_\phi(\widetilde{\eta}_\phi).
\end{equation}
In the context of stochastic gradient ascent,
we can cancel the multiplication by the matrix $\nabla^2
\log Z_\phi(\widetilde{\eta}_\phi)$ simply by choosing the sequence of positive definite
matrices in Algorithm~\ref{alg:sgd} to be $G^{(t)} \triangleq \nabla^2
\log Z_\phi(\widetilde{\eta}_\phi^{(t)})^{-1}$.
This choice yields a stochastic \emph{natural} gradient ascent algorithm
\citep{amari2007methods}, where the updates are stochastic approximations to the natural
gradient
\begin{equation}
    \widetilde{\nabla}_{\widetilde{\eta}_\phi} \mathcal{L}
    =
    \eta_\phi + \sum_k \E_{q(z^{(k)})} [t_{zy}(z^{(k)},\bar{y}^{(k)})] - \widetilde{\eta}_\phi.
\end{equation}
Natural gradients effectively include a second-order quasi-Newton correction
for local curvature in the variational family, making the updates invariant to
reparameterization of the variational family and thus often improving
performance of the algorithm.
More importantly, at least for the case of complete-data conjugate families
considered here, natural gradient steps are in fact easier to compute than
`flat' gradient steps in either the natural parameterization or moment
parameterization of the variational family $q(\phi)$.

Therefore a stochastic natural gradient ascent algorithm on the global
variational parameter $\widetilde{\eta}_\phi$ proceeds at iteration $t$ by
sampling a minibatch $\bar{y}^{(k)}$ and taking a step of some size
$\rho^{(t)}$ in an approximate natural gradient direction via
\begin{equation}
    \widetilde{\eta}_\phi \gets (1-\rho^{(t)}) \widetilde{\eta}_\phi + \rho^{(t)} \left( \eta_\phi + K \E_{q(z^{(k)})} [t(z^{(k)},\bar{y}^{(k)})] \right)
\end{equation}
where we have assumed the minibatches are of equal size to simplify notation.
The local variational factor $q(z^{(k)})$ is computed using a local mean
field update on the data minibatch and the global variational factor.
That is, if $q(z^{(k)})$ is not further factorized in the mean field
approximation, it is computed according to
\begin{equation}
    q(z^{(k)}) \propto \exp \left\{ \E_{q(\phi)} [ \log p(z^{(k)}  \given  \phi)
        p(\bar{y}^{(k)} \given z^{(k)}, \phi)] \right\}.
\end{equation}
We summarize the general SVI algorithm in Algorithm~\ref{alg:svi}.

\begin{algorithm}[t]
    \caption{Stochastic Variational Inference (SVI)}
    \begin{algorithmic}
        \State Initialize global variational parameter $\widetilde{\eta}_\phi^{(0)}$
        \For{$t=0,1,2,\ldots$}
        \State $\hat{k} \gets$ sample index $k$ with probability $p_k > 0$, for $k=1,2,\ldots,K$
            \State $q(z^{(\hat{k})}) \gets
            \text{\textsc{LocalMeanField}}(\widetilde{\eta}^{(t)},
            \bar{y}^{(\hat{k})})$
            \State $\widetilde{\eta}_\phi^{(t+1)} \gets (1-\rho^{(t)}) \widetilde{\eta}_\phi^{(t)} + \rho^{(t)} \left( \eta_\phi + \frac{1}{p_{\hat{k}}} \E_{q(z^{(\hat{k})})} \left[t(z^{(\hat{k})},\bar{y}^{(\hat{k})})\right] \right)$
        \EndFor
    \end{algorithmic}
    \label{alg:svi}
\end{algorithm}

\subsection{Stochastic gradients with general nonconjugate models}
\label{sec:bbvi}
The development of SVI in the preceding section assumes that $p(\phi)$ and
$p(z,y \given \phi)$ are a conjugate pair of exponential families.
This assumption led to a particularly convenient form for the natural gradient
of the mean field variational objective and hence an efficient stochastic
gradient ascent algorithm.
However, when models do not have this conjugacy structure, more general
algorithms are required.

In this section we review Black Box Variational Inference (BBVI), which is a
stochastic gradient algorithm for variational inference that can be applied at
scale~\citep{ranganath:2014-bbvi}.
The ``black box'' name suggests its generality: while the
stochastic variational inference of Section~\ref{sec:svi} requires particular
model structure, BBVI only requires that the model's log joint distribution can
be evaluated.
It also makes few demands of the variational family, since it only requires
that the family can be sampled and that the gradient of its log joint with
respect to the variational parameters can be computed efficiently.
With these minimal requirements, BBVI is not only useful in the big-data
setting but also a tool for handling nonconjugate variational inference more
generally.
Because BBVI uses Monte Carlo approximation to compute stochastic gradient updates,
it fits naturally into a stochastic gradient optimization framework, and hence
it has the additional benefit of yielding a scalable algorithm simply by adding
minibatch sampling to its updates at the cost of increasing their variance.
In this subsection we review the general BBVI algorithm and then compare it to
the SVI algorithm of Section~\ref{sec:svi}.
For a review of Monte Carlo estimation, see Section~\ref{sec:monte-carlo}.

We consider a general model $p(\theta,y) = p(\theta)\prod_{k=1}^K p(y^{(k)} \given
\theta)$ including parameters $\theta$ and observations $y=\{y^{(k)}\}_{k=1}^K$
divided into $K$ minibatches.
The distribution of interest is the posterior $p(\theta \given y)$ and we write
the variational family as $q(\theta) = q(\theta \given \widetilde{\eta}_\theta)$,
where we suppress the particular mean field factorization structure of
$q(\theta)$ from the notation.
The mean field variational lower bound is then
\begin{equation}
    \mathcal{L} = \E_{q(\theta)} \left[ \log \frac{p(\theta,y)}{q(\theta)} \right].
\end{equation}
Taking the gradient with respect to the variational parameter
$\widetilde{\eta}_\theta$ and expanding the expectation into an integral, we have
\begin{align}
    \nabla_{\widetilde{\eta}_\theta} \mathcal{L}
    &= \nabla_{\widetilde{\eta}_\theta} \int q(\theta) \log \frac{p(\theta,y)}{q(\theta)} d\theta
    \label{eq:bbvi_0}
    \\
    &= \int \nabla_{\widetilde{\eta}_\theta} \left[ \log \frac{p(\theta, y)}{q(\theta)} \right] q(\theta) d\theta
    + \int \log \frac{p(\theta,y)}{q(\theta)} \nabla_{\widetilde{\eta}_\theta} q(\theta) d\theta
    \label{eq:bbvi_1}
\end{align}
where we have moved the gradient into the integrand and applied the product
rule to yield two terms.
The first term is identically zero:
\begin{align}
    \int \nabla_{\widetilde{\eta}_\theta} \left[ \log \frac{p(\theta, y)}{q(\theta)} \right] q(\theta) d\theta
    &= - \int \frac{1}{q(\theta)} \nabla_{\widetilde{\eta}_\theta} \left[ q(\theta) \right] q(\theta) d\theta
    \\
    &= - \int \nabla_{\widetilde{\eta}_\theta} q(\theta) d \theta
    \\
    &= - \nabla_{\widetilde{\eta}_\theta} \int q(\theta) d\theta = 0
\end{align}
where we have used $\nabla_{\widetilde{\eta}_\theta} \log p(\theta,y) = 0$.
To write the second term of \eqref{eq:bbvi_1} in a form that allows convenient
Monte Carlo approximation, we first note the identity
\begin{equation}
    \nabla_{\widetilde{\eta}_\theta} \log q(\theta) = \frac{\nabla_{\widetilde{\eta}_\theta} q(\theta)}{q(\theta)}
    \implies
    \nabla_{\widetilde{\eta}_\theta} q(\theta) = q(\theta) \nabla_{\widetilde{\eta}_\theta} \log q(\theta)
\end{equation}
and hence we can write the second term of \eqref{eq:bbvi_1} as
\begin{align}
    \int  \log \frac{p(\theta,y)}{q(\theta)}  \nabla_{\widetilde{\eta}_\theta} q(\theta) d\theta
    &=
    \int  \log \frac{p(\theta,y)}{q(\theta)}  \nabla_{\widetilde{\eta}_\theta} \left[ \log q(\theta) \right]  q(\theta) d \theta
    \\
    &=
    \E_{q(\theta)} \left[ \log \frac{p(\theta,y)}{q(\theta)} \nabla_{\widetilde{\eta}_\theta} \log q(\theta) \right]
    \label{eq:bbvi_3}
    \\
    &\approx
    \frac{1}{|\mathcal{S}|} \sum_{\hat{\theta} \in S} \log \frac{p(\hat\theta,y)}{q(\hat\theta)} \nabla_{\widetilde{\eta}_\theta} \log q(\hat\theta)
    \label{eq:bbvi_2}
\end{align}
where in the final line we have written the expectation as a Monte Carlo
estimate using a set of samples $\mathcal{S}$, where $\hat\theta \iid\sim q(\theta)$ for
$\hat\theta \in \mathcal{S}$.
Notice that the gradient is written as a weighted sum of gradients of the
variational log density with respect to the variational parameters, where the
weights depend on the model log joint density.

The BBVI algorithm uses the Monte Carlo estimate \eqref{eq:bbvi_2} to compute
stochastic gradient updates.
This gradient estimator is also known as the score function estimator
\citep{kleijnen1996optimization,gelman1998simulating}.
The variance of these updates, and hence the convergence of the overall
stochastic gradient algorithm, depends both on the sizes of the gradients of
the variational log density and on the variance of $q(\theta)$.
Large variance in the gradient estimates can lead to very slow optimization,
and so \citet{ranganath:2014-bbvi} proposes and evaluates two variance reduction schemes,
including a control variate method as well as a Rao-Blackwellization method
that can exploit factorization structure in the variational family.

To provide a scalable version of BBVI, gradients can be further approximated by
subsampling minibatches of data.
That is, using~${\log p(\theta,y) = \log p(\theta) + \sum_{k=1}^K \log p(y^{(k)} \given 
\theta)}$ we write \eqref{eq:bbvi_3} and \eqref{eq:bbvi_2} as
\begin{align}
    \nabla_{\widetilde{\eta}_\theta} \mathcal{L}
    &=
    \E_{\hat{k}} \left[ \E_{q(\theta)} \left[
            \left( \log \frac{p(\theta)}{q(\theta)} + K \log p(y^{(\hat{k})} \given \hat\theta) \right)
            \nabla_{\widetilde{\eta}_\theta} \log q(\theta)
        \right] \right]
    \\
    &\approx
    \frac{1}{|\mathcal{S}|} \sum_{\hat{\theta} \in \mathcal{S}} \left( \log \frac{p(\hat\theta)}{q(\hat\theta)} + K \E_{\hat{k}} \left[ \log p(y^{(\hat{k})} \given \hat\theta) \right] \right) \nabla_{\widetilde{\eta}_\theta} \log q(\hat\theta)
\end{align}
with the minibatch index $\hat{k}$ distributed uniformly over
$\{1,2,\ldots,K\}$ and the minibatches are assumed to be the same size for
simpler notation.
This subsampling over minibatches further increases the variance of the updates
and thus may further limit the rate of convergence of the algorithm.
We summarize this version of the BBVI algorithm in Algorithm~\ref{alg:bbvi}.

\begin{algorithm}[t]
    \caption{Minibatch Black-Box Variational Inference (BBVI)}
    \begin{algorithmic}
        \State Initialize $\widetilde{\eta}_\theta^{(0)}$
        \For{$t=0,1,2,\ldots$}
            \State $\mathcal{S} \gets \{ \hat{\theta}_s \} \text{ where } \hat{\theta}_s \sim q(\; \cdot \; | \, \widetilde{\eta}_\theta^{(t)})$
            \State $\hat{k} \sim \text{Uniform}(\{1,2,\ldots,K\})$
            \State $\widetilde{\eta}_\theta^{(t+1)} \gets \widetilde{\eta}_\theta^{(t)} + \frac{1}{|\mathcal{S}|} \sum_{\hat{\theta} \in \mathcal{S}} \left( \log \frac{p(\hat{\theta})}{q(\hat{\theta})} + K \log p(y^{(\hat{k})} \given \hat{\theta}) \right) \nabla_{\widetilde{\eta}_\theta} \log q(\hat{\theta})$
        \EndFor
    \end{algorithmic}
    \label{alg:bbvi}
\end{algorithm}

It is instructive to compare the fully general BBVI algorithm applied to
hierarchical models to the SVI algorithm of Section~\ref{sec:svi}; this
comparison not only shows the benefits of exploiting conjugacy structure but
also suggests a potential Rao-Blackwellization scheme.
Taking~${\theta=(\phi,z)}$ and~${q(\theta) = q(\phi)q(z)}$ and starting from
\eqref{eq:bbvi_0} and \eqref{eq:bbvi_3}, we can write the gradient as
\begin{align}
    \nabla_{\widetilde{\eta}_\theta} \mathcal{L}
    &=
    \E_{q(\phi)q(z)} \left[ \log \frac{p(\phi,z,y)}{q(\phi)q(z)} \nabla_{\widetilde{\eta}_\phi} \log q(\phi)q(z) \right]
    \\
    &\approx
    \frac{1}{|S|} \sum_{\hat\phi \in S} \left( \E_{q(z)} \log \frac{p(\hat\phi,z,y)}{q(\hat\phi)} - \E_{q(z)} \log q(z) \right)
    \nabla_{\widetilde{\eta}_\phi} \log q(\hat\phi)
\end{align}
where $S$ is a set of samples with $\hat\phi \iid\sim q(\phi)$ for $\hat\phi
\in S$.
Thus if the entropy of the local variational distribution $q(z)$ and the
expectations with respect to $q(z)$ of the log density $\log p(\hat\phi,z,y)$
can be computed without resorting to Monte Carlo estimation, then the resulting
update would likely have a lower variance than the BBVI update that requires
sampling over both $q(\phi)$ and $q(z)$.

This comparison also makes clear the advantages of exploiting conjugacy in SVI:
when the updates of Section~\ref{sec:svi} can be used, neither $q(\phi)$ nor
$q(z)$ needs to be sampled. Furthermore, while BBVI uses stochastic gradients
in its updates, the SVI algorithm of Section~\ref{sec:svi} uses stochastic
natural gradients, adapting to the local curvature of the variational family.
Computing stochastic natural gradients in BBVI would require both computing
the Fisher information matrix of the variational family and solving a linear
system with it.

\subsection{Exploiting reparameterization for some nonconjugate models}
\label{sec:reparameterization-trick}
While the score function estimator developed for BBVI in Section~\ref{sec:bbvi}
is sufficiently general to handle essentially any model, some nonconjugate
models admit convenient stochastic gradient estimators that can have lower
variance.
In particular, in settings where the latent variables are continuous (or any
discrete latent variables that can be marginalized efficiently) samples from some
variational distributions can be reparameterized in a way that enables an
alternative stochastic gradient estimator.
This technique is related to non-centered reparameterizations~\citep{papaspiliopoulos2007general}
and has recently been called the reparameterization trick~\citep{DBLP:journals/corr/KingmaW13,rezende2014stochastic}.

The reparameterization trick applies when samples~${\hat \theta \sim q(\theta)}$,
where~$q(\theta)$ has parameter $\widetilde{\eta}_\theta$, can be written
as
\begin{equation}
    \hat \theta = f(\widetilde{\eta}_\theta, \epsilon)
\end{equation}
where $\epsilon \sim p(\epsilon)$ is a random variable with a distribution $p(\epsilon)$ that does not depend on
$\widetilde{\eta}_\theta$ and where $\nabla_{\widetilde{\eta}_\theta}
f(\widetilde{\eta}_\theta, \epsilon)$ can be computed efficiently for almost every
value of $\epsilon$.
In this case, we can compute stochastic estimates of the gradient of the
variational objective by first writing a Monte Carlo approximation of the
objective function itself:
\begin{equation}
    \mathcal{L}
    = \E_{q(\theta)} \left[ \log \frac{p(\theta, y)}{q(\theta)} \right]
    \approx \frac{1}{|S|} \sum_{\hat \epsilon \in S} \log \frac{p(f(\widetilde{\eta}_\theta, \hat \epsilon), y)}{q(f(\widetilde{\eta}_\theta, \hat \epsilon))}
\end{equation}
where $\hat \epsilon \iid \sim p(\epsilon)$ for each $\hat \epsilon \in S$.
Alternatively, when the variational entropy term $\E_{q(\theta)} \log
q(\theta)$ can be computed efficiently, \eg if the variational distribution
is a Gaussian, only the energy term needs to be approximated via Monte Carlo:
\begin{equation}
    \mathcal{L}
    \approx - \E_{q(\theta)} \log q(\theta) + \frac{1}{|S|} \sum_{\hat \epsilon \in S} \log p(f(\widetilde{\eta}_\theta, \hat \epsilon), y).
\end{equation}

This Monte Carlo approximation is a differentiable unbiased estimate of
$\mathcal{L}$ as a function of the variational parameter
$\widetilde{\eta}_\theta$, and so we can form a Monte Carlo estimate of the
gradient of the variational objective simply by differentiating it:
\begin{equation}
    \nabla_{\widetilde{\eta}_\theta} \mathcal{L}
    \approx
    - \nabla_{\widetilde{\eta}_\theta} \E_{q(\theta)} \log q(\theta)
    + \frac{1}{|S|} \sum_{\hat \epsilon \in S} \nabla_{\widetilde{\eta}_\theta} \log p(f(\widetilde{\eta}_\theta, \hat \epsilon), y).
\end{equation}
This estimator often has lower variance than the fully general score function
estimator~\citep{DBLP:journals/corr/KingmaW13} and can be easier to compute.

\section{Streaming variational Bayes (SVB)}
\label{sec:svb}
Streaming variational Bayes (SVB) provides an alternative framework in which to
derive minibatch-based scalable variational inference~\citep{broderick:2013-svb}.
While the methods of Section~\ref{sec:stochastic} generally apply stochastic
gradient optimization algorithms to a fixed variational mean field objective,
SVB instead considers the streaming data setting, in which case there may be no
fixed dataset size and hence no fixed variational objective.
To handle streaming data, the SVB approach is based on the classical idea of
Bayesian updating, in which a posterior is updated to reflect new data as they
become available.
This sequence of posteriors is approximated by a sequence of variational
models, and each variational model is computed from the previous variational
model via an incremental update on new data.

More concretely, given a prior $p(\theta)$ over a parameter $\theta$ and a
(possibly infinite) \todo{philosophically troubling with a finite set of parameters} sequence of data minibatches $y^{(1)},y^{(2)},\ldots$, each
distributed independently according to a likelihood distribution
$p(y^{(k)} \given \theta)$, we consider the sequence of posteriors
\begin{equation}
    p(\theta \given y^{(1)},\ldots,y^{(t)}), \quad t=1,2,\ldots.
\end{equation}
Given an approximation updating algorithm $\mathcal{A}$ one can compute a
corresponding sequence of approximations
\begin{equation}
    p(\theta \given y^{(1)}, \cdots, y^{(t)}) \approx q_t(\theta) = \mathcal{A} \left( y^{(t)}, q_{t-1}(\theta) \right), \quad t=1,2,\ldots
    \label{eq:svb_1}
\end{equation}
with $q_0(\theta) p(\theta)$.
This sequential updating view naturally suggests an online or one-pass
algorithm in which the update \eqref{eq:svb_1} is applied successively to each
of a sequence of minibatches.

A sequence of such updates may also exploit parallel or distributed computing
resources.
For example, the sequence of approximations may be computed as
\begin{align}
    p(\theta \given y^{(1)}, &\cdots, y^{(K_t)})
    \approx q_t(\theta)
    \\
    &= q_{t-1}(\theta) \left( \prod_{k=K_{t-1}+1}^{K_t} \mathcal{A} \left(y^{(k)}, q_{t-1}(\theta) \right) q_{t-1}(\theta)^{-1} \right)
    \label{eq:svb_2}
\end{align}
where $K_{t-1}+1,K_{t-1}+2,\ldots,K_{t}$ indexes a set of data minibatches for
which each update is computed in parallel before being combined in the final
update from $q_{t-1}(\theta)$ to $q_t(\theta)$.

This combination of partial results is especially appealing when the prior
$p(\theta)$ and the family of approximating distributions $q(\theta)$ are in
the same exponential family,
\begin{gather}
    p(\theta) \propto \exp \left\{ \langle \eta, \; t(\theta) \rangle \right\}
    \quad
    q_0(\theta) \propto \exp \left\{ \langle \widetilde{\eta}_0, \; t(\theta) \rangle \right\}
    \\
    q_t(\theta) = \mathcal{A}(y^{(k)}, q_{t-1}(\theta)) \propto \exp \left\{ \langle \widetilde{\eta}_t, \; t(\theta) \rangle \right\}
\end{gather}
for a prior natural parameter $\eta$ and a sequence of variational parameters~$\widetilde{\eta}_t$.
In the exponential family case, the updates \eqref{eq:svb_2} can be written
\begin{align}
    p(\theta \given y^{(1)}, \cdots, y^{(K_t)})
    \approx q_t(\theta)
    &\propto
    \exp \left\{ \langle \widetilde{\eta}_t, \; t(\theta) \rangle \right\}
    \\
    &=
    \exp \left\{ \langle \widetilde{\eta}_{t-1} + \sum_k (\widetilde{\eta}_k - \widetilde{\eta}_{t-1}), \; t(\theta) \rangle \right\}
\end{align}
where we may take the algorithm $\mathcal{A}$ to return an updated natural
parameter, $\widetilde{\eta}_k = \mathcal{A}(y^{(k)},\widetilde{\eta}_t)$.

Finally, similar updates can be performed in an asynchronous distributed
master-worker setting.
Each worker can process a minibatch and send the corresponding natural
parameter increment  to a master process, which
updates the global variational parameter and transmits back the updated
variational parameter along with a new data minibatch.
In symbols, we can write that a worker operating on minibatch~$y^{(k)}$ for
some minibatch index $k$ computes the update increment $\Delta
\widetilde{\eta}_{k}$ according to
\begin{equation}
    \Delta \widetilde{\eta}_{k} = \mathcal{A}(y^{(k)}, \widetilde{\eta}_{\tau(k)})
\end{equation}
where $\tau(k)$ is the index of the global variational parameter used in the
worker's computation.
Upon receiving an update, the master updates its global variational parameter
synchronously according to
\begin{equation}
    \widetilde{\eta}_{t+1} = \widetilde{\eta}_{t} + \Delta \widetilde{\eta}_k.
\end{equation}
We summarize a version of this process in Algorithms~\ref{alg:svb} and~\ref{alg:svb2}.

\begin{algorithm}[t]
    \caption{Streaming Variational Bayes (SVB)}
    \begin{algorithmic}
        \State Initialize $\widetilde{\eta}_0$
        \For{each worker $p=1,2,\ldots,P$}
            \State Send task $(y^{(p)}, \widetilde{\eta}_0)$ to worker $p$
        \EndFor
        \As{workers send updates}
            \State Receive update $\Delta \widetilde{\eta}_k$ from worker $p$
            \State $\widetilde{\eta}_{t+1} \gets \widetilde{\eta}_t + \Delta \widetilde{\eta}_k$
            \State Retrieve new data minibatch index $k'$
            \State Send new task $(y^{(k')},\widetilde{\eta}_{t+1})$ to worker $p$
        \EndAs
    \end{algorithmic}
    \label{alg:svb}
\end{algorithm}

\begin{algorithm}[t]
    \caption{SVB Worker Process}
    \begin{algorithmic}
        \Repeat
            \State Receive task $(y^{(k)},\widetilde{\eta}_t)$ from master
            \State $\Delta \widetilde{\eta}_k \gets \mathcal{A}(y^{(k)},\widetilde{\eta}_t) - \widetilde{\eta}_t$
            \State Send update $\Delta \widetilde{\eta}_k$ to master
        \Until{no tasks remain}
    \end{algorithmic}
    \label{alg:svb2}
\end{algorithm}

A related algorithm, which we do not detail here, is Memoized Variational
Inference (MVI)~\citep{hughes2013memoized,hughes2015reliable}.
While this algorithm is designed for the fixed dataset setting rather than the
streaming setting, the updates can be similar to those of SVB.
In particular, MVI optimizes the mean field objective in the conjugate
exponential family setting using the mean field coordinate descent algorithm
but with an atypical update order, in which only some local variational factors
are updated at a time.
This update order enables minibatch-based updating but, unlike the stochastic
gradient algorithms, does not optimize out the other local variational factors
not included in the minibatch and instead leaves them fixed.

Streaming variational inference algorithms similar to SVB have also recently
been studied in some Bayesian nonparametric mixture
models~\citep{tank2015streaming}.

\section{Summary}

\paragraph{Stochastic gradients vs. streaming.}
The methods of Section~\ref{sec:stochastic} apply stochastic optimization to
variational mean field inference objectives.
In optimization literature and practice, stochastic gradient methods have a
large body of both theoretical and empirical support, and so such methods offer
a compelling framework for scalable inference.
The streaming ideas surveyed in Section~\ref{sec:svb} are less well understood,
but by treating the streaming setting, rather than the setting of a large
fixed-size dataset, they may extend the reach of Bayesian modeling and
inference.

\paragraph{Minibatching.}
All of this chapter's scalable approaches to mean field variational inference
are based on processing minibatches of data. These algorithms arrive at this
data access pattern via two routes: the first applies stochastic gradient
optimization to mean field variational inference~(\S\ref{sec:stochastic}) and
the second considers the streaming data setting~(\S\ref{sec:svb}).
SVI~(\S\ref{sec:svi}) and BBVI~(\S\ref{sec:bbvi}) optimize the variational
objective by replacing full gradient updates with stochastic gradient updates.
In both SVI and BBVI, these approximate gradients arise from randomly
sampling data minibatches, while in BBVI there is additional stochasticity due
to the Monte Carlo approximation required to handle nonconjugate structure.
In contrast to SVI and BBVI, SVB~(\S\ref{sec:svb}) processes data minibatches to
drive incremental posterior updates, constructing a sequence of approximate
posterior distributions that correspond to classical sequential Bayesian
updating without having a single fixed objective to optimize.

\paragraph{Generality, requirements, and assumptions.}
As with most approaches to scaling MCMC samplers for Bayesian inference,
these minibatch-based variational inference methods depend on model structure.
In SVI and scalable BBVI, minibatches map to terms
in a factorization of the joint probability.
In SVB, minibatches map to a sequence of likelihoods to be incorporated into
the variational posterior.
Some of these methods further depend on and exploit exponential family and
conjugacy structure.  SVI is based on complete-data conjugacy,
while BBVI was specifically developed for nonconjugate models.
SVB is a general framework, but in the conjugate exponential family case
the updates can be written in terms of simple updates to natural parameters.
A direction for future research might be to develop new methods based on
identifying and exploiting some `middle ground' between the structural requirements of
SVI and BBVI, or similarly of SVB with and without exponential family structure.

\section{Discussion}

\paragraph{Parallel variants.}
The minibatch-based variational inference methods developed in this chapter suggest
parallel and asynchronous variants.
In the case of SVB, distributed and asynchronous versions,
such as the master-worker pattern depicted by Algorithms~\ref{alg:svb} and~\ref{alg:svb2},
have been empirically studied~\citet{broderick:2013-svb}.
However, we lack theoretical understanding about these procedures, and it is
unclear how to define and track notions of convergence or stability.
Methods based on stochastic gradients, such as SVI, can naturally be extended to
exploit parallel and asynchronous (or ``Hogwild'') variants of stochastic
gradient ascent.
In such parallel settings, these optimization-based techniques benefit from
powerful gradient convergence results \citep[Section
7.8]{bertsekas1989parallel}, though tuning such algorithms is still a
challenge.
Other parallel versions of these ideas and algorithms have also been developed
in \citet{Campbell14_UAI} and \citet{Campbell15_NIPS}.